\PassOptionsToPackage{dvipsnames,table}{xcolor}
\documentclass{article}

\usepackage{xcolor}

\definecolor{linkblue}{RGB}{0,90,180}
\definecolor{lightlinkblue}{RGB}{70,150,230} 

\usepackage[
    colorlinks=true,
    bookmarksopen=true,
    linkcolor=lightlinkblue,
    citecolor=linkblue,      
    urlcolor=linkblue       
]{hyperref}

\usepackage{xspace}
\usepackage[a4paper, left=3cm, right=3cm, top=3cm, bottom=3cm]{geometry}

\usepackage[T1]{fontenc}
\usepackage[utf8]{inputenc}
\usepackage{lmodern}
\usepackage{libertine}
\usepackage{microtype}

\usepackage{amsmath}
\usepackage{amssymb}
\usepackage{amsthm}
\usepackage{mathrsfs}
\usepackage{mathtools}
\usepackage{bm}
\usepackage{nicefrac}
\usepackage{upgreek}

\usepackage{graphicx} 
\usepackage{subcaption}
\usepackage{booktabs}
\usepackage{makecell}
\usepackage{multirow,multicol}
\usepackage{array}
\usepackage{enumitem}
\usepackage{float}
\usepackage{tikz}
\usepackage{siunitx}
\usepackage{bbm}
\usepackage{comment}
\usepackage{dsfont}
\usepackage{pifont}
\usepackage{lineno}

\usepackage{algorithm}
\usepackage{algpseudocode}
\makeatletter
\providecommand{\theHALG@line}{}
\renewcommand{\theHALG@line}{\thealgorithm.\arabic{ALG@line}}
\makeatother

\usepackage[authoryear,round]{natbib}

\usepackage[capitalise,sort,nameinlink]{cleveref}

\theoremstyle{plain}
\newtheorem{theorem}{Theorem}[section]

\newtheorem{proposition}[theorem]{Proposition}

\theoremstyle{definition}

\theoremstyle{remark}
\newtheorem{remark}[theorem]{Remark}

\crefname{thm}{theorem}{theorems}
\Crefname{thm}{Theorem}{Theorems}
\crefname{lem}{lemma}{lemmas}
\Crefname{lem}{Lemma}{Lemmas}
\crefname{cor}{corollary}{corollaries}
\Crefname{cor}{Corollary}{Corollaries}
\crefname{prop}{proposition}{propositions}
\Crefname{prop}{Proposition}{Propositions}
\crefname{defn}{definition}{definitions}
\Crefname{defn}{Definition}{Definitions}
\crefname{ex}{example}{examples}
\Crefname{ex}{Example}{Examples}

\newcommand{\bb}[1]{\mathbb{#1}}
\newcommand{\diff}{\,\mathrm{d}}

\newcommand{\R}{\bb{R}}

\newcommand{\N}{\bb{N}}

\providecommand{\E}{\mathbb{E}}

\definecolor{ControlColor}{HTML}{AD2831}
\definecolor{TreatmentColor}{HTML}{55A868}
\definecolor{ComponentOneColor}{HTML}{17BECF}
\definecolor{ComponentTwoColor}{HTML}{FF7F0E}
\definecolor{ComponentThreeColor}{HTML}{2CA02C}
\definecolor{ComponentFourColor}{HTML}{D62728}
\definecolor{ComponentFiveColor}{HTML}{9467BD}
\newcommand{\Control}{\textcolor{ControlColor}{Control}}
\newcommand{\Treatment}{\textcolor{TreatmentColor}{Treatment}}
\newcommand{\CompOne}{\textcolor{ComponentOneColor}{Component 1}}
\newcommand{\CompTwo}{\textcolor{ComponentTwoColor}{Component 2}}
\newcommand{\CompThree}{\textcolor{ComponentThreeColor}{Component 3}}
\newcommand{\CompFour}{\textcolor{ComponentFourColor}{Component 4}}
\newcommand{\CompFive}{\textcolor{ComponentFiveColor}{Component 5}}

\title{Continuous-time learning of probability distributions:\\
a case study in a digital trial of young children with type 1 diabetes}
\author{Antonio Álvarez-López\\
\small Nanyang Technological University\\
\small \texttt{antonio.al@ntu.edu.sg}
\and
Marcos Matabuena\\
\small Mohamed bin Zayed University of Artificial Intelligence\\
\small \texttt{marcos.matabuena@mbzuai.ac.ae}}
\date{}

\begin{document}

\maketitle

\begin{abstract}
Understanding how biomarker distributions evolve over time is a central challenge in digital health and chronic disease monitoring. For example, in diabetes, changes in the distribution of glucose measurements can reveal patterns of disease progression and treatment response that conventional summary measures miss. Motivated by a 13-week clinical trial comparing the closed-loop insulin delivery system t:slim X2 with standard therapy in children with type 1 diabetes, we propose a probabilistic framework to model the continuous-time evolution of time-indexed distributions using continuous glucose monitoring (CGM) data collected every five minutes. We represent the glucose distribution as a Gaussian mixture, with time-varying mixture weights governed by a neural ODE. We estimate the model parameters using a distribution-matching criterion based on the maximum mean discrepancy. The resulting framework is interpretable, computationally efficient, and sensitive to subtle temporal distributional changes.  In the CGM trial data analyzed here, the method detects treatment-related improvements in glucose dynamics that are difficult to capture with traditional analytical approaches.
\end{abstract}


\section{Introduction}\label{s:intro}

Characterizing the distribution of a random variable is a classical problem in statistics \citep{Silverman2018Density} and remains a central challenge in modern machine learning \citep{bengio2017deep}, where accurate distribution representations are essential for tasks such as text generation and automated reporting \citep{Mesko2023Imperative}. 
More broadly, many scientific questions require understanding not only individual observations but also how the full distribution of a process evolves over time. This perspective is particularly relevant in clinical applications. 

In digital health, estimating the distribution of individual physiological time-series data over specific time periods enables the construction of individual representations that capture their underlying physiological processes with high precision \citep{Matabuena2021Glucodensities,matabuena2023distributional,Ghosal2023Distributional}. Recent studies show that, when used properly, such representations can reveal clinically relevant patterns that traditional (non-digital) biomarkers do not detect \citep{Katta2024Interpretable,Matabuena2024Glucodensity,Park2025Beyond,matabuena2026exploratory}.


In this paper, motivated by digital health applications, we study the problem of continuously estimating a time-indexed distribution $\{F_t\}_{t\in[0,T]}$ from sequentially observed data. The goal is to learn how the underlying distribution evolves and to represent that evolution in a way that is both flexible and interpretable. Standard approaches are often unsatisfactory in this setting. Extending classical kernel density estimators  \citep{chacon2018multivariate} to include time typically leads to a strong sensitivity to tuning parameters and to the curse of dimensionality, while flow-based generative models \citep{Papamakarios2021Normalizing} may be less interpretable and can require substantial training efforts. Semiparametric alternatives offer partial remedies. For example, time-varying models such as Generalized Additive Models for Location, Scale, and Shape \citep{rigby2005generalized} alleviate some of these issues, but most implementations are designed for scalar responses and may impose rigid functional forms. More recent multilevel functional approaches based on quantile representations \citep{10.1214/26-AOAS2139I} provide an interpretable framework for modeling univariate distributional outcomes. However, their reliance on linear modeling assumptions may limit their ability to capture complex nonlinear statistical dependencies, while the use of quantile representations makes their extension to multivariate distributional outcomes challenging.

To address these limitations, we propose a continuous-time Gaussian mixture framework in which distributional dynamics are represented through time-varying mixture weights. We estimate a common Gaussian dictionary by maximum mean discrepancy (MMD) and then use a neural ODE to obtain a continuous mixture-weight trajectory for each participant. The fitted distributions can be compared over time using either the full density or familiar summaries such as means and covariances.

\subsection*{Problem formulation}

Let $T>0$. For each $t\in[0,T]$, let $X_t\in\R^d$ denote a random vector representing the quantity of interest at time $t$. Its (cumulative) distribution function is
\begin{equation}\label{eqn:eq1}
  F_t(x)
  \;\coloneqq\;
  \mathbb{P}\left(X_t\le x\right)
  \;=\;
  \int_{(-\infty,x]} f_t(r)\diff r,\qquad x\in\R^d,
\end{equation}
where inequality and integral are taken component-wise when $d>1$, and $f_{t}(\cdot)$ is the probability density function (when it exists) at time $t$. 

Our target object is $F_t(\cdot)$, or equivalently, the density curve \(t\mapsto f_t\), from which \(F_t\) can be recovered through \eqref{eqn:eq1}. In practice, however, the process is not observed continuously in time. Instead, data are available on a discrete time grid
\begin{equation}\label{eq:time.grid}
\tau_m \coloneqq \{t_0,\ldots,t_m\}\subset[0,T].
\end{equation}
Let $\mu_t$ denote the law of $X_t$. At each \(t_i\in\tau_m\), we observe a sample drawn from $\mu_{t_i}$,
\begin{equation}\label{eq:obs.DH}
X_{t_i,1},\ldots,X_{t_i,N_i}\sim \mu_{t_i},\qquad t_i\in\tau_m.
\end{equation}
Depending on the application, these observations may be treated either as approximately independent snapshots across time or as part of a longitudinal setting with temporal dependence. In both cases, the statistical problem is to recover a coherent continuous-time representation of the underlying distributional dynamics from these discrete observations.

We represent temporal distributional changes through a Gaussian mixture
with shared components and time-varying weights. In our motivating
application, the Gaussian components may be interpreted as latent
glycemic regimes whose locations and scales remain relatively stable
over time, while changes in their relative weights reflect changes
in the prevalence of these regimes. This provides a common and
interpretable representation for comparing distributions across
individuals and time points, at the cost of restricting temporal
variation to the mixture weights. We model these weights using a
neural ODE, obtaining continuous-time trajectories that describe
the evolution of the underlying distributions.

\subsection*{Digital health motivation and distributional data analysis}

Our motivation comes from the analysis of glucose distributions in longitudinal diabetes trials \citep{Battelino2023Continuous}, where glucose is constantly recorded using continuous glucose monitoring (CGM). Under free-living conditions, individual glucose time series cannot be directly aligned, making the raw temporal stochastic process difficult to compare between participants \citep{Ghosal2024Multivariate,Matabuena2021Glucodensities,Matabuena2024Glucodensity}. In a continuous time setting, the time-varying probability distribution provides a more natural biomarker to characterize the evolution of glucose metabolism \citep{Katta2024Interpretable,Park2025Beyond}. Compared to conventional CGM summary statistics, such as mean glucose or time-in-range metrics, this representation conveys richer information by capturing the full spectrum of low, moderate and high glucose values within a unified functional profile \citep{Katta2024Interpretable}.

More broadly, distributional data analysis \citep{Ghosal2023Distributional,Szabo2016Learning} is an emerging area that treats probability distributions, or collections of them, as statistical objects for unsupervised and supervised learning, including the prediction of clinical outcomes \citep{Matabuena2021Glucodensities}. Biomedical applications are among their most prominent use cases. In digital health, measurements collected by continuous glucose monitors, accelerometers, or imaging modalities such as functional magnetic resonance imaging \citep{https://doi.org/10.1002/hbm.26112} are increasingly represented through empirical distributions that serve as latent descriptions of underlying physiological processes \citep{Ghosal2025Distributional,Ghosal2024Multivariate,Matabuena2024Glucodensity,Matabuena2022Physical}. In recent years, several regression frameworks have been proposed in which predictors, responses, or both are represented as probability distribution functions \citep{matabuena2023distributional,Ghosal2025Distributional,Ghosal2024Multivariate,Ghosal26Survival,matabuena2025predicting}. A related line of work represents modeling probability distributions as random objects in metric spaces and develops statistical procedures for that general setting (see, e.g., \cite{Lugosi2024Uncertainty}). Despite this progress, there remains no general framework for modeling moderate- to high-dimensional distributions that simultaneously offers flexibility and interpretability. The methodology introduced here is intended to help bridge this gap.

\subsection*{Contributions}

The main contributions of this paper are as follows:

\begin{enumerate}
\item We propose a continuous-time model for evolving multivariate probability distributions based on a shared Gaussian mixture representation. A neural ODE describes the evolution of the mixture weights, providing a low-dimensional and interpretable representation of distributional changes.

\item We develop an estimation procedure based on maximum mean discrepancy \citep{gretton2012kernel}. For Gaussian mixtures and Gaussian kernels, the objective admits closed-form expressions, allowing us to estimate the shared mixture components and time-dependent weights without likelihood-based inference. At each time point, the empirical loss function takes the form of a V-statistic \citep{serfling2009approximation}.

\item We apply the method to continuous glucose monitoring data from a clinical trial involving young children with type 1 diabetes \cite{wadwa2023trial}. We study individual and group-level changes in glucose distributions, considering both glucose values and their rate of change. The Gaussian dictionary is shared across individuals, while the mixture-weight dynamics are individual specific.

\item We propose a practical approach for deriving time-varying means and covariance matrices from the fitted Gaussian mixture distributions. We introduce functional MMD tests to compare treatment groups, both directly and relative to baseline, with statistical inference based on a participant-level wild bootstrap. Component-specific tests, together with the aggregated Cauchy association test (ACAT) \citep{Liu02012020}, provide additional insights into how treatment effects vary across mixture components and evolve in the course of the clinical trial.

\end{enumerate}

In the clinical figures, \Treatment{} is shown in green and \Control{}
in red. The mixture components are identified consistently by
\CompOne{}, \CompTwo{}, \CompThree{}, \CompFour{}, and
\CompFive{}.







\begin{figure}[htbp]
    \centering

    \begin{subfigure}[t]{0.48\textwidth}
        \centering
        \includegraphics[width=\linewidth]{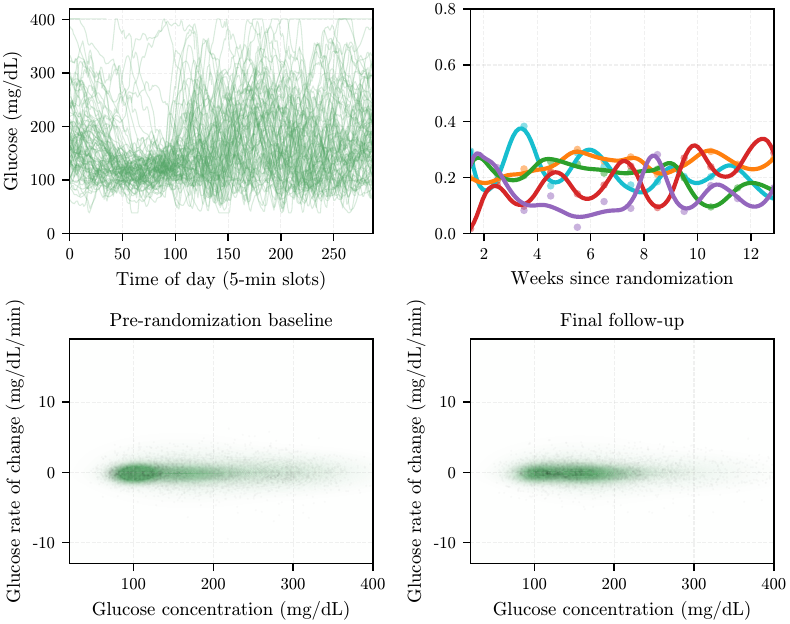}
        \caption{Participant 16 (\Treatment{}).}
    \end{subfigure}
    \hfill
    \begin{subfigure}[t]{0.48\textwidth}
        \centering
        \includegraphics[width=\linewidth]{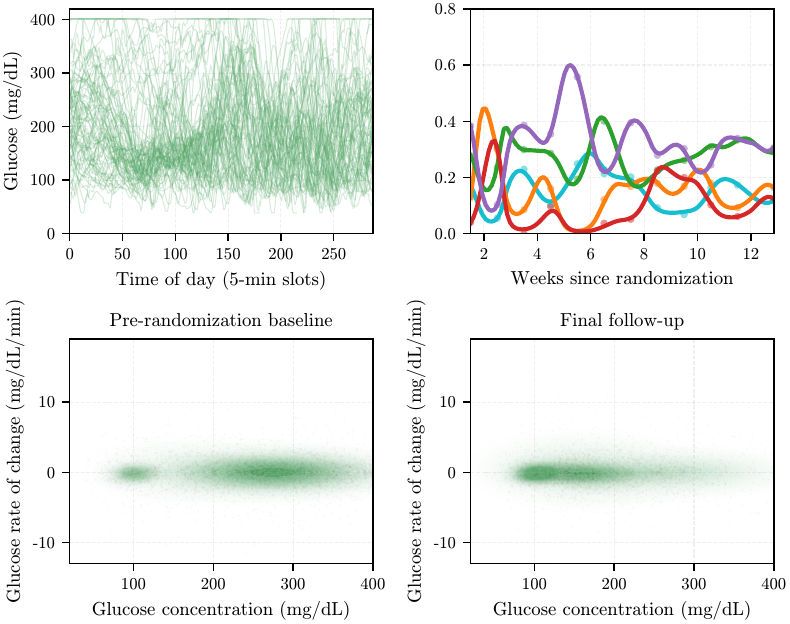}
        \caption{Participant 26 (\Treatment{}).}
    \end{subfigure}

    \vspace{0.6em}

    \begin{subfigure}[t]{0.48\textwidth}
        \centering
        \includegraphics[width=\linewidth]{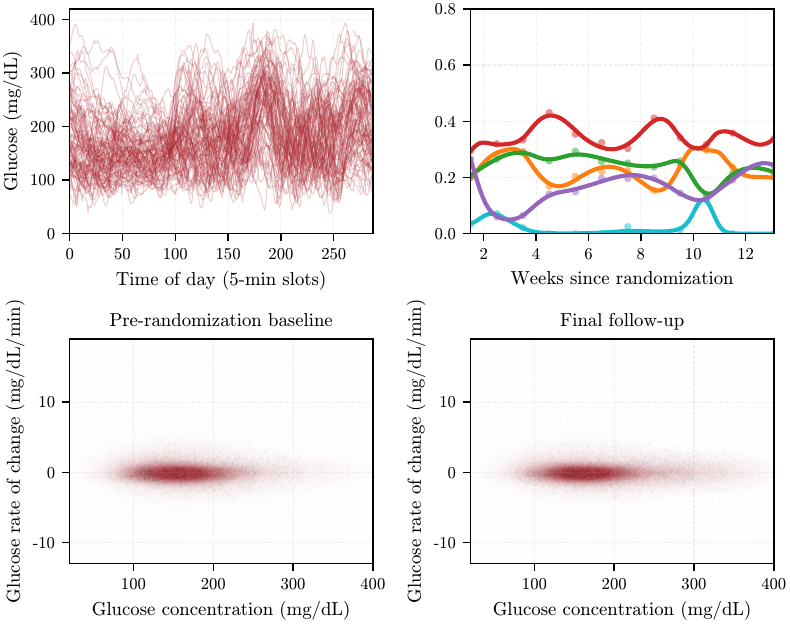}
        \caption{Participant 39 (\Control{}).}
    \end{subfigure}

    \caption{Illustrative individual fits for the bivariate Gaussian time-dependent mixture model, with glucose concentration and glucose rate of change and $K=5$ shared Gaussian components. For each participant, the upper-left panel shows the CGM profiles, the upper-right panel shows the discrete local mixture weights together with the neural-ODE trajectories, and the lower panels compare the fitted pre-randomization baseline density with the density at the exact last neural-ODE state of follow-up. Baseline weights are estimated separately using the same dictionary and participant-specific fitting settings; the baseline is not a neural-ODE training observation. The upper panels describe post-randomization follow-up only. The time axis is expressed in weeks since randomization; missing weekly blocks are not imputed.}
    \label{fig:representative_individuals_d2k5}
\end{figure}

\section{Clinical motivation: a randomized trial in young children with type 1 diabetes}\label{sec:case}

Optimizing insulin therapy for young children with type 1 diabetes remains particularly challenging \citep{schoelwer2024use,ware2024eighteen}. In children younger than 6 years of age, insulin requirements are typically low, while food intake, meal timing, and physical activity can be unpredictable, complicating insulin dosing decisions and glycemic management. Moreover, day-to-day fluctuations in insulin requirements and individual differences in glycemic responses may further complicate treatment optimization and the achievement of glycemic targets. Although advances in diabetes technology have improved treatment options, relatively few hybrid closed-loop insulin delivery systems have received approval from the U.S. Food and Drug Administration for use in children under 6 years of age \citep{kitagawa2025artificial,hughes2023digital}.

 Our paper is motivated by a randomized clinical trial evaluating the t:slim X2 insulin pump with Control-IQ Technology (Tandem Diabetes Care) in young children with type 1 diabetes \citep{wadwa2023trial}. The closed-loop system uses CGM measurements collected every five minutes to guide automated insulin delivery through adjustments to basal insulin rates and automated correction boluses. Although closed-loop insulin delivery has been extensively investigated in older children, adolescents, and adults \citep{beck2023meta,stahl2025efficacy}, evidence in children younger than 6 years of age has been comparatively limited. The trial enrolled 102 children aged 2--5 years, randomly assigned in a 2:1 ratio to closed-loop control ($n=68$) or standard care ($n=34$) during an initial 13-week randomized phase, followed by a 13-week extension phase. The original study demonstrated improvements in glycemic control with closed-loop therapy, including a greater proportion of time spent within the target glucose range and reductions in mean glucose and glycosylated hemoglobin (HbA1c). However, the between-group difference in time spent below 70 mg/dL was not statistically significant.

Despite their clinical relevance, conventional CGM summaries, such as mean glucose and time in range, provide a limited picture of the underlying glucose homeostasis status \citep{Matabuena2024Glucodensity}. In particular, classical CGM summaries do not fully characterize the distribution of glucose values or their temporal dynamics, including the direction and rate of glucose changes. Such information may be relevant for understanding individual differences in glucose regulation and responses to insulin therapy. Recent methodological developments have introduced \emph{glucodensity}, a functional representation of CGM measurements based on the marginal distribution of glucose concentrations \citep{Matabuena2021Glucodensities}. Although this representation provides a more comprehensive characterization of glucose exposure than classical scalar summaries, it does not directly capture how glucose concentrations change over time. To address this limitation, multivariate glucodensity representations jointly characterize glucose concentrations $G$, their first derivative (rate of change), and potentially their second derivative (acceleration) through the joint distributions of $(G,\dot G)$ or $(G,\dot G,\ddot G)$, where $\dot G(u)=dG(u)/du$ and $\ddot G(u)=d^2G(u)/du^2$, with $u$ denoting the within-participant measurement time \citep{Matabuena2024Glucodensity}. These multivariate representations provide a richer characterization of glycemic behavior by integrating local glucose dynamics and marginal distributional patterns into a single joint functional profile.

Motivated by these developments, we reanalyze the randomized clinical trial data using a multivariate glucodensity framework to investigate how the distributions of glucose levels and their rates of change evolve under closed-loop insulin delivery compared with standard care. Specifically, we represent each participant's weekly glucose distribution using a Gaussian mixture model with $K=5$ components and a shared Gaussian dictionary across participants and follow-up times. Participant-specific, time-varying mixture weights characterize the longitudinal evolution of individual glucose distributions over the 13-week post-randomization phase. Detailed inclusion criteria and preprocessing procedures are provided in the Appendix C.

Our primary objective is to investigate whether this longitudinal distributional representation provides additional, clinically interpretable information about differences between treatment groups beyond conventional CGM summary measures. In particular, we examine whether incorporating glucose dynamics reveals changes in glucose regulation that may not be apparent from marginal glucose distributions or scalar outcomes alone. \Cref{fig:representative_individuals_d2k5} illustrates the proposed representation for three representative participants, two receiving closed-loop therapy and one receiving standard care, showing their raw CGM trajectories, estimated mixture-weight trajectories, and fitted bivariate densities before randomization and at each participant's last fitted follow-up state.

\section{Methodology}\label{sec-meth}

We introduce a Gaussian mixture framework for learning time-varying distributions from discretely observed longitudinal data. The method models distributional dynamics in continuous time, yields interpretable subject-specific trajectories, and supports cross-group comparisons and downstream statistical inference. While we illustrate the method on digital health data, the framework applies broadly to longitudinal settings with continuously evolving distributions.

\subsection{Methodological background}


We first recall the two tools used by the estimator. 

\subsubsection{Maximum mean discrepancy}
To compare probability distributions $\mu$ and $\nu$ on a common measurable space $\mathcal{X}$, we use the \emph{maximum mean discrepancy} (MMD). To do so, MMD represents each distribution as an element in a Hilbert space \citep{gretton2012kernel,muandet2017kernel}.

Let $\mathcal{H}$ be the reproducing kernel Hilbert space (RKHS) induced by a positive definite kernel $k \colon \mathcal{X}\times\mathcal{X} \to \R$. If $\mu$ is a probability distribution such that $\int k(x,x)\diff\mu(x)<\infty$ (e.g., if $k$ is bounded), then the \emph{kernel mean embedding} of $\mu$, denoted by $m_{\mu} \in \mathcal{H}$, is well defined:
\[  m_{\mu}(\cdot) \coloneqq \int k(\,\cdot\,,x)\diff \mu(x).
\]
Formally, MMD evaluates the distance between two distributions as the $\mathcal{H}$-norm of the difference between their embeddings:
\begin{equation}\label{eq:mmd_def}
\operatorname{MMD}^{2}(\mu,\nu)
  \coloneqq
  \lVert m_{\mu} - m_{\nu} \rVert_{\mathcal{H}}^{2}
  =
  \mathbb{E}\left[k(X,X')\right]
  +
  \mathbb{E}\left[k(Y,Y')\right]
  -
  2\,\mathbb{E}\left[k(X,Y)\right],
\end{equation}
where $X,X'\stackrel{\text{i.i.d.}}\sim \mu$ and $Y,Y'\stackrel{\text{i.i.d.}}\sim \nu$.

If the mapping $\mu \mapsto m_{\mu}$ is injective, equivalently if the kernel $k$ is \emph{characteristic}, then MMD defines a valid metric, satisfying $\operatorname{MMD}(\mu,\nu)=0$ if and only if $\mu=\nu$, see \citep{sriperumbudur2011universality,sejdinovic2013equivalence}.

Throughout the paper, we fix $\mathcal{X}=\R^d$ and use the Gaussian kernel:
\begin{equation}\label{eq:Gkernel}
k(x,y)\coloneqq \exp\left(-\frac{\|x-y\|_2^2}{2\sigma^2}\right),
\end{equation}
where $\sigma>0$ is a bandwidth parameter. Since the Gaussian kernel is bounded and characteristic, the kernel mean embedding is well-defined, and the associated MMD defines a metric on probability distributions. 
The kernel bandwidth is calibrated at the
subject level and then held fixed across that subject's observation times.
This avoids introducing artificial temporal variation through a changing
kernel scale. The precise bandwidth rule used in our clinical application
and its sensitivity analysis are described below in \Cref{ss.discrete.mmd}.

\noindent\textbf{Intuition.} MMD measures how far apart two distributions are in the feature space induced by the kernel \(k\). In our setting, at each observed time point \(t_i\) in the grid \eqref{eq:time.grid}, we use MMD to compare the empirical distribution associated with the observations in \eqref{eq:obs.DH} to a fitted Gaussian mixture distribution.

\noindent\textbf{Why MMD?}
We use MMD for three reasons. First, for Gaussian components and a Gaussian kernel, the objective admits closed-form expressions, and weight estimation reduces to a regularized quadratic program on the simplex. Second, MMD allows us to fit distributions at individual time points without specifying a likelihood for the full longitudinal process. Third, depending on the choice of kernel and its bandwidth, MMD-based estimation can be less sensitive to outliers than likelihood-based methods. 
 

\subsubsection{Neural ODEs} We replace discrete layers by the continuous evolution of a hidden state
\[
\dot z(t)=v_\phi\left(z(t),t\right),\qquad z(t_0)=z_0,
\]
where $z(t)\in\R^h$ and \(v_\phi:\R^h\times[t_0,T]\to\R^h\) is a learnable vector field, typically parameterized by a multilayer perceptron. 

Neural ODEs have been widely used to model latent trajectories in continuous time \citep{kidger2020neural,Rubanova2019Latent,Jia2019Neural}, including in biomedical applications \citep{qian2021integrating}. In our framework, the latent path $z(t)$ parameterizes the mixture weights over time, providing a smooth, continuous-time interpolation of the weights estimated on the discrete grid \eqref{eq:time.grid} without imposing a rigid parametric form.

\noindent\textbf{Why a neural ODE?}
Our goal is to recover a continuous-time trajectory of distributions rather than isolated snapshots at discrete time points. Achieving continuity through highly overlapping windows repeatedly reuses observations, inducing artificial serial correlation and smoothing dictated by the window width. To avoid this, we estimate the $K$ mixture weights over non-overlapping windows and govern their continuous-time evolution with a neural ODE. Unlike traditional spline smoothers that require a fixed basis, this dynamical formulation jointly models the $K$ weights, seamlessly interpolates across irregularly spaced windows, and adapts to varying sampling frequencies.

\subsection{Proposed framework}
\label{sec:ourmodel}
We approximate each marginal $f_t$ by a Gaussian mixture of the form
\begin{equation}\label{eq:model}
  f_{\theta(t)}(x)
  = \sum_{s=1}^{K} \alpha_s(t)\,
    \mathcal{N}(x \mid m_s, \Sigma_s),
\end{equation}
where $m_s \in \R^d$ and $\Sigma_s \in \mathscr{M}_{d \times d}(\R)$ denote
the mean vector and symmetric positive-definite covariance matrix of
component $s$, respectively. The mixture weights form a continuous-time trajectory 
$\alpha(t) = [\alpha_1(t), \ldots, \alpha_K(t)]^\top$ taking values in
the probability simplex
\begin{equation}\label{eq:simp}
  \Delta^{K-1}
  \coloneqq
  \left\{ w \in \R_{\geq 0}^K \;\colon\;
  \sum_{s=1}^{K} w_s = 1 \right\}.
\end{equation}
Therefore, the model is parametrized at any fixed time $t$ by
\[
\theta(t) \coloneqq
\left(\alpha_1(t),\dots,\alpha_K(t),\,
           m_1,\dots,m_K,\,
           \Sigma_1,\dots,\Sigma_K\right),
\]
with $K(1 + d + d(d+1)/2) - 1$ free scalar parameters. The component means and covariance matrices are fixed over time and shared across participants, whereas each participant has a separate, continuous path of mixture weights.

As shown in Appendix A, under mild regularity
conditions, mixtures of the form \eqref{eq:model} can approximate
any continuous curve of densities in $L^1(\R^d)$ uniformly in time.
In practice, we regard $K$ as a resolution parameter controlling
the trade-off between flexibility and interpretability,
selecting its value based on the specific application
(cf.\ \cref{sec:CGM}).

\subsubsection{Discrete-time estimation of mixture distributions via MMD}\label{ss.discrete.mmd}

Given the observation grid \(\tau_m=\{t_0,\ldots,t_m\}\) and the sample \(X_{t_i,1},\ldots,X_{t_i,N_i}\sim \mu_{t_i}\) (recall \eqref{eq:time.grid} and \eqref{eq:obs.DH}), we define the empirical distribution by
\[
\mu_{t_i,N_i} \coloneqq \frac{1}{N_i}\sum_{j=1}^{N_i}\delta_{X_{t_i,j}},\qquad i=0,\dots,m.
\]
At each discrete time step \(t_i\), we fit a static Gaussian mixture 
\[
f_{\theta_i}(x) = \sum_{s=1}^{K} \alpha_{i,s}\,\mathcal N(x \mid m_s, \Sigma_s),
\qquad \alpha_i=(\alpha_{i,1},\dots,\alpha_{i,K})\in\Delta^{K-1},
\]
by minimizing \(\operatorname{MMD}^2(\mu_{t_i,N_i},\mu_{\theta_i})\), where \(\mu_{\theta_i}\) is the distribution with density \(f_{\theta_i}\).  

The Gaussian component parameters form a shared dictionary and are estimated
globally. Once this dictionary has been fitted, the local estimation step
updates only the mixture weights at each observation time. This separation
preserves a common component interpretation across time and, in longitudinal
applications, across individuals. In the clinical analysis, each participant uses one Gaussian-kernel bandwidth
across all usable weekly blocks. The bandwidth is set by the median heuristic \citep{garreau2018largesampleanalysismedian}
applied to that participant's pooled observations and is held fixed over time.
For the bivariate representation, distances are computed after the common
global standardization of glucose and glucose rate of change. The reported fit
uses the reference bandwidth; Appendix E examines multipliers
\(c\in\{0.1,0.5,1,2,5\}\).

By expanding the squared MMD, the objective function reduces to a convenient quadratic form in the mixing weights $\alpha_i$:
\begin{equation}\label{eq:obj}
   \operatorname{MMD}^2(\mu_{t_i,N_i},\mu_{\theta_i})
 = \alpha_i^\top I_i \, \alpha_i - 2\,\alpha_i^\top J_i + D_i,  
\end{equation}
where $D_i=\frac{1}{N_i^2}\sum_{j=1}^{N_i}\sum_{\ell=1}^{N_i} k_i(X_{t_i,j}, X_{t_i,\ell})$ depends solely on the empirical data and can therefore be omitted from the minimization, while the matrix $I_i \in \mathscr{M}_{K\times K}(\R)$ and the vector $J_i \in \R^K$ admit closed-form expressions for their entries:
\begin{align*}
(I_i)_{s,r} &=\frac{(\sigma_i^2)^{d/2}}{\sqrt{\det(\Sigma_s + \Sigma_r+\sigma_i^2 \mathsf{Id})}}
\exp\left(-\frac{1}{2} (m_s-m_r)^\top (\Sigma_s+\Sigma_r+\sigma_i^2 \mathsf{Id})^{-1} (m_s-m_r)\right),\\
(J_i)_s &= \frac{1}{N_i}\sum_{j=1}^{N_i} \frac{(\sigma_i^2)^{d/2}}{\sqrt{\det(\Sigma_s+\sigma_i^2 \mathsf{Id})}}
\exp\left(-\frac{1}{2} (X_{t_i,j}-m_s)^\top (\Sigma_s+\sigma_i^2 \mathsf{Id})^{-1}(X_{t_i,j}-m_s)\right). 
\end{align*}

\paragraph*{Optimization}

We estimate the shared dictionary and local weights by alternating two steps. With the dictionary fixed, we update each weight vector by solving a regularized quadratic program on the simplex. With the weights fixed, we update the component means and covariance matrices by (Adam) gradient descent on the MMD objective. We initialize the dictionary using $K$-means and repeat these steps until a prescribed stopping criterion is met.

For a fixed dictionary, the local weight estimate at time \(t_i\) is
\begin{equation}\label{eq:local_qp}
\widehat\alpha_i
=
\underset{\alpha\in\Delta^{K-1}}{\operatorname{argmin}}
\left\{
\alpha^\top I_i\alpha
-
2\alpha^\top J_i
+
\lambda\|\alpha\|_2^2
\right\}.
\end{equation}

The ridge parameter improves numerical conditioning and stabilizes the
solution when dictionary components are nearly collinear in RKHS feature
space.

\subsubsection{Continuous-time dynamics of mixture weights}

Once the discrete-time weights $\{\widehat\alpha_i\}_{i=0}^{m}$ are estimated, we model their continuous-time evolution between observation times using a neural ODE. This approach captures the joint dynamics on the probability simplex without prescribing a parametric temporal trend. Specifically, we evolve an unconstrained logit state $z(t)\in\R^K$ according to
\begin{equation}\label{eq:node_latent}
\dot z(t)
=
v_\phi\left(z(t),t\right),
\qquad t\in[0,T],
\end{equation}
where $v_\phi$ is a neural network vector field. The corresponding mixture weights are obtained via
\begin{equation}\label{eq:alpha_from_logits}
\alpha(t)
=
\operatorname{softmax}\left(
z(t)-\frac{1}{K}\mathbf 1\mathbf 1^\top z(t)
\right)
\in\Delta^{K-1},
\end{equation}
where the centering selects a zero-mean representative of the logit vector without changing the resulting mixture weights. We initialize $z(t_0) = z_0$ from the first estimated weight vector and optimize $z_0$ jointly with the vector-field parameters $\phi$.

\paragraph*{Temporal input representation}

Standard neural ODEs incorporate time $t$ directly as an input to the vector field \citep{Jia2019Neural}. While this suffices for universal flow approximation \citep{Li2024Universal,AlvarezLopez2026Constructive}, explicitly encoding temporal structure often improves flexibility for empirical fit. We therefore augment the time input via the feature map
\begin{equation}\label{eq:time_features}
\psi(t)
=
\left(
t,t^2,\ldots,t^p,
\sin(2\pi t),\cos(2\pi t),
\ldots,
\sin(2\pi q t),\cos(2\pi q t)
\right),
\end{equation}
yielding the vector field $v_\phi(z(t),\psi(t))$. Polynomial terms capture smooth non-linear trends, whereas Fourier features represent oscillatory behavior. Related coordinate and time embeddings have been used to increase the representational capacity of neural networks for low-dimensional inputs \citep{tancik2020fourier,kazemi2019time2vec}.

In the clinical experiments (\Cref{sec:CGM}), we combine a linear time trend with two Fourier harmonics ($p=1, q=2$). An in-sample sensitivity analysis comparing polynomial and Fourier specifications is provided in Appendix E. By contrast, the simulation study (Appendix B) uses the unaugmented time input ($p=1, q=0$).

\paragraph*{Training objective}

The neural ODE is trained to fit a continuous trajectory to the discrete weight estimates. We minimize the cross-entropy between the target and fitted simplex trajectories:
\begin{equation}\label{eq:ode_loss}
\mathcal L(\phi,z_0)
=
-\sum_{i=0}^{m}
\sum_{s=1}^{K}
\widehat\alpha_{i,s}
\log\!\left(
\alpha_s(t_i;\phi,z_0)
\right).
\end{equation}
This objective penalizes discrepancies between the discrete MMD weight distributions and their continuous-time counterparts at the observation times.

\paragraph*{Label identifiability}
Once the shared dictionary has been estimated and its component ordering fixed, the labels remain aligned across participants and times. The usual global label-permutation invariance of a Gaussian mixture remains, but it does not affect comparisons of weights under the fixed dictionary.

\section{Simulation study}

We evaluate finite-sample density estimation accuracy by comparison with several established estimators. In contrast to competing approaches, our method provides an interpretable representation through mixture weights $\alpha_s(\cdot)$ with $s=1,\ldots,K$. We consider both univariate ($d=1$) and high-dimensional ($d=10$) settings to investigate how estimation accuracy varies with dimensionality, with the latter serving as a methodological stress test rather than a direct representation of the clinical application. Furthermore, while the data-generating process allows the Gaussian component means and covariance matrices to evolve over time, our fitted model relies on a fixed Gaussian dictionary. The simulations therefore also evaluate the performance of our approach under this form of model misspecification. Further details are provided in Appendix B.

\section{Case study: longitudinal glucose distributions in a randomized trial}
\label{sec:CGM}

The objective of this case study is to investigate whether the proposed time-varying multivariate distributional model can use the thousands of glucose measurements collected by CGM devices to complement conventional scalar analyses and provide additional clinical findings. Although standard CGM metrics are naturally embedded within the glucodensity framework \citep{Matabuena2021Glucodensities}, our primary objective is to determine whether multivariate functional representations can reveal clinically relevant aspects of glucose regulation that are not fully captured by conventional summaries, particularly those related to glucose dynamics.

From a modeling perspective, our goal is to illustrate how the proposed framework characterizes longitudinal differences in glycemic profiles between the two study arms, particularly by incorporating glucose dynamics through rates of change.


\subsection*{Preliminaries and scientific questions}

As described in \Cref{sec:case}, our analysis is based on the randomized clinical trial ``Trial of Hybrid Closed-Loop Control in Young Children with Type 1 Diabetes'', published in the \emph{New England Journal of Medicine} \citep{wadwa2023trial}. The study evaluated hybrid closed-loop insulin delivery in children under 6 years of age with type~1 diabetes mellitus.  We focus on the 13-week randomized phase, representing each participant's CGM measurements as a sequence of weekly probability distributions. To illustrate the continuous-time model introduced in \Cref{sec-meth}, we primarily consider the bivariate setting (\(d=2\)), in which glucose concentration and its rate of change constitute the two coordinates. A corresponding univariate analysis (\(d=1\)), based solely on glucose concentration, is presented in Appendix E to assess the sensitivity of the results to the choice of distributional representation.





Our analysis aims to address three main clinical questions:

\begin{enumerate}
    \item How do glucose profiles differ between the treatment and control groups during the randomized phase, both in their overall distributions and in changes relative to baseline?

    \item How do treatment-related differences in glucose regulation evolve over time? Can the longitudinal analysis identify changes in treatment response that would be missed by conventional endpoint assessments?

    \item Beyond differences in glucose concentrations, does treatment influence how rapidly glucose levels rise and fall? Does incorporating glucose dynamics provide information about differences in treatment response between groups?
\end{enumerate}

\subsection*{Modeling the joint distribution of glucose and its rate of change}



Let $G_i(u)$ denote the CGM measurement of participant $i$ at time $u\in[0,\tau_i]$, and let $\dot G_i(u)$ denote its rate of change. For both randomized arms, CGM measurements are recorded at nominal five-minute intervals and grouped into weekly distributions, as described in Appendix C. For participant $i$ and usable weekly block $t$, let $n_{it}$ be the number of usable bivariate observations. We consider the sample
\[
(G_{it\ell},\dot G_{it\ell}), \qquad \ell=1,\dots,n_{it},
\]
where $G_{it\ell}$ is the glucose measurement at the $\ell$th usable observation in the block and the first derivative is estimated by the centered finite difference
\[
\dot G_i(u_j)
=\frac{G_i(u_{j+1})-G_i(u_{j-1})}{u_{j+1}-u_{j-1}},
\]
with $u$ measured in minutes. Derivatives are computed only from observations with sufficiently regular sampling intervals, without interpolation across gaps. The precise criteria are given in Appendix C.


Our target is the joint distribution
\[
F_{it}(g,v)
=
\mathbb{P}\left(G_{it}\le g,\;\dot G_{it}\le v\right),
\]
together with its associated density $f_{it}(g,v)$.

To obtain a representation that is both computationally tractable and directly comparable across participants, we approximate $f_{it}$ using a dynamic Gaussian mixture model,
\[
\widehat f_{it}(g,v)
=
\sum_{s=1}^{K}\widehat\alpha_{it,s}\,
\mathcal N\left((g,v)\mid m_s,\Sigma_s\right),
\]
where $\mathcal N(\cdot\mid m_s,\Sigma_s)$ denotes the bivariate Gaussian density with mean $m_s\in\R^2$ and positive-definite covariance matrix $\Sigma_s\in\mathscr{M}_{2\times2}(\R)$. The component-specific parameters $\{(m_s,\Sigma_s)\}_{s=1}^K$ are shared across participants and time windows, whereas the mixing weights $\widehat\alpha_{it,s}$ are allowed to vary across participants and over time, subject to
\[
\widehat\alpha_{it,s}\ge 0,
\qquad
\sum_{s=1}^{K}\widehat\alpha_{it,s}=1.
\]

For the bivariate analysis, we fix $K=5$ to balance flexibility and interpretability. We draw the reference sample for the shared Gaussian dictionary, without using treatment labels, from eligible observations in all post-randomization blocks with at least 72 hours of CGM across the 102 randomized participants.

Glucose and rate of change are standardized using the means and standard deviations of the calibration pool before MMD estimation. Each participant uses a single Gaussian-kernel bandwidth, computed from the pooled participant data and held constant across the usable weekly windows, with multiplier \(c=1\). The local MMD weight estimates use ridge parameter \(\lambda=0.01\).

The resulting discrete weekly weights are interpolated using a
participant-specific neural ODE. In the clinical application, we use
a five-dimensional logit state and a neural ODE vector field with
width 64 and $N_{\mathrm{layers}}=1$. The logits are centered and
transformed into mixture weights using the softmax function.
For normalized time $\tilde t\in[0,1]$, the temporal input is
$[\tilde t,\sin(2\pi\tilde t),\cos(2\pi\tilde t),\sin(4\pi\tilde t),\cos(4\pi\tilde t)]$,
combining a linear time term with two Fourier harmonics.
The results and functional comparisons are reported in weeks since randomization.
The model is trained using cross-entropy loss on the simplex-valued
weights, with a learning rate of $0.01$. Additional implementation
details are provided in Appendix D.
Missing weekly distributions are not imputed as training targets.
Instead, the neural ODE provides a continuous trajectory within
each participant's fitted support, without extrapolation beyond
that interval. For each participant, the initial state corresponds to
the first usable post-randomization observation, not to the
pre-randomization baseline.

From a clinical perspective, the five mixture components can be
interpreted as a dictionary of glycemic regimes, characterized
by their estimated means and covariance structures in the
$(\text{glucose},\text{rate-of-change})$ space.
These components can be ordered from more to less
favorable glycemic control profiles based on their clinical characteristics.



\subsection*{Temporal evolution of mixture weights by treatment group}

To assess how differences between groups evolve during follow-up,
\Cref{fig:trajectories_comparison} shows the estimated trajectories
of the mixture weights for the five components of the bivariate
model in the Treatment and Control groups. For each component,
the figure displays participant-specific trajectories alongside
the corresponding group mean trajectory. The mixture
weights $\widehat\alpha_{i,s}(t)$ reveal substantial differences between the
two treatment arms across the five components. These differences
are particularly evident when comparing the group-average
weight trajectories over time.

\begin{figure}[htbp]
    \centering

    \begin{subfigure}[t]{0.48\textwidth}
        \centering
        \includegraphics[width=\linewidth]
        {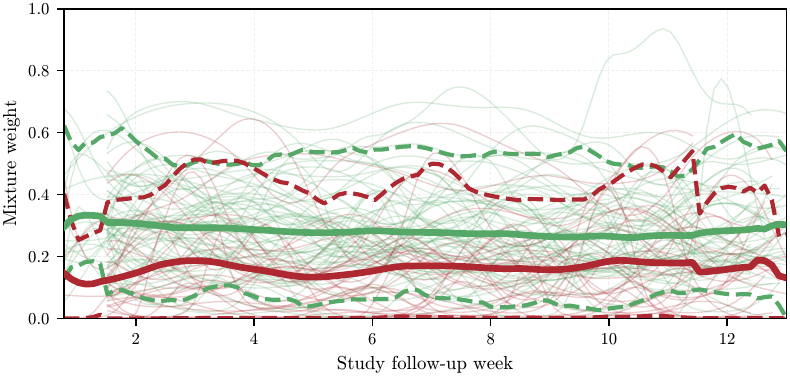}
        \caption{\CompOne}
    \end{subfigure}
    \hfill
    \begin{subfigure}[t]{0.48\textwidth}
        \centering
        \includegraphics[width=\linewidth]
        {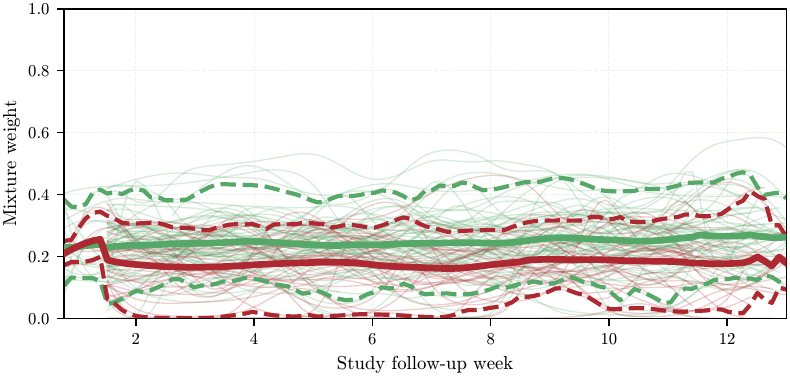}
        \caption{\CompTwo}
    \end{subfigure}

    \vspace{0.5em}

    \begin{subfigure}[t]{0.48\textwidth}
        \centering
        \includegraphics[width=\linewidth]
        {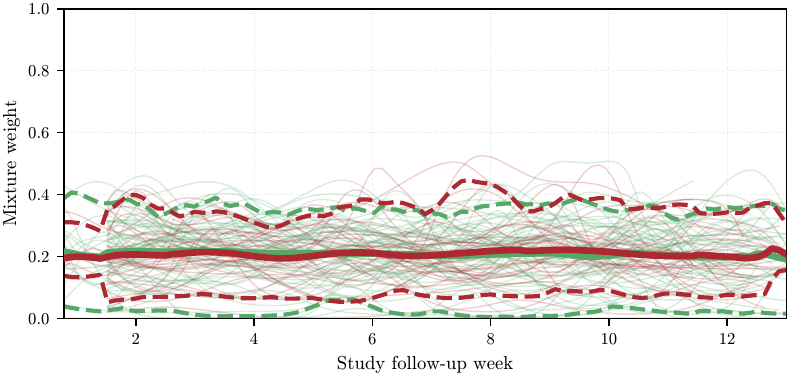}
        \caption{\CompThree}
    \end{subfigure}
    \hfill
    \begin{subfigure}[t]{0.48\textwidth}
        \centering
        \includegraphics[width=\linewidth]
        {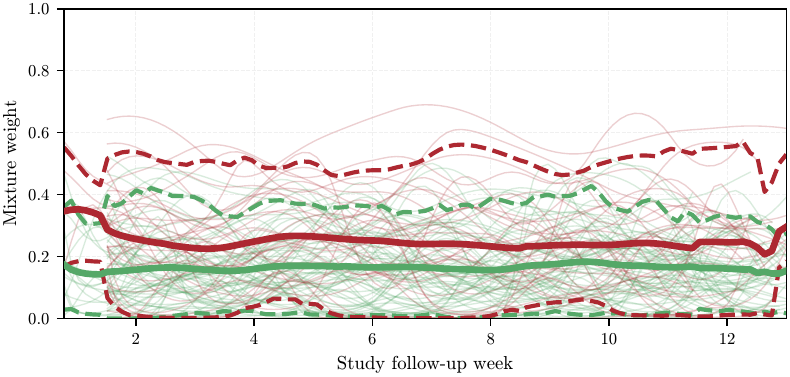}
        \caption{\CompFour}
    \end{subfigure}

    \vspace{0.5em}

    \begin{subfigure}[t]{0.48\textwidth}
        \centering
        \includegraphics[width=\linewidth]
        {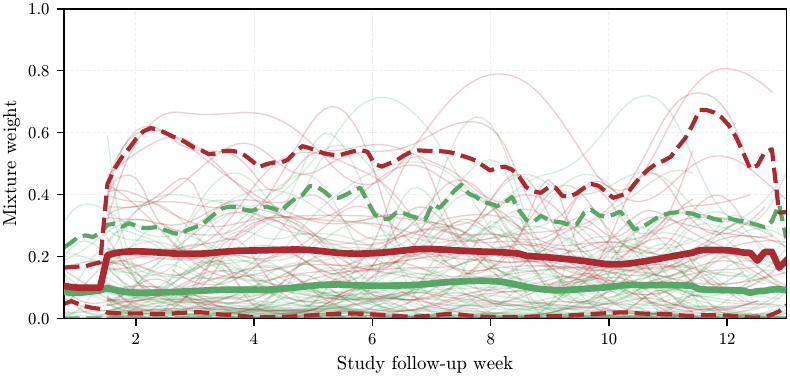}
        \caption{\CompFive}
    \end{subfigure}
\caption{Bivariate ($d=2$) mixture-weight trajectories in the randomized-phase analysis. Thin curves are participant-specific neural-ODE trajectories, thick curves are arm-specific available-case means for the \Treatment{} (green) and \Control{} (red) groups, and dashed curves mark the 5th and 95th percentiles across available participants in each arm. These percentile bands describe between-participant variation, not confidence intervals for the arm means. No extrapolation is used outside an individual's fitted support.}
    \label{fig:trajectories_comparison}
\end{figure}

\subsection*{Time-varying mixture moments}

Because the Gaussian dictionary is fixed, the fitted mixture weights also
induce exact time-varying moments of each participant's distribution.
For participant \(i\), the mean vector and covariance matrix are
\[
\widehat m_i(t)=\sum_{s=1}^{K}\widehat\alpha_{i,s}(t)m_s,
\]
and
\[
\widehat C_i(t)
=
\sum_{s=1}^{K}
\widehat\alpha_{i,s}(t)
\left[
\Sigma_s+
\{m_s-\widehat m_i(t)\}\{m_s-\widehat m_i(t)\}^{\top}
\right].
\]
These quantities summarize the fitted mean glucose level,
mean rate of glucose change, glucose variance, glucose--rate
covariance, and variance of the glucose rate of change.
Group-level summaries are obtained by averaging the
participant-specific moments among participants with available
trajectories at each time point, rather than by pooling
participants into a single mixture distribution.

\Cref{fig:global_moments_by_arm} shows the individual and group-average trajectories of these measures over follow-up. Children in the Treatment group generally have lower mean glucose levels and less glucose variability than those in the Control group. The glucose–rate covariance is predominantly negative over time, suggesting that higher glucose levels are generally associated with lower rates of glucose change throughout the clinical trial. The variance of the glucose rate of change provides an additional measure of glucose dynamics, describing how much the rate of increase or decrease fluctuates over time.  Group-average trajectories should be interpreted with some caution, as the number and composition of participants with available data may change over follow-up.


\begin{figure}[!htbp]
    \centering
    \begin{subfigure}[t]{0.3\textwidth}
        \centering
        \includegraphics[width=\linewidth]{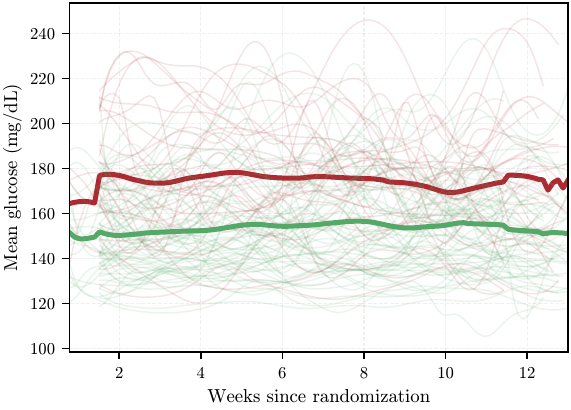}
        \caption{Mean glucose.}
    \end{subfigure}
    \hspace{0.04\textwidth} 
    \begin{subfigure}[t]{0.3\textwidth}
        \centering
        \includegraphics[width=\linewidth]{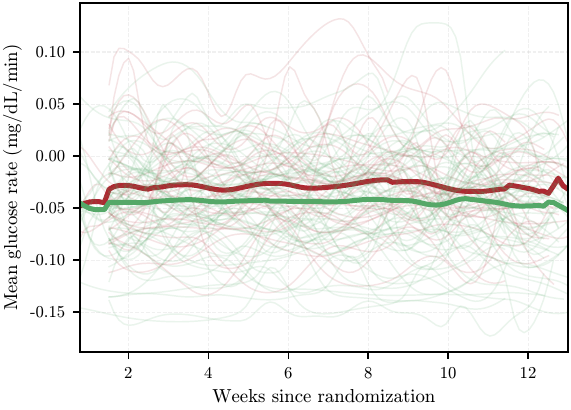}
        \caption{Mean glucose rate of change.}
    \end{subfigure}

    \vspace{1em} 

    \begin{subfigure}[t]{0.3\textwidth}
        \centering
        \includegraphics[width=\linewidth]{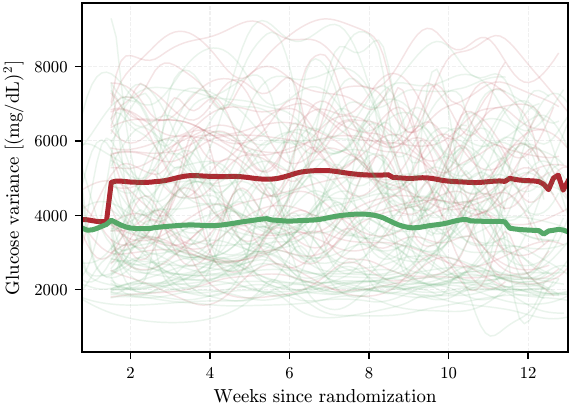}
        \caption{Glucose variance.}
    \end{subfigure}
    \hfill
    \begin{subfigure}[t]{0.3\textwidth}
        \centering
        \includegraphics[width=\linewidth]{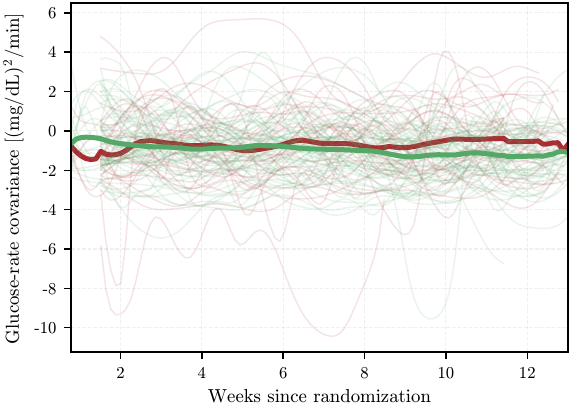}
        \caption{Glucose--rate covariance.}
    \end{subfigure}
    \hfill
    \begin{subfigure}[t]{0.3\textwidth}
        \centering
        \includegraphics[width=\linewidth]{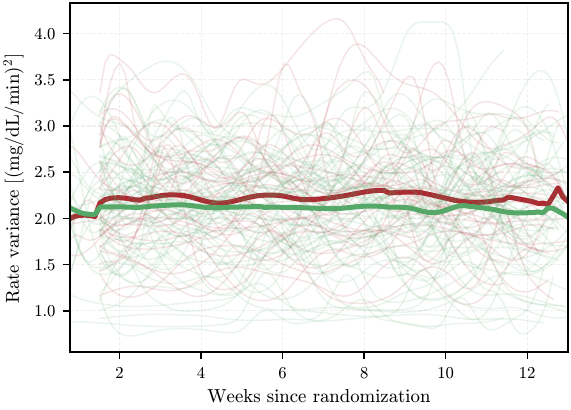}
        \caption{Rate-of-change variance.}
    \end{subfigure}

    \caption{Time-varying means, variances, and covariances implied by the fitted bivariate mixtures.
    Thin curves represent individual trajectories,
    while thick curves represent the \Treatment{}
    and \Control{} available-case arm averages. Covariance matrices
    are computed at the participant level before
    averaging and therefore do not represent
    covariances of pooled group distributions.}
    \label{fig:global_moments_by_arm}
\end{figure}

\subsection*{Fréchet summaries of distributional change}

To summarize distributional changes from before treatment to the end
of follow-up, we compute arm-specific 2-Wasserstein Fréchet means \citep{Frechet1948Les}
at two endpoints: the pre-randomization baseline, based on CGM
recorded during the preceding 14 days, and each participant's
last fitted follow-up state. The same 58 Treatment and 27 Control
participants contribute to both summaries. Final follow-up times
vary between participants.

Let $\mathcal I_a$ denote the participants with both a valid baseline and a fitted follow-up trajectory in arm $a\in\{0,1\}$,
where $0$ and $1$ correspond to the control and treatment arms,
respectively, and let
$\mu_{i,0}$ and $\mu_{i,1}$ be the fitted pre-randomization baseline
and final follow-up distributions of participant $i$, respectively, expressed in the common
globally standardized coordinates used for the bivariate fit. For
$r\in\{0,1\}$, the corresponding Fréchet mean is
\begin{equation}\label{eq:frechet_barycenter}
\bar\mu_{a,r}
\in
\operatorname*{arg\,min}_{\nu\in\mathcal P_2(\R^2)}
\frac{1}{|\mathcal I_a|}
\sum_{i\in\mathcal I_a}
W_2^2(\nu,\mu_{i,r}),
\end{equation}
where $\mathcal P_2(\R^2)$ is the space of probability measures
with finite second moments and $W_2$ denotes the
2-Wasserstein distance.

We approximate these $W_2$ barycenters numerically using samples from
the fitted mixtures, in the same standardized bivariate space. Gaussian
kernel density estimation is used to visualize the approximations.
Denoting the displayed density estimates by $\widehat f^{\mathrm{KDE}}_{a,r}$,
we also display the signed differences
\[
\widehat\Delta_a(x)
=
\widehat f^{\mathrm{KDE}}_{a,1}(x)-\widehat f^{\mathrm{KDE}}_{a,0}(x),
\qquad x\in\R^2.
\]
\Cref{fig:frechet_d2_summary}  displays the corresponding results for the two treatment arms. The post-baseline density plots illustrate the potential of our methodology for analyzing clinical trials in digital health through multivariate distributional representations. In the Control group, the distributions remain largely unchanged between the two time periods. In contrast, the Treatment group exhibits clear changes in the distributional patterns, with a redistribution of probability mass toward glucose levels within the range of 70--120 mg/dL. Moreover, the distribution of glucose rates becomes more concentrated around zero within this glucose range, indicating smaller glucose fluctuations and greater glycemic stability. Overall, by visualizing changes in the joint density functions, our approach provides a more comprehensive characterization of the treatment effect, revealing improvements in glycemic control that cannot be fully captured by the conventional scalar summaries recommended in current clinical guidelines.  Our approach provides a continuous-glucose analogue of time in range \citep{Battelino2023Continuous}, while also incorporating glucose dynamics into the CGM representation.

\begin{figure}[!htbp]
    \centering
    \includegraphics[width=0.96\textwidth]
    {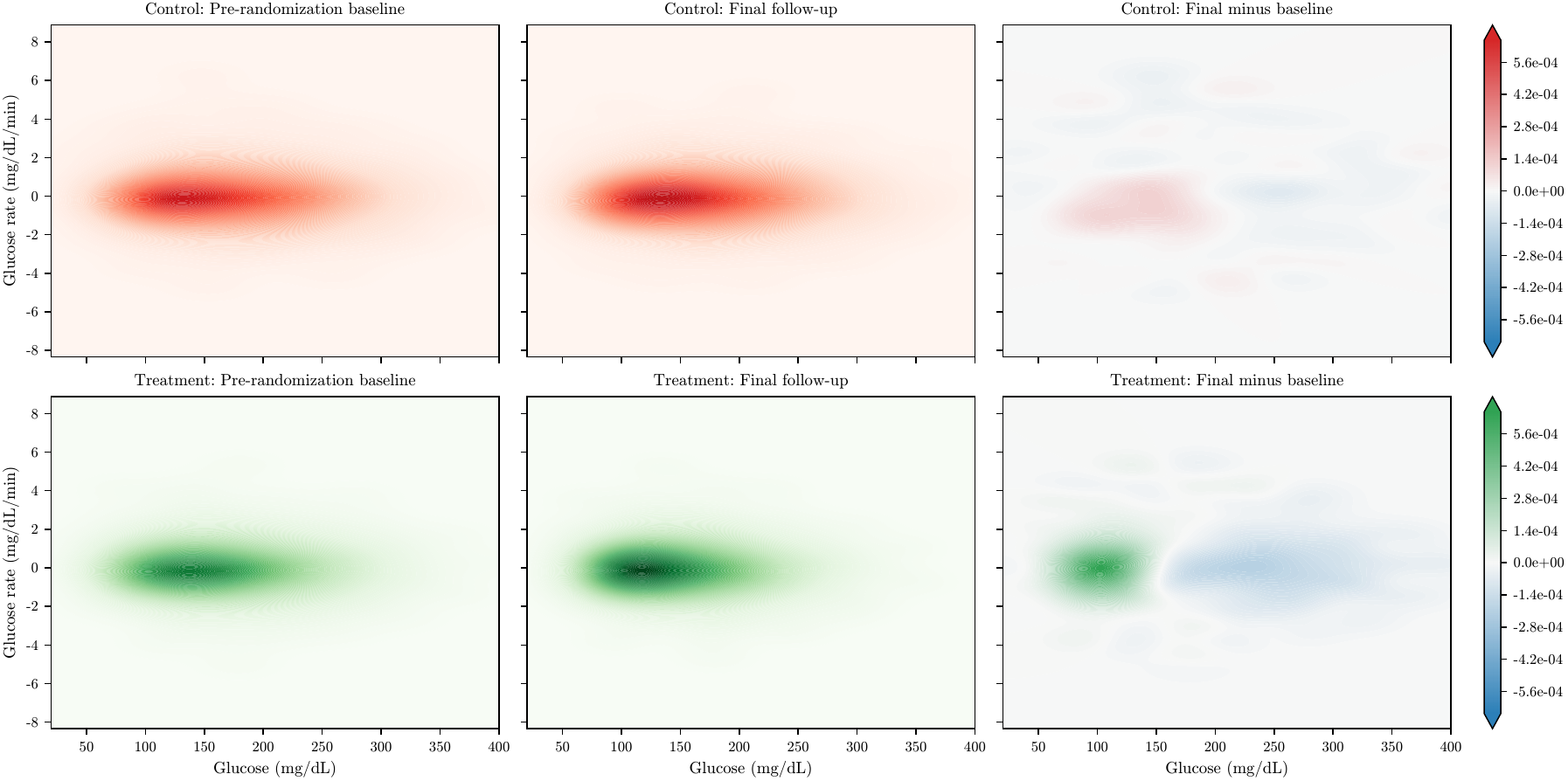}
\caption{Fréchet summaries of change from pre-randomization baseline to the final fitted follow-up state. Kernel density estimates of the numerically approximated 2-Wasserstein barycenters are shown for the \Control{} and \Treatment{} arms, using the same 27 Control and 58 Treatment participants at both endpoints with equal weights within arm. Signed panels show final minus baseline. Final times are participant specific; these are descriptive summaries, distinct from the functional tests.}
    \label{fig:frechet_d2_summary}
\end{figure}


\subsection*{Temporal inference for between-group differences}

To assess differences between the Treatment and Control groups over time,
we use a participant-level wild bootstrap described in the Supplemental Material on a common follow-up interval.
We select the longest continuous interval covered by at least 90\% of
fitted participants in each group, using only fitted time ranges and treatment
assignments. This yields weeks $[1.5,11.5]$, with 62 of 66 Treatment participants and
31 of 33 Control participants (93 of 99 in total) contributing throughout
the interval. The global functional test uses 101 equally spaced evaluation
points; the exploratory pointwise tests use the 11 weekly centers in this
interval. No extrapolation is used.

For the primary bivariate analysis, we consider two comparisons:
differences in the absolute mixture-weight trajectories and differences
in their changes from pre-randomization baseline. The latter uses
$\Delta\alpha_i(t)=\widehat\alpha_i(t)-\alpha_i^{(t_0)}$, where
$\alpha_i^{(t_0)}$ denotes the weights estimated from the 14 days
before randomization, requiring at least 168 hours of CGM data.
Both analyses use the same treatment-blind Gaussian dictionary.
Participants without a valid baseline are excluded only from the
baseline-centered analysis.

\Cref{tab:primary_functional_mmd} shows strong evidence of
between-group differences in both absolute trajectories
($\mathrm{MMD}^2=0.1298$, $p=0.0001$) and changes from baseline
($\mathrm{MMD}^2=0.0872$, $p=0.0001$). Both tests use 9,999
wild-bootstrap resamples, with a minimum attainable Monte Carlo
$p$-value of $0.0001$. These findings suggest that the groups
differ not only in their glucose distributional patterns during
follow-up but also in how these patterns change relative to baseline.

\begin{table}[!htbp]
\centering
\caption{Global functional MMD tests comparing \Treatment{} and \Control{}
groups in the bivariate randomized-phase analysis, using
participant-level wild-bootstrap inference. The absolute comparison includes
62 treatment and 31 control participants; the baseline-centered comparison
includes 55 treatment and 25 control participants.}
\label{tab:primary_functional_mmd}
\begin{tabular}{lcc}
\hline
Comparison & MMD$^2$ & Wild-bootstrap $p$ \\
\hline
Absolute trajectories & 0.1298 & 0.0001 \\
Change from pre-randomization baseline & 0.0872 & 0.0001 \\
\hline
\end{tabular}
\end{table}

As a secondary inferential analysis, we perform component-specific pointwise MMD tests and combine cumulative \(p\)-values using the aggregated Cauchy association test (ACAT) to aggregate pointwise evidence over time. \Cref{fig:main_acat_d2} identifies the mixture components and follow-up periods showing evidence of between-group differences. While the global functional MMD tests remain our primary inferential comparisons, these component-specific analyses provide a more interpretable characterization of the distributional changes, highlighting which mixture components, such as Components 2 and 5, contribute most prominently to the observed differences and when these differences emerge during follow-up.

\begin{figure}[!htbp]
    \centering
    \begin{subfigure}[t]{0.49\textwidth}
        \includegraphics[width=\linewidth]
        {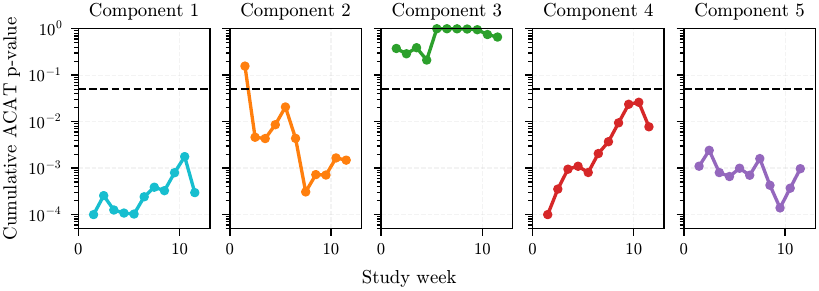}
        \caption{Absolute trajectories.}
    \end{subfigure}\hfill
    \begin{subfigure}[t]{0.49\textwidth}
        \includegraphics[width=\linewidth]
        {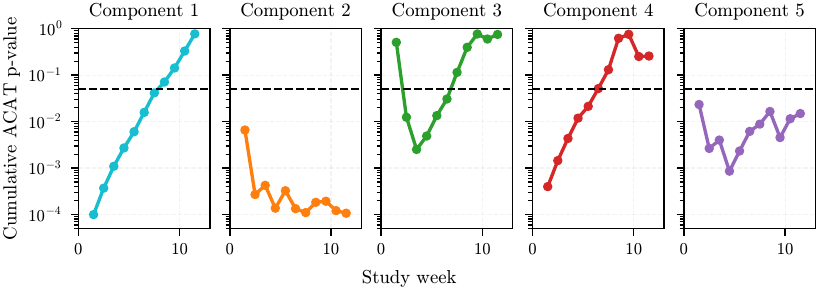}
        \caption{Change from pre-randomization baseline.}
    \end{subfigure}
    \caption{Cumulative component-specific evidence of between-group
    differences in the bivariate analysis. Curves represent
    \CompOne{}, \CompTwo{}, \CompThree{}, \CompFour{}, and \CompFive{},
    using the same colors as the fitted mixture-weight figures.
    Global functional MMD tests provide the primary inference.}
    \label{fig:main_acat_d2}
\end{figure}

\section{Conclusion}

We developed an interpretable statistical framework for modeling the evolution of probability distributions in longitudinal digital health studies. Our approach combines a shared mixture representation with a flexible continuous-time probabilistic model, preserving the interpretability of mixture-based methods while accommodating complex and potentially nonlinear changes in the underlying distributions over time. Simulation results presented in Appendix B demonstrate competitive estimation accuracy relative to existing methods without sacrificing interpretability.

In the CGM case study, we show that the proposed model provides a more comprehensive characterization of treatment responses than conventional scalar summaries \citep{Battelino2023Continuous}. Rather than reducing glucose trajectories to a limited set of summary measures, our approach characterizes how the joint distribution of glucose concentrations and their rates of change evolves over time. To our knowledge, this is among the first applications of continuous-time probabilistic distributional modeling to jointly characterize glucose concentrations and their rates of change in a randomized clinical trial.

Future research should account for uncertainty arising from both estimation stages, compare the neural ODE approach with simpler smoothing methods, and assess the generalizability of the proposed representation to other clinical trials and higher-dimensional measurements. Overall, our methodology provides a flexible and interpretable framework for analyzing longitudinal distributional data, with potential applications across digital health studies involving continuous monitoring of human physiology.


\bibliographystyle{plainnat}
\bibliography{bibliography}

@article{schoelwer2024use,
  author = {Schoelwer, Melissa J. and DeBoer, Mark D. and Breton, Marc D.},
  title = {Use of diabetes technology in children},
  journal = {Diabetologia},
  volume = {67},
  number = {10},
  pages = {2075--2084},
  year = {2024},
  doi = {10.1007/s00125-024-06218-0},
}

@article{Frechet1948Les,
  author = {Fr{\'e}chet, Maurice},
  title = {{Les {\'e}l{\'e}ments al{\'e}atoires de nature quelconque dans un espace distanci{\'e}}},
  journal = {Annales de l'Institut Henri Poincar{\'e}},
  year = {1948},
  volume = {10},
  number = {4},
  pages = {215--310},
  url = {https://www.numdam.org/item/AIHP_1948__10_4_215_0/},
  ids = {AIHP_1948__10_4_215_0}
}

@inproceedings{NIPS2014_4e382cb4,
  author = {Chwialkowski, Kacper and Sejdinovic, Dino and Gretton, Arthur},
  title = {A Wild Bootstrap for Degenerate Kernel Tests},
  booktitle = {Advances in Neural Information Processing Systems 27},
  pages = {3608--3616},
  publisher = {Curran Associates, Inc.},
  year = {2014},
  url = {https://proceedings.neurips.cc/paper_files/paper/2014/file/4e382cb49370f64415df2672b19fb1f2-Paper.pdf},
}

@misc{AlvarezLopez2026Constructive,
  title         = {Constructive Interpolation and Generalization Rates for Neural {ODE}s: A Control Perspective},
  author        = {{\'A}lvarez-L{\'o}pez, Antonio and Liverani, Lorenzo and Zuazua, Enrique},
  year          = {2026},
  eprint        = {2606.00469},
  archivePrefix = {arXiv},
  primaryClass  = {math.OC},
  url           = {https://arxiv.org/abs/2606.00469}
}

@article{Li2024Universal,
author = {Li, Ziqian and Liu, Kang and Liverani, Lorenzo and Zuazua, Enrique},
year = {2026},
month = {02},
pages = {193-223},
title = {Universal Approximation of Dynamical Systems by Semiautonomous Neural ODEs and Applications},
volume = {64},
journal = {SIAM Journal on Numerical Analysis},
doi = {10.1137/24M1679690},
ids={LiLiuLiveraniZuazua_SANODEs,li2024universalapproximationdynamicalsystems}
}

@misc{garreau2018largesampleanalysismedian,
  author = {Garreau, Damien and Jitkrittum, Wittawat and Kanagawa, Motonobu},
  title = {Large sample analysis of the median heuristic},
  year = {2018},
  url = {https://arxiv.org/abs/1707.07269},
  note = {arXiv:1707.07269 (preprint)},
}

@article{Liu02012020,
  author = {Liu, Yaowu and Xie, Jun},
  title = {Cauchy Combination Test: A Powerful Test with Analytic {p}-Value Calculation under Arbitrary Dependency Structures},
  journal = {Journal of the American Statistical Association},
  volume = {115},
  number = {529},
  pages = {393--402},
  year = {2020},
  doi = {10.1080/01621459.2018.1554485},
}

@article{https://doi.org/10.1002/hbm.26112,
  author = {Tang, Bohao and Zhao, Yi and Venkataraman, Archana and Tsapkini, Kyrana and Lindquist, Martin A. and Pekar, James and Caffo, Brian},
  title = {Differences in functional connectivity distribution after transcranial direct-current stimulation: A connectivity density point of view},
  journal = {Human Brain Mapping},
  volume = {44},
  number = {1},
  pages = {170--185},
  year = {2023},
  doi = {10.1002/hbm.26112},
}

@article{matabuena2026exploratory,
  author = {Matabuena, Marcos and Straczkiewicz, Marcin and Calcagno, Narghes and Burke, Katherine M. and Royse, Timothy B. and Iyer, Amrita and Carney, Kendall T. and Hall, Sydney and Berry, James D. and Onnela, Jukka-Pekka},
  title = {Exploratory analysis of smartphone-based step counts as a digital biomarker for survival in {ALS} patients},
  journal = {Frontiers in Digital Health},
  volume = {7},
  pages = {1705368},
  year = {2026},
  doi = {10.3389/fdgth.2025.1705368},
}

@article{gretton2012kernel,
  author = {Gretton, Arthur and Borgwardt, Karsten M and Rasch, Malte J and Sch{\"o}lkopf, Bernhard and Smola, Alexander},
  title = {A kernel two-sample test},
  journal = {Journal of Machine Learning Research},
  volume = {13},
  number = {25},
  pages = {723--773},
  year = {2012},
}

@article{matabuena2025predicting,
  author = {Matabuena, Marcos and Ghosal, Aritra and Meiring, Wendy and Petersen, Alexander},
  title = {Predicting distributions of physical activity profiles in the National Health and Nutrition Examination Survey database using a partially linear Fr{\'e}chet single index model},
  journal = {Biostatistics},
  volume = {26},
  number = {1},
  pages = {kxaf013},
  year = {2025},
  doi = {10.1093/biostatistics/kxaf013},
}

@article{matabuena2023distributional,
  author = {Matabuena, Marcos and Petersen, Alexander},
  title = {Distributional data analysis of accelerometer data from the {NHANES} database using nonparametric survey regression models},
  journal = {Journal of the Royal Statistical Society: Series C (Applied Statistics)},
  volume = {72},
  number = {2},
  pages = {294--313},
  year = {2023},
  doi = {10.1093/jrsssc/qlad007},
}

@article{Ghosal26Survival,
  author = {Ghosal, Rahul and Cho, Sunwoo Emma and Matabuena, Marcos},
  title = {Survival on Image Regression With Application to Partially Functional Distributional Representation of Physical Activity},
  journal = {Statistical Analysis and Data Mining: An {ASA} Data Science Journal},
  volume = {19},
  number = {1},
  pages = {e70068},
  year = {2026},
  doi = {10.1002/sam.70068},
}

@article{10.1214/26-AOAS2139I,
  author = {Matabuena, Marcos and Crainiceanu, Ciprian M.},
  title = {Multilevel functional distributional models with applications to continuous glucose monitoring in diabetes clinical trials},
  journal = {The Annals of Applied Statistics},
  volume = {20},
  number = {1},
  pages = {476--495},
  year = {2026},
  doi = {10.1214/26-AOAS2139},
}

@article{bengio2017deep,
  author = {LeCun, Yann and Bengio, Yoshua and Hinton, Geoffrey},
  title = {Deep learning},
  journal = {Nature},
  volume = {521},
  number = {7553},
  pages = {436--444},
  year = {2015},
  doi = {10.1038/nature14539},
  isbn = {1476-4687},
}

@article{wadwa2023trial,
  author = {Wadwa, R. Paul and Reed, Zachariah W. and Buckingham, Bruce A. and DeBoer, Mark D. and Ekhlaspour, Laya and Forlenza, Gregory P. and Schoelwer, Melissa and Lum, John and Kollman, Craig and Beck, Roy W. and Breton, Marc D.},
  title = {Trial of hybrid closed-loop control in young children with type 1 diabetes},
  journal = {The New England Journal of Medicine},
  volume = {388},
  number = {11},
  pages = {991--1001},
  year = {2023},
  doi = {10.1056/NEJMoa2210834},
  note = {For the {PEDAP} Trial Study Group},
}

@article{kitagawa2025artificial,
  author = {Kitagawa, Hiroyuki and Munekage, Masaya and Seo, Satoru and Hanazaki, Kazuhiro},
  title = {Artificial pancreas: the past and the future},
  journal = {Journal of Artificial Organs},
  volume = {28},
  number = {4},
  pages = {514--521},
  year = {2025},
  doi = {10.1007/s10047-025-01510-1},
}

@article{hughes2023digital,
  author = {Hughes, Michael S and Addala, Ananta and Buckingham, Bruce},
  title = {Digital technology for diabetes},
  journal = {The New England Journal of Medicine},
  volume = {389},
  number = {22},
  pages = {2076--2086},
  year = {2023},
  doi = {10.1056/NEJMra2215899},
}

@article{ware2024eighteen,
  author = {Ware, Julia and Allen, Janet M. and Boughton, Charlotte K. and Wilinska, Malgorzata E. and Hartnell, Sara and Thankamony, Ajay and de Beaufort, Carine and Campbell, Fiona M. and Fr{\"o}hlich-Reiterer, Elke and Fritsch, Maria and Hofer, Sabine E. and Kapellen, Thomas M. and Rami-Merhar, Birgit and Tauschmann, Martin and Hovorka, Roman},
  title = {Eighteen-Month Hybrid Closed-Loop Use in Very Young Children with Type 1 Diabetes: A Single-Arm Multicenter Trial},
  journal = {Diabetes Care},
  volume = {47},
  number = {12},
  pages = {2189--2195},
  year = {2024},
  doi = {10.2337/dc24-1313},
  note = {For the {KidsAP} Consortium},
}

@article{beck2023meta,
  author = {Beck, Roy W and Kanapka, Lauren G and Breton, Marc D and Brown, Sue A and Wadwa, R Paul and Buckingham, Bruce A and Kollman, Craig and Kovatchev, Boris},
  title = {A Meta-Analysis of Randomized Trial Outcomes for the t:slim {X2} Insulin Pump with {Control-IQ} Technology in Youth and Adults from Age 2 to 72},
  journal = {Diabetes Technology \& Therapeutics},
  volume = {25},
  number = {5},
  pages = {329--342},
  year = {2023},
  doi = {10.1089/dia.2022.0558},
}

@article{stahl2025efficacy,
  author = {Stahl-Pehe, Anna and Shokri-Mashhadi, Nafiseh and Wirth, Marielle and Schlesinger, Sabrina and Kuss, Oliver and Holl, Reinhard W. and B{\"a}chle, Christina and Warz, Klaus-D. and B{\"u}rger-B{\"u}sing, Jutta and Sp{\"o}rkel, Olaf and Rosenbauer, Joachim},
  title = {Efficacy of automated insulin delivery systems in people with type 1 diabetes: a systematic review and network meta-analysis of outpatient randomised controlled trials},
  journal = {EClinicalMedicine},
  volume = {82},
  pages = {103190},
  year = {2025},
  doi = {10.1016/j.eclinm.2025.103190},
}

@article{Battelino2023Continuous,
  author = {Battelino, Tadej and Alexander, Charles M. and Amiel, Stephanie A. and Arreaza-Rubin, Guillermo and Beck, Roy W. and Bergenstal, Richard M. and Buckingham, Bruce A. and Carroll, James and Ceriello, Antonio and Chow, Elaine and Choudhary, Pratik and Close, Kelly and Danne, Thomas and Dutta, Sanjoy and Gabbay, Robert and Garg, Satish and Heverly, Julie and Hirsch, Irl B. and Kader, Tina and Kenney, Julia and Kovatchev, Boris and Laffel, Lori and Maahs, David and Mathieu, Chantal and Mauricio, D{\'i}dac and Nimri, Revital and Nishimura, Rimei and Scharf, Mauro and Del Prato, Stefano and Renard, Eric and Rosenstock, Julio and Saboo, Banshi and Ueki, Kohjiro and Umpierrez, Guillermo E. and Weinzimer, Stuart A. and Phillip, Moshe},
  title = {Continuous glucose monitoring and metrics for clinical trials: an international consensus statement},
  journal = {The Lancet Diabetes {\&} Endocrinology},
  volume = {11},
  number = {1},
  pages = {42--57},
  year = {2023},
  doi = {10.1016/S2213-8587(22)00319-9},
}

@book{chacon2018multivariate,
  author = {Chac{\'o}n, Jos{\'e} E. and Duong, Tarn},
  title = {Multivariate kernel smoothing and its applications},
  publisher = {Chapman and Hall/CRC},
  year = {2018},
  doi = {10.1201/9780429485572},
}

@article{Ghosal2023Distributional,
  author = {Ghosal, Rahul and Varma, Vijay R and Volfson, Dmitri and Hillel, Inbar and Urbanek, Jacek and Hausdorff, Jeffrey M and Watts, Amber and Zipunnikov, Vadim},
  title = {Distributional data analysis via quantile functions and its application to modeling digital biomarkers of gait in {Alzheimer's} disease},
  journal = {Biostatistics},
  volume = {24},
  number = {3},
  pages = {539--561},
  year = {2023},
  doi = {10.1093/biostatistics/kxab041},
}

@article{Ghosal2024Multivariate,
  author = {Ghosal, Rahul and Matabuena, Marcos},
  title = {Multivariate Scalar on Multidimensional Distribution Regression With Application to Modeling the Association Between Physical Activity and Cognitive Functions},
  journal = {Biometrical Journal},
  volume = {66},
  number = {7},
  pages = {e202400042},
  year = {2024},
  doi = {10.1002/bimj.202400042},
}

@article{Ghosal2025Distributional,
  author = {Ghosal, Rahul and Ghosh, Sujit K and Schrack, Jennifer A and Zipunnikov, Vadim},
  title = {Distributional outcome regression via quantile functions and its application to modelling continuously monitored heart rate and physical activity},
  journal = {Journal of the American Statistical Association},
  volume = {120},
  number = {551},
  pages = {1347--1359},
  year = {2025},
  doi = {10.1080/01621459.2025.2460232},
}

@inproceedings{Jia2019Neural,
  author = {Jia, Junteng and Benson, Austin R},
  title = {Neural Jump Stochastic Differential Equations},
  booktitle = {Advances in Neural Information Processing Systems 32},
  publisher = {Curran Associates, Inc.},
  year = {2019},
}

@inproceedings{Katta2024Interpretable,
  author = {Katta, Srikar and Parikh, Harsh and Rudin, Cynthia and Volfovsky, Alexander},
  title = {Interpretable Causal Inference for Analyzing Wearable, Sensor, and Distributional Data},
  booktitle = {Proceedings of the 27th International Conference on Artificial Intelligence and Statistics (PMLR 238)},
  pages = {3340--3348},
  publisher = {PMLR},
  year = {2024},
}

@inproceedings{kidger2020neural,
  author = {Kidger, Patrick and Morrill, James and Foster, James and Lyons, Terry},
  title = {Neural Controlled Differential Equations for Irregular Time Series},
  booktitle = {Advances in Neural Information Processing Systems 33},
  pages = {6696--6707},
  publisher = {Curran Associates, Inc.},
  year = {2020},
}

@misc{Lugosi2024Uncertainty,
  author = {Lugosi, G{\'a}bor and Matabuena, Marcos},
  title = {Conformal and {kNN} Predictive Uncertainty Quantification Algorithms in Metric Spaces},
  year = {2026},
  url = {https://arxiv.org/abs/2507.15741},
  note = {arXiv:2507.15741 (preprint)},
}

@article{Matabuena2021Glucodensities,
  author = {Matabuena, Marcos and Petersen, Alexander and Vidal, Juan C and Gude, Francisco},
  title = {Glucodensities: A new representation of glucose profiles using distributional data analysis},
  journal = {Statistical Methods in Medical Research},
  volume = {30},
  number = {6},
  pages = {1445--1464},
  year = {2021},
  doi = {10.1177/0962280221998064},
}

@article{Matabuena2022Physical,
  author = {Matabuena, Marcos and F\'{e}lix, Paulo and Hammouri, Ziad Akram Ali and Mota, Jorge and del Pozo Cruz, Borja},
  title = {Physical activity phenotypes and mortality in older adults: a novel distributional data analysis of accelerometry in the {NHANES}},
  journal = {Aging Clinical and Experimental Research},
  volume = {34},
  number = {12},
  pages = {3107--3114},
  year = {2022},
  doi = {10.1007/s40520-022-02260-3},
}

@article{Matabuena2024Glucodensity,
  author = {Matabuena, Marcos and Ghosal, Rahul and Aguilar, Javier Enrique and Keshet, Ayya and Wagner, Robert and Fern{\'a}ndez Merino, Carmen and S{\'a}nchez Castro, Juan and Zipunnikov, Vadim and Onnela, Jukka-Pekka and Gude, Francisco},
  title = {Glucodensity functional profiles outperform traditional continuous glucose monitoring metrics},
  journal = {Scientific Reports},
  volume = {15},
  number = {1},
  pages = {33662},
  year = {2025},
  doi = {10.1038/s41598-025-18119-2},
}

@article{Mesko2023Imperative,
  author = {Mesk\'{o}, Bertalan and Topol, Eric J},
  title = {The imperative for regulatory oversight of large language models (or generative {AI}) in healthcare},
  journal = {npj Digital Medicine},
  volume = {6},
  number = {1},
  pages = {120},
  year = {2023},
  doi = {10.1038/s41746-023-00873-0},
}

@article{muandet2017kernel,
  author = {Muandet, Krikamol and Fukumizu, Kenji and Sriperumbudur, Bharath and Sch\"{o}lkopf, Bernhard},
  title = {Kernel mean embedding of distributions: A review and beyond},
  journal = {Foundations and Trends in Machine Learning},
  volume = {10},
  number = {1--2},
  pages = {1--141},
  year = {2017},
  doi = {10.1561/2200000060},
}

@inproceedings{Papamakarios2017Masked,
  author = {Papamakarios, George and Pavlakou, Theo and Murray, Iain},
  title = {Masked Autoregressive Flow for Density Estimation},
  booktitle = {Advances in Neural Information Processing Systems 30},
  pages = {2338--2347},
  publisher = {Curran Associates, Inc.},
  year = {2017},
}

@article{Papamakarios2021Normalizing,
  author = {Papamakarios, George and Nalisnick, Eric and Rezende, Danilo Jimenez and Mohamed, Shakir and Lakshminarayanan, Balaji},
  title = {Normalizing flows for probabilistic modeling and inference},
  journal = {Journal of Machine Learning Research},
  volume = {22},
  number = {57},
  pages = {1--64},
  year = {2021},
  url = {https://jmlr.org/papers/v22/19-1028.html},
}

@article{Park2025Beyond,
  author = {Park, Junyoung and Kok, Neo and Gaynanova, Irina},
  title = {Beyond fixed thresholds: optimizing summaries of wearable device data via piecewise linearization of quantile functions},
  journal = {Statistics in Medicine},
  volume = {45},
  number = {15--17},
  pages = {e70646},
  year = {2026},
  doi = {10.1002/sim.70646},
}

@inproceedings{qian2021integrating,
  author = {Qian, Zhaozhi and Zame, William and Fleuren, Lucas and Elbers, Paul and van der Schaar, Mihaela},
  title = {Integrating expert {ODEs} into neural {ODEs}: pharmacology and disease progression},
  booktitle = {Advances in Neural Information Processing Systems 34},
  pages = {11364--11383},
  publisher = {Curran Associates, Inc.},
  year = {2021},
}

@article{rigby2005generalized,
  author = {Rigby, Robert A. and Stasinopoulos, D. Mikis},
  title = {Generalized additive models for location, scale and shape},
  journal = {Journal of the Royal Statistical Society: Series C (Applied Statistics)},
  volume = {54},
  number = {3},
  pages = {507--554},
  year = {2005},
  doi = {10.1111/j.1467-9876.2005.00510.x},
}

@article{Rosenblatt1952Remarks,
  author = {Rosenblatt, Murray},
  title = {Remarks on a Multivariate Transformation},
  journal = {The Annals of Mathematical Statistics},
  volume = {23},
  number = {3},
  pages = {470--472},
  year = {1952},
  doi = {10.1214/aoms/1177729394},
}

@inproceedings{Rubanova2019Latent,
  author = {Rubanova, Yulia and Chen, Ricky T. Q. and Duvenaud, David K},
  title = {Latent Ordinary Differential Equations for Irregularly-Sampled Time Series},
  booktitle = {Advances in Neural Information Processing Systems 32},
  pages = {5320--5330},
  publisher = {Curran Associates, Inc.},
  year = {2019},
}

@article{sejdinovic2013equivalence,
  author = {Sejdinovic, Dino and Sriperumbudur, Bharath and Gretton, Arthur and Fukumizu, Kenji},
  title = {Equivalence of distance-based and {RKHS}-based statistics in hypothesis testing},
  journal = {The Annals of Statistics},
  volume = {41},
  number = {5},
  pages = {2263--2291},
  year = {2013},
  doi = {10.1214/13-AOS1140},
}

@book{serfling2009approximation,
  author = {Serfling, Robert J},
  title = {Approximation theorems of mathematical statistics},
  publisher = {Wiley},
  year = {2009},
  note = {2009 electronic edition (original edition 1980)},
}

@book{silverman1986density,
  author = {Silverman, Bernard W},
  title = {Density estimation for statistics and data analysis},
  publisher = {Chapman and Hall},
  year = {1986},
}

@book{Silverman2018Density,
  author = {Silverman, Bernard W},
  title = {Density estimation for statistics and data analysis},
  publisher = {Routledge},
  year = {2018},
  note = {2018 reprint of the 1986 monograph},
}

@article{sriperumbudur2011universality,
  author = {Sriperumbudur, Bharath K and Fukumizu, Kenji and Lanckriet, Gert R. G.},
  title = {Universality, Characteristic Kernels and {RKHS} Embedding of Measures},
  journal = {Journal of Machine Learning Research},
  volume = {12},
  number = {70},
  pages = {2389--2410},
  year = {2011},
}

@article{Szabo2016Learning,
  author = {Szab\'{o}, Zolt\'{a}n and Sriperumbudur, Bharath K and P\'{o}czos, Barnab\'{a}s and Gretton, Arthur},
  title = {Learning theory for distribution regression},
  journal = {Journal of Machine Learning Research},
  volume = {17},
  number = {152},
  pages = {1--40},
  year = {2016},
}

@book{Tsybakov2008Introduction,
  author = {Tsybakov, Alexandre B.},
  title = {Introduction to Nonparametric Estimation},
  series = {Springer Series in Statistics},
  publisher = {Springer},
  edition = {1},
  year = {2008},
  doi = {10.1007/b13794},
  isbn = {0387790519, 9780387790510},
}

@inproceedings{tancik2020fourier,
  author = {Tancik, Matthew and Srinivasan, Pratul P. and Mildenhall, Ben and
          Fridovich-Keil, Sara and Raghavan, Nithin and Singhal, Utkarsh and
          Ramamoorthi, Ravi and Barron, Jonathan T. and Ng, Ren},
  title = {Fourier Features Let Networks Learn High Frequency Functions in Low Dimensional Domains},
  booktitle = {Advances in Neural Information Processing Systems 33},
  publisher = {Curran Associates, Inc.},
  year = {2020},
}

@misc{kazemi2019time2vec,
  author = {Kazemi, Seyed Mehran and Goel, Rishab and Eghbali, Sepehr and
          Ramanan, Janahan and Sahota, Jaspreet and Thakur, Sanjay and
          Wu, Stella and Smyth, Cathal and Poupart, Pascal and Brubaker, Marcus},
  title = {{Time2Vec}: Learning a Vector Representation of Time},
  year = {2019},
  url = {https://arxiv.org/abs/1907.05321},
  note = {arXiv:1907.05321 (preprint)},
}
\appendix

\section{Statistical theory and proofs}\label{sec:the}

We provide theoretical guarantees for the approximation
capacity of the shared Gaussian mixture representation
and the finite-sample stability of the discrete-time
MMD weight estimator, prior to neural-ODE smoothing.

\begin{theorem}[Universality]\label{thm:uniform_shared_dictionary}
Let $\{f_t\}_{t\in[0,T]}\subset L^1(\R^d)$ be a family of probability densities satisfying:
\begin{enumerate}
\item\label{item:tightness} For every $\eta>0$, there exists $R<\infty$ such that
$\int_{\|x\|>R}f_t(x)\diff x<\eta$ for all $t\in[0,T]$;
\item\label{item:equicont} $\displaystyle \lim_{\|h\|\to0}\sup_{t\in[0,T]}\|f_t(\cdot+h)-f_t(\cdot)\|_{L^1(\R^d)}=0$.
\end{enumerate}
Then, for every $\varepsilon>0$, there exist $K\in\N$, $\sigma^2>0$, and centers $\{m_s\}_{s=1}^K\subset\R^d$ such that, for each $t\in[0,T]$, one can choose $\alpha(t)\in\Delta^{K-1}$ satisfying
\[
\sup_{t\in[0,T]}
\left\|
f_t(\cdot)-\sum_{s=1}^K \alpha_s(t)\,\mathcal N\left(\cdot\mid m_s,\sigma^2\mathsf{Id}\right)
\right\|_{L^1(\R^d)}
<\varepsilon.
\]
If, in addition, $t\mapsto f_t$ is continuous in $L^1(\R^d)$, then $t\mapsto\alpha(t)$ can be chosen continuous.
\end{theorem}

\cref{thm:uniform_shared_dictionary} establishes the uniform approximation capacity of the proposed model, although it does not prescribe a finite-sample rule for selecting $K$.

\begin{proof}[Proof of \cref{thm:uniform_shared_dictionary}]
Without loss of generality, fix $\varepsilon\in(0,1)$. By assumption~\ref{item:tightness}, there exists $R<\infty$ such that
\[
\sup_{t\in[0,T]}\int_{\|x\|>R} f_t(x)\diff x<\frac{\varepsilon}{4}.
\]
Let $\varphi_\sigma(x)=\mathcal N(x\mid 0,\sigma^2\mathsf{Id})$ denote the Gaussian mollifier and define $g_t=f_t\ast\varphi_\sigma$. By the approximate-identity property of $\varphi_\sigma$ and the uniform $L^1$-equicontinuity in assumption~\ref{item:equicont}, we can fix $\sigma>0$ sufficiently small so that
\[
\sup_{t\in[0,T]}\|f_t-g_t\|_{L^1(\R^d)}<\frac{\varepsilon}{4}.
\]
Next, partition $B_R=\{x\in\R^d:\|x\|\leq R\}$ into finitely many measurable sets $(C_s)_{s=1}^K$ of diameter at most $h>0$, and pick a representative point $m_s\in C_s$ for each $s\in\{1,\ldots,K\}$. Define the unnormalized weights and total mass on $B_R$ by
\[
\tilde\alpha_s(t)=\int_{C_s} f_t(y)\diff y,
\qquad
Z_t=\sum_{s=1}^K \tilde\alpha_s(t)=\int_{B_R} f_t(y)\diff y \in \left[1-\frac{\varepsilon}{4},1\right],
\]
and set $\alpha_s(t)=\tilde\alpha_s(t)/Z_t$, so that $\alpha(t)\in\Delta^{K-1}$.

To bound the approximation error, decompose $g_t$ as
\[
g_t(x)=\sum_{s=1}^K\int_{C_s} f_t(y)\varphi_\sigma(x-y)\diff y
+\int_{B_R^c} f_t(y)\varphi_\sigma(x-y)\diff y.
\]
By Tonelli's theorem and the normalization $\|\varphi_\sigma\|_{L^1(\R^d)}=1$, the tail integral over $B_R^c$ has $L^1$-norm equal to $1-Z_t<\varepsilon/4$. For the sum over $B_R$, note that for every $y\in C_s$,
\[
\|\varphi_\sigma(\cdot-y)-\varphi_\sigma(\cdot-m_s)\|_{L^1(\R^d)}
\leq
\|\nabla\varphi_\sigma\|_{L^1(\R^d)}\|y-m_s\|
\leq
C_d\,\frac{h}{\sigma},
\]
where $C_d=\|\nabla\varphi_1\|_{L^1(\R^d)}$. Integrating this over $y\in C_s$ and summing over $s=1,\ldots,K$ gives
\[
\sup_{t\in[0,T]}
\left\|
g_t-\sum_{s=1}^K \tilde\alpha_s(t)\varphi_\sigma(\cdot-m_s)
\right\|_{L^1(\R^d)}
<
\frac{\varepsilon}{4} + C_d\,\frac{h}{\sigma}.
\]
Furthermore, replacing $\tilde\alpha_s(t)$ with the normalized weights $\alpha_s(t)$ incurs an additional $L^1$-error of at most
\[
\left\|
\sum_{s=1}^K \left(\tilde\alpha_s(t)-\alpha_s(t)\right)\varphi_\sigma(\cdot-m_s)
\right\|_{L^1(\R^d)}
\leq
\sum_{s=1}^K |\tilde\alpha_s(t)-\alpha_s(t)|
= 1-Z_t < \frac{\varepsilon}{4}.
\]
Choosing $h\leq \sigma\varepsilon/(4C_d)$ and combining the four error bounds via the triangle inequality yields
\[
\sup_{t\in[0,T]}
\left\|
f_t(\cdot)-\sum_{s=1}^K \alpha_s(t)\,\varphi_\sigma(\cdot-m_s)
\right\|_{L^1(\R^d)}
<\varepsilon.
\]
Since $\varphi_\sigma(\cdot-m_s)=\mathcal N(\cdot\mid m_s,\sigma^2 \mathsf{Id})$, the uniform approximation bound follows. Finally, suppose $t\mapsto f_t$ is continuous in $L^1(\R^d)$. For each $s\in\{1,\ldots,K\}$ and all $t,t'\in[0,T]$,
\[
|\tilde\alpha_s(t)-\tilde\alpha_s(t')|
\leq
\int_{C_s}|f_t(y)-f_{t'}(y)|\diff y
\leq
\|f_t-f_{t'}\|_{L^1(\R^d)},
\]
so each map $t\mapsto \tilde\alpha_s(t)$ is continuous. Consequently, $t\mapsto Z_t=\sum_{s=1}^K\tilde\alpha_s(t)$ is continuous as well, and since $Z_t\geq 1-\varepsilon/4>0$, the normalized trajectory $t\mapsto \alpha(t)=\tilde\alpha(t)/Z_t$ is continuous on $[0,T]$.
\end{proof}

The following stability result allows component-specific ridge
regularization parameters. The estimator used in the main paper
corresponds to the special case $\lambda_s=\lambda$ for all
$s=1,\ldots,K$, for which
$\sum_{s=1}^K\lambda_s\alpha_s^2=\lambda\|\alpha\|_2^2$.

\begin{theorem}[Finite-sample stability]\label{thm:weights_rate}
Fix $t_i\in\tau_m$ and consider 
\[
\widehat\alpha_i
=
\underset{\alpha\in\Delta^{K-1}}{\operatorname{argmin}}
\left\{
\alpha^\top I_i\alpha
-2\alpha^\top J_i
+\sum_{s=1}^K\lambda_s\alpha_s^2
\right\}.
\]
Let $J_i^\star=\E[J_i]$ and define the population counterpart
\[
\alpha_i^\star
=
\underset{\alpha\in\Delta^{K-1}}{\operatorname{argmin}}
\left\{
\alpha^\top I_i\alpha
-2\alpha^\top J_i^\star
+\sum_{s=1}^K\lambda_s\alpha_s^2
\right\}.
\]
Assume that $(X_{t_i,j})_{j=1}^{N_i}$ are i.i.d., and that the Gaussian dictionary and the kernel $k_i$ (including its bandwidth) are fixed independently of these observations. Suppose further that $k_i$ is positive definite with $\sup_x k_i(x,x)\leq 1$, and let $\lambda_{\min}=\min_s\lambda_s\geq 0$.

If $\lambda_{\min}>0$, then, for every
$\delta\in(0,1)$, with probability at least $1-\delta$,
\[
\|\widehat\alpha_i-\alpha_i^\star\|_2
\leq
\frac{1+\sqrt{2\log(1/\delta)}}
{2\sqrt{\lambda_{\min}N_i}}.
\]
Moreover, denoting by $P_\alpha$ the Gaussian mixture with weights $\alpha$ and the fixed dictionary,
\[
\operatorname{MMD}_{k_i}
(P_{\widehat\alpha_i},P_{\alpha_i^\star})
\leq
\frac{1+\sqrt{2\log(1/\delta)}}{\sqrt{N_i}},
\]
which holds even when $\lambda_{\min}=0$. Finally, if these assumptions hold at all $t_i\in\tau_m$, then with probability at least $1-\delta$,
\[
\max_{0\leq i\leq m}
\|\widehat\alpha_i-\alpha_i^\star\|_2
\leq
\frac{1+\sqrt{2\log((m+1)/\delta)}}
{2\sqrt{\lambda_{\min}N_*}},
\qquad
N_*=\min_{0\leq i\leq m}N_i.
\]
\end{theorem}

\begin{remark}\label{rem:translate}
The theorem concerns weight estimation with a fixed
Gaussian dictionary and kernel under i.i.d. sampling.
It does not account for temporal dependence,
data-dependent dictionary or bandwidth estimation,
or the subsequent neural-ODE smoothing step.

For fixed Gaussian components, the mixture density
is linear in the weights. Hence,
\[
\|f_{\widehat\theta_i}-f_{\theta_i^\star}\|_{L^1(\R^d)}
\leq \|\widehat\alpha_i-\alpha_i^\star\|_1
\leq \sqrt{K}\,
\|\widehat\alpha_i-\alpha_i^\star\|_2.
\]
Consequently, the corresponding joint distribution
functions satisfy
\[
\sup_{x\in\R^d}
|F_{\widehat\theta_i}(x)-F_{\theta_i^\star}(x)|
\leq \frac{1}{2}
\|\widehat\alpha_i-\alpha_i^\star\|_1
\leq \frac{\sqrt{K}}{2}
\|\widehat\alpha_i-\alpha_i^\star\|_2.
\]
\end{remark}

\begin{proof}[Proof of \cref{thm:weights_rate}]
Let $\mathcal H_i$ be the reproducing kernel Hilbert space (RKHS) associated with $k_i$, and let $h_s\in\mathcal H_i$ denote the kernel mean embedding of the $s$-th Gaussian component. Define the linear operator
\[
A:\R^K\to\mathcal H_i,
\qquad
A\alpha=\sum_{s=1}^K\alpha_s h_s,
\]
along with the empirical and population mean embeddings
\[
\widehat\mu_i
=\frac{1}{N_i}\sum_{j=1}^{N_i}
k_i(X_{t_i,j},\cdot),
\qquad
\mu_i=\E[\widehat\mu_i].
\]
For the adjoint operator $A^*:\mathcal H_i\to\R^K$, we have $I_i=A^*A$, $J_i=A^*\widehat\mu_i$, and $J_i^\star=A^*\mu_i$.

Let $\Lambda=\mathrm{diag}(\lambda_1,\ldots,\lambda_K)$ and set $d=\widehat\alpha_i-\alpha_i^\star$. The first-order optimality conditions for the convex quadratic programs defining $\widehat\alpha_i$ and $\alpha_i^\star$ over the convex set $\Delta^{K-1}$ read
\begin{align*}
\langle A\widehat\alpha_i-\widehat\mu_i, A(\alpha-\widehat\alpha_i)\rangle_{\mathcal H_i}
+\widehat\alpha_i^\top\Lambda(\alpha-\widehat\alpha_i)&\geq 0,\\
\langle A\alpha_i^\star-\mu_i, A(\alpha-\alpha_i^\star)\rangle_{\mathcal H_i}
+(\alpha_i^\star)^\top\Lambda(\alpha-\alpha_i^\star)&\geq 0,
\end{align*}
for all $\alpha\in\Delta^{K-1}$. Evaluating the first inequality at $\alpha=\alpha_i^\star$, the second at $\alpha=\widehat\alpha_i$, and summing them yields
\[
\|Ad\|_{\mathcal H_i}^2+d^\top\Lambda d
\leq
\langle \widehat\mu_i-\mu_i,Ad\rangle_{\mathcal H_i}.
\]
Applying the Cauchy--Schwarz inequality and the bound $d^\top\Lambda d\geq \lambda_{\min}\|d\|_2^2$, we obtain
\[
\|Ad\|_{\mathcal H_i}^2
+\lambda_{\min}\|d\|_2^2
\leq
\|\widehat\mu_i-\mu_i\|_{\mathcal H_i}
\|Ad\|_{\mathcal H_i}.
\]
Dropping the non-negative term $\lambda_{\min}\|d\|_2^2$ immediately gives
\[
\operatorname{MMD}_{k_i}(P_{\widehat\alpha_i},P_{\alpha_i^\star})
=\|Ad\|_{\mathcal H_i}
\leq
\|\widehat\mu_i-\mu_i\|_{\mathcal H_i},
\]
which requires only $\lambda_{\min}\geq 0$. On the other hand, when $\lambda_{\min}>0$, completing the square for $\|Ad\|_{\mathcal H_i}$ yields
\[
\lambda_{\min}\|d\|_2^2
\leq
\|\widehat\mu_i-\mu_i\|_{\mathcal H_i}\|Ad\|_{\mathcal H_i}-\|Ad\|_{\mathcal H_i}^2
\leq
\frac{1}{4}\|\widehat\mu_i-\mu_i\|_{\mathcal H_i}^2,
\]
so that
\[
\|d\|_2
\leq
\frac{\|\widehat\mu_i-\mu_i\|_{\mathcal H_i}}
{2\sqrt{\lambda_{\min}}}.
\]

It remains to bound the concentration error $\|\widehat\mu_i-\mu_i\|_{\mathcal H_i}$. Since $\sup_x k_i(x,x)\leq 1$ and the observations $(X_{t_i,j})_{j=1}^{N_i}$ are i.i.d., Jensen's inequality gives
\[
\E\|\widehat\mu_i-\mu_i\|_{\mathcal H_i}
\leq
\left(\E\|\widehat\mu_i-\mu_i\|_{\mathcal H_i}^2\right)^{1/2}
\leq \frac{1}{\sqrt{N_i}}.
\]
Moreover, replacing a single observation $X_{t_i,j}$ alters $\|\widehat\mu_i-\mu_i\|_{\mathcal H_i}$ by at most $2/N_i$. Applying McDiarmid's inequality therefore ensures that, with probability at least $1-\delta$,
\[
\|\widehat\mu_i-\mu_i\|_{\mathcal H_i}
\leq
\E\|\widehat\mu_i-\mu_i\|_{\mathcal H_i} + \sqrt{\frac{2\log(1/\delta)}{N_i}}
\leq
\frac{1+\sqrt{2\log(1/\delta)}}{\sqrt{N_i}}.
\]
Substituting this high-probability bound into the deterministic stability estimates establishes the pointwise claims. Finally, the uniform bound over $t_0,\ldots,t_m$ follows from a union bound by replacing $\delta$ with $\delta/(m+1)$.
\end{proof}

\section{Supplementary simulation study}\label{sec:simulations_DH2}

Below, we describe the synthetic data, the competing methods against which we benchmark our proposed approach, and the simulation results.

\paragraph*{Data-generating process}

Fix $T=1$ and $d\ge1$. The target density is a 3-component Gaussian mixture with time-varying means and a common time-varying variance:
\begin{equation}\label{eq:mixture}
  f_t(x)
  = \frac{1}{3}\sum_{s=1}^{3}
      \mathcal N\left(x \mid m_s(t),\,\sigma^{2}(t)\,\mathsf{Id}\right),
  \qquad (t,x)\in[0,1]\times\R^d,
\end{equation}
where $\sigma^{2}(t) = 1 + t$. For $0\leq t\leq 1/2$, the component means are
\begin{equation}\label{eq:means}
  m_1(t) = -2 + 20t,\qquad
  m_2(t) = 16t,\qquad
  m_3(t) = 5 + 6t.
\end{equation}
For $1/2<t\leq 1$, all three means equal $8$. When $d\geq 2$, each mean vector repeats the corresponding scalar in every coordinate. The components coincide from $t=1/2$ onward.

We evaluate the models on the regular grid $t_i=i/10$ for $i=0,\dots,10$ and generate $B=100$ independent replicates. At each time, the $n$ observations are allocated as evenly as possible across the three components. Within a replicate, the same Gaussian draws are rescaled and shifted at each time. The sample sizes are
\[
n \in \{20,\,50,\,100,\,200,\,300\}.
\]
This simulation assesses how well a fixed dictionary approximates continuously evolving distributions when the component means and variance vary over time.

For the proposed estimator, we use a working Gaussian mixture with $K=10$ components. The continuous-time mixture weights follow the neural-ODE formulation introduced in Section~3.2.2 of the main article, without additional polynomial or Fourier time features.

\subsection*{Competitors}

We compare the proposed estimator with three baselines for estimating the time-indexed density $f_t$ from snapshot samples observed on the discrete time grid:\newline
(i) a univariate generalized additive model for location, scale, and shape (GAMLSS);\newline
(ii) a time-conditional kernel density estimator (KDE);\newline
(iii) a conditional masked autoregressive flow (MAF).

\paragraph*{Generalized additive models for location, scale, and shape}
We fit a univariate Gaussian distributional regression model with time-varying mean and variance,
\[
  X_t\sim \mathcal N\left(m(t),\,\sigma^2(t)\right),
\]
where $m(\cdot)$ and $\log\sigma(\cdot)$ are modeled as smooth spline functions of $t$ and estimated using the \texttt{gamlss} package in \textsf{R}, see \citep{rigby2005generalized}. We include this baseline only for $d=1$, as multivariate extensions would require additional strong modeling assumptions on the dependence structure.

\paragraph*{Time-conditional KDE}
At each observed time $t_i\in\tau_m$, we estimate $f_{t_i}$ using
\[
  f^{\text{KDE}}_{t_i}(x)
  =\frac{1}{N_i}\sum_{j=1}^{N_i} K_{h_i}\left(x-X_{t_i,j}\right),
  \qquad 
  K_{h}(z)=h^{-d}K(z/h),
\]
using a Gaussian kernel $K$ and Scott's bandwidth rule for $h_i$ \citep{silverman1986density,chacon2018multivariate}. 
For intermediate times $t\in[t_i,t_{i+1}]$, we linearly interpolate the endpoint densities:
\[
  f^{\text{KDE}}_{t}(x)\coloneqq(1-\lambda)f^{\text{KDE}}_{t_i}(x)+\lambda f^{\text{KDE}}_{t_{i+1}}(x),
  \qquad 
  \lambda=\frac{t-t_i}{t_{i+1}-t_i}.
\]
As a fully nonparametric method, KDE is inherently affected by the curse of dimensionality, see \cite{Tsybakov2008Introduction}.

\paragraph*{Masked autoregressive flow} 
We model $f_t$ with a conditional normalizing flow $f_\theta(x\mid t)$ built from an invertible map $x=\Phi^\theta(u;t)$ with base noise $u\sim\mathcal N(0,\mathsf{Id})$. The conditional density is given by the change-of-variables formula:
\[
  f_\theta(x\mid t)
  =\mathcal N\left((\Phi^\theta)^{-1}(x;t)\mid 0,\mathsf{Id}\right)\,
   \left|\det\nabla_x (\Phi^\theta)^{-1}(x;t)\right|.
\]
In MAF \citep{Papamakarios2017Masked}, the inverse map is autoregressive (implemented via masked networks), rendering the Jacobian triangular and the log-likelihood tractable. Parameters $\theta$ are learned via conditional maximum likelihood over the full sample $\{(t_i,X_{t_i,j})\}_{i,j}$, see  \cite{Papamakarios2021Normalizing}.

\subsection*{Results}

We present results for both low- ($d=1$) and high-dimensional ($d=10$) regimes.

\paragraph*{Low dimension ($d=1$)}
\Cref{fig:l2errsd1} shows that our model is competitive with the baselines. While MAF achieves lower average $L^2$ errors across the displayed sample sizes, our method additionally
provides interpretable mixture-weight trajectories $\alpha(t)$; this
interpretability should be distinguished from comparative estimation accuracy.

\begin{figure}[!htbp]
  \centering
    \includegraphics[width=\textwidth]{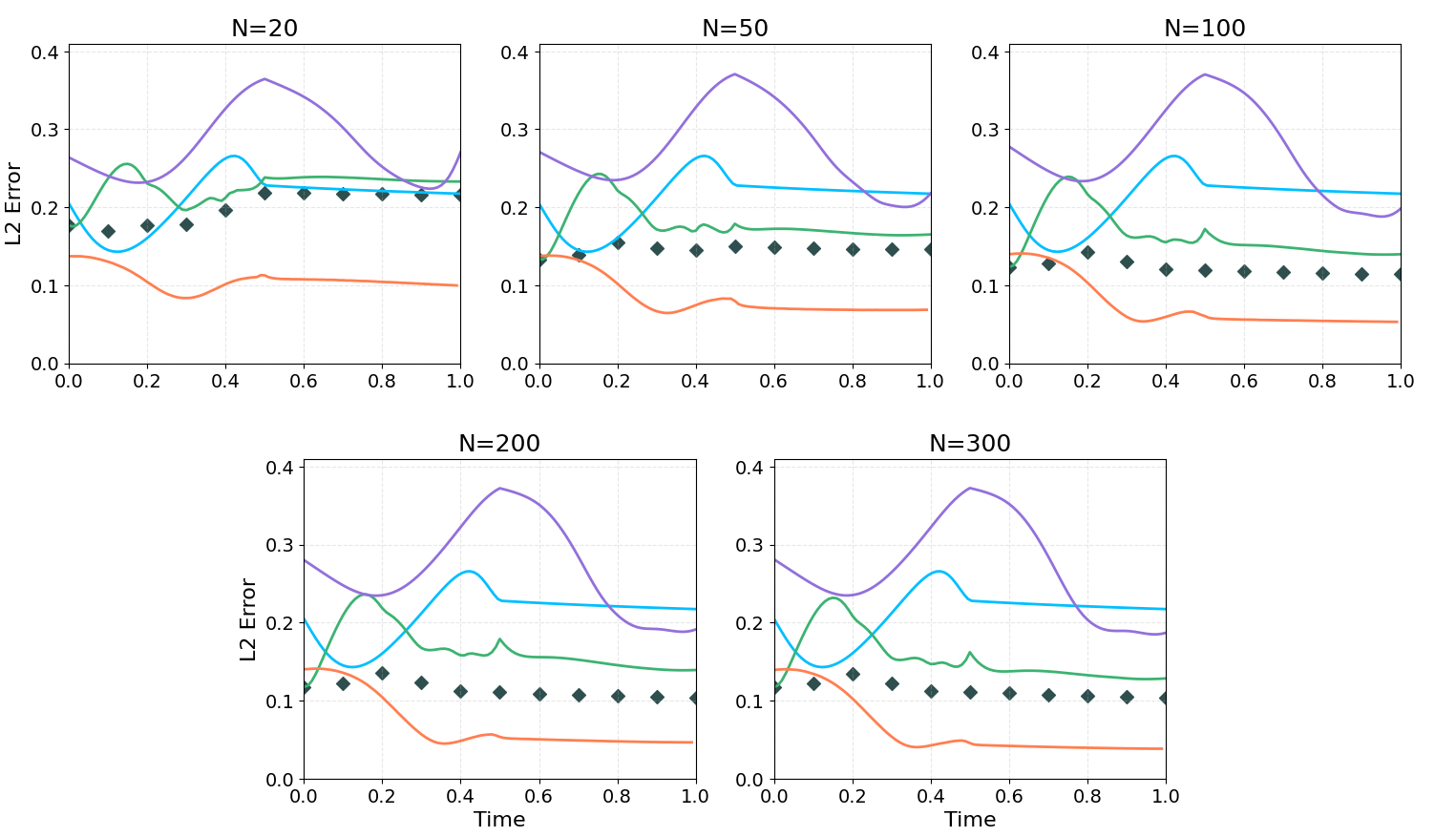}
  \caption{Pointwise $L^2$-error over time for $d=1$. We compare {\color{ForestGreen}our model}, {\color{cyan}KDE}, {\color{orange}MAF}, and {\color{BlueViolet}GAMLSS}. Curves represent averages over $100$ independent runs. Errors corresponding to the discrete-time MMD stage are shown as \rotatebox[origin=c]{45}{$\blacksquare$}.}
  \label{fig:l2errsd1}
\end{figure}

\paragraph*{High dimension ($d=10$)}
\Cref{fig:l2errsd10} shows that our model generally achieves lower $L^2$
errors than the competing estimators, particularly for $n\geq50$. MAF has
the largest errors in these experiments. KDE performs better than MAF but
usually less accurately than our model; at $n=20$, the relative performance
of KDE and our model varies over time.

\begin{figure}[!htbp]
  \centering
    \includegraphics[width=\textwidth]{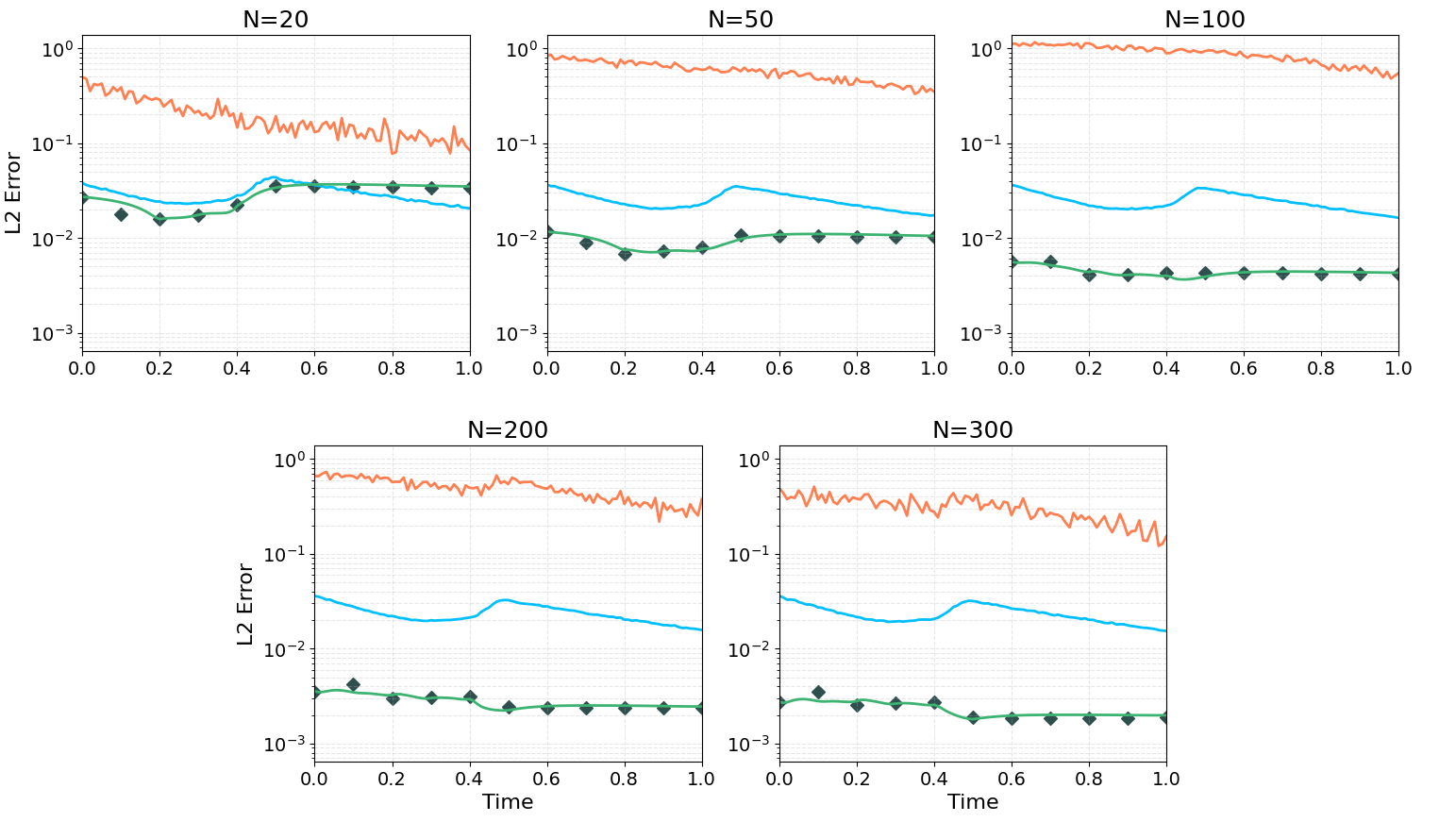}
  \caption{Pointwise $L^2$-error over time for $d=10$. We compare {\color{ForestGreen}our model}, {\color{cyan}KDE} and {\color{orange}MAF}. Curves represent averages over $100$ independent runs. Errors corresponding to the discrete-time MMD stage are shown as \rotatebox[origin=c]{45}{$\blacksquare$}.}
  \label{fig:l2errsd10}
\end{figure}

------------------------------------------

\section{Implementation details}
\label{sec:supp-implementation}

\subsection*{CGM data preprocessing}

Follow-up CGM is divided into 13 nominal non-overlapping weekly blocks anchored
to randomization, from day~4 through the actual Week~13 visit
(or the last valid visit if unavailable). Blocks with fewer than
72 hours of CGM are excluded without imputation. Original weekly
indices and actual block midpoints are retained.

For the bivariate analysis, glucose derivatives are computed before
weekly aggregation using the centered finite difference defined
in the main article. Both adjacent intervals must be within five
seconds of the nominal five-minute cadence. Missing observations
are not interpolated. Device-reported LOW and HIGH values are
coded as 39 and 401 mg/dL, without clipping; duplicate timestamps
are resolved deterministically.

The pre-randomization baseline comprises the preceding 14 calendar
days and requires at least 168 hours of CGM. It is available for
88 participants (60 Treatment and 28 Control). Baseline weights
are estimated separately using the same dictionary, scaling,
participant-specific bandwidth and ridge penalty as the follow-up
fits. Bivariate derivatives use baseline readings alone.
The baseline is excluded from NODE training and supplies the
reference weights for baseline-centered comparisons.

\paragraph*{Baseline--final distributional summaries}
These figures include the same 85 participants (58 Treatment and
27 Control) at both endpoints, for both $d=1$ and $d=2$;
three participants with valid baselines lack a fitted follow-up NODE.
The final distribution is each participant's exact last fitted
NODE state. End times range from 5.5 to 14.14 weeks, rather than
a common endpoint. The functional-test cohort additionally requires
complete support over the inference interval and contains
80 participants.

Fréchet summaries weight participants equally within arm.
For $d=2$, empirical $W_2$ barycenters use 250 samples per participant
and 250 free support points in globally standardized coordinates;
KDEs are transformed back to clinical units with the affine Jacobian.
For $d=1$, mean-quantile barycenters are displayed using a common
KDE bandwidth across arms and endpoints. Signed differences show
final minus baseline.

\subsection*{Model-fitting settings}

For each representation, we pool eligible observations from all usable follow-up blocks, subsampling without replacement when needed to cap each participant's contribution at 2,016 observations. The bivariate coordinates are standardized using the means and standard deviations of the resulting pool. Each shared dictionary is fitted to 10,000 observations drawn without replacement from its pool. Each participant uses one bandwidth across weekly blocks.

Table~\ref{tab:clinical-hyperparams} summarizes
the implementation settings used in the clinical analysis.

\begin{table}[!htbp]
\centering
\small
\caption{Implementation settings for the clinical-data analyses.}
\label{tab:clinical-hyperparams}
\begin{tabular}{ll}
\toprule
Setting & Value \\
\midrule
Mixture components & $K=5$ \\
Local weight optimization & Regularized quadratic program \\
Shared-dictionary optimization & Adam; learning rate $0.01$ \\
MMD outer iterations / gradient steps & $50/50$ \\
Ridge parameter & $\lambda=0.01$ \\
Kernel bandwidth & participant-level median-distance rule\\
Minimum usable CGM per weekly block & 72 hours \\
Minimum weekly observations for NODE fit & 5 usable blocks \\
NODE integration method & Dormand--Prince (\texttt{dopri5}) \\
NODE configured width / layers & 64 / 1 \\
NODE learning rate & $0.01$ \\
NODE loss & cross-entropy on simplex weights \\
NODE maximum epochs & 1000 with early stopping \\
NODE normalization & softmax; trainable initial state \\
Time features & linear time term; Fourier order 2 \\
Output integration grid & $\Delta t=0.01$ on the normalized NODE time scale \\
\bottomrule
\end{tabular}
\end{table}

\section{Statistical inference for treatment-group comparisons}\label{sec:inference}

The graphical conventions follow the main article: \Treatment{} is shown
in green, \Control{} in red, and \CompOne{}, \CompTwo{}, \CompThree{},
\CompFour{}, and \CompFive{} retain their respective colors.

We now turn to statistical inference and describe how the estimated mixture-weight trajectories can be used to compare treatment arms in a randomized clinical trial. For simplicity, suppose that there are two study arms, indexed by \(a\in\{0,1\}\), where \(a=0\) denotes the control group and \(a=1\) the treatment group. Let \(n_a\) be the number of participants in arm \(a\), with the total sample size \(n=n_0+n_1\), and assume that the participants are independent within and between groups. For each participant \(p\in\{1,\ldots,n_a\}\) in arm \(a\), let
\[
\widehat\alpha_k^{(p,a)}(t), \qquad k=1,\ldots,K, \quad t\in I_{p,a},
\]
denote the estimated weight trajectory of the mixture component \(k\). At any fixed time \(t\), the quantity \(\widehat\alpha_k^{(p,a)}(t)\) represents the estimated contribution of component \(k\) to the participant-specific latent distribution. Consequently, the trajectory \(t\mapsto \widehat\alpha_k^{(p,a)}(t)\) summarizes how the participant's distributional profile evolves over time for component $k$. Comparing these trajectories across treatment arms provides a natural and interpretable way to assess treatment-related distributional differences, while also allowing for heterogeneity within each arm.  Here $I_{p,a}$ denotes the fitted support of participant $p$ in arm $a$. Inference is restricted to a fixed eligible cohort and a common interval, as described below.


\subsection*{Common-support functional inference}

Let $\widehat\alpha_i(t)\in\Delta^{K-1}$ denote the estimated neural-ODE mixture-weight trajectory for participant $i$. Because weeks with fewer than 72 hours of CGM data are excluded, participants may have different fitted observation intervals. Although the neural ODE provides continuous trajectories within these intervals, functional comparisons require a common time interval and a fixed set of participants.

We select the longest contiguous interval covered by at least 90\% of participants in each randomized arm, based only on observation availability and treatment assignments. This yields the interval $I=[1.5,11.5]$, retaining 62 of 66 Treatment participants and 31 of 33 Control participants (93 of 99 in total) for both the $d=1$ and $d=2$ analyses. Alternative coverage thresholds of 85\% and 95\% are considered in sensitivity analyses.

For the global functional test, let
$u_r=1.5+0.1r$, $r=0,\ldots,100$,
denote the 101 equally spaced evaluation points in
$I=[1.5,11.5]$. For participant $i$, define
\[
A_i=
\left(
\widehat\alpha_i(u_0),
\ldots,
\widehat\alpha_i(u_{100})
\right).
\]

We compare the distributions of these trajectories between randomized arms using a two-sample maximum mean discrepancy (MMD) test with a Gaussian kernel based on pairwise trajectory distances. Each participant represents one independent sampling unit, and statistical significance is assessed using a participant-level wild bootstrap. The null hypothesis is equality of the distributions of the complete mixture-weight trajectories, rather than equality of their mean trajectories alone. Because the analysis is restricted to participants satisfying the common-support criterion, inference pertains to this selected cohort.

For the global test, squared $L^2$ distances between the fitted trajectories are approximated by trapezoidal quadrature on the 101-point grid, after mapping $I$ linearly to $[0,1]$. The Gaussian-kernel bandwidth is the median strictly positive pairwise $L^2$ distance in the pooled eligible cohort, without using treatment labels. For the pointwise tests, we instead use the 11 weekly centers $t_j=1.5+(j-1)$, $j=1,\ldots,11$.

We additionally perform a baseline-centered analysis to compare individual changes in mixture weights relative to pre-randomization glucose profiles. The baseline period consists of the 14 calendar days immediately preceding randomization, excluding the randomization day, and requires at least 168 hours of valid CGM observations. These criteria yield valid baselines for 88 of the 102 randomized participants (60 \Treatment{} and 28 \Control{}). Missing baseline observations are not imputed.

For each eligible participant, let $\alpha_i^{\textrm{base}}$ denote the baseline mixture weights obtained by projecting the baseline distribution onto the same Gaussian dictionary, using the previously estimated participant-specific bandwidth and ridge parameter without refitting the dictionary. We define the baseline-centered trajectory as
\[
A_i^\Delta=
\left(
\widehat\alpha_i(u_0)-\alpha_i^{\textrm{base}},
\ldots,
\widehat\alpha_i(u_{100})-\alpha_i^{\textrm{base}}
\right).
\]
For the bivariate analysis ($d=2$), glucose derivatives are estimated using only pre-randomization observations, ensuring that no finite-difference calculation crosses the randomization boundary.

The baseline-centered comparison includes participants with both a valid baseline and complete fitted support over $I$, resulting in 80 participants (55 \Treatment{} and 25 \Control{}). Participants without a valid baseline remain eligible for the primary analysis. The corresponding MMD test compares the distributions of baseline-centered trajectories between treatment groups.

\subsection*{Wild-bootstrap calibration}

We calibrate the MMD statistics using a participant-level wild bootstrap \citep{NIPS2014_4e382cb4}. For each bootstrap replicate, independent Rademacher multipliers are generated for all participants and centered within each treatment arm. Each pairwise kernel contribution is weighted by the product of the corresponding centered multipliers.

The same centered multiplier is applied to a participant's entire fitted trajectory and is reused across evaluation times and mixture components. Thus, resampling is performed at the participant level rather than independently at each time point. This construction preserves the participant as the sampling unit and avoids treating repeated measurements as independent observations.

We use $B=9{,}999$ bootstrap replicates for all reported tests. For an observed statistic $T$ and bootstrap statistics $T^{*(1)},\ldots,T^{*(B)}$, the Monte Carlo $p$-value is computed as
\[
\widehat p
=
\frac{
1+\sum_{b=1}^{B}
\mathbf{1}\{T^{*(b)}\geq T\}
}{
B+1
},
\qquad B=9{,}999.
\]
This finite-resampling correction gives a minimum attainable $p$-value of $0.0001$. The resampling comparisons are conditional on the fitted trajectories and shared dictionary; neither is re-estimated in the bootstrap replicates.

\subsection*{Componentwise localization}

To further investigate the differences identified by the global functional test, we perform comparisons of individual mixture components over time. These analyses aim to identify which components and evaluation periods contribute to the observed distributional differences.

For each component $k\in\{1,\ldots,K\}$ and evaluation time $t_j\in I$, we consider the estimated mixture weights
\[
\mathcal X_{k,j}
=
\left\{
\widehat\alpha_k^{(p,0)}(t_j):
p=1,\ldots,n_0
\right\},
\qquad
\mathcal Y_{k,j}
=
\left\{
\widehat\alpha_k^{(q,1)}(t_j):
q=1,\ldots,n_1
\right\},
\]
corresponding to the Control and Treatment groups, respectively.

We test the pointwise null hypothesis
\[
H_{0,k,j}:
\mathcal L\{\alpha_k(t_j)\;\colon\; a=0\}
=
\mathcal L\{\alpha_k(t_j)\;\colon\; a=1\},
\]
where $\mathcal L$ denotes the probability distribution. This hypothesis concerns the full distribution of each mixture weight, rather than its mean alone, allowing the test to detect differences in location, variability, or distributional shape.

We use the Gaussian kernel
\[
\kappa_\sigma(x,y)
=
\exp\left(
-\frac{\|x-y\|_2^2}{2\sigma^2}
\right),
\]
where $\sigma>0$ is selected using the median heuristic applied to the pooled sample.

Writing $X_p=\widehat\alpha_k^{(p,0)}(t_j)$ and $Y_q=\widehat\alpha_k^{(q,1)}(t_j)$, the empirical squared MMD is
\[
\begin{aligned}
\widehat{\mathrm{MMD}}_{k,j}^{\,2}
={}&
\frac{1}{n_0^2}
\sum_{p,p'=1}^{n_0}
\kappa_\sigma(X_p,X_{p'})
+
\frac{1}{n_1^2}
\sum_{q,q'=1}^{n_1}
\kappa_\sigma(Y_q,Y_{q'})
-
\frac{2}{n_0n_1}
\sum_{p=1}^{n_0}
\sum_{q=1}^{n_1}
\kappa_\sigma(X_p,Y_q).
\end{aligned}
\]

We consider the scaled statistic
\[
T_{k,j}
=
\frac{n_0n_1}{n_0+n_1}
\widehat{\mathrm{MMD}}_{k,j}^{\,2},
\]
where larger values indicate greater discrepancies between the two group-specific distributions.

For the baseline-centered analysis, the same procedure is applied to
\[
\widehat\alpha_k^{(i,a)}(t_j)
-
\alpha_k^{(i,a,0)},
\]
where $\alpha_k^{(i,a,0)}$ is the participant's baseline mixture weight. Although these differences may be negative, they remain valid inputs to the Gaussian kernel.

Pointwise $p$-values are obtained using the participant-level wild bootstrap described above, with the same multipliers reused across times and components within each replicate. 

\subsection*{Cumulative ACAT aggregation over time}
\label{sec:cumulative-acat}

The componentwise MMD tests provide a sequence of pointwise $p$-values over time. To summarize how evidence of between-group differences evolves within each mixture component, we use the aggregated Cauchy association test (ACAT) \citep{Liu02012020}.

Let $p_k(t_j)$ denote the wild-bootstrap $p$-value for component $k$ at evaluation time $t_j$. We first apply the Cauchy transformation
\begin{equation}
\label{eq:acat-transform}
X_k(t_j)
=
\tan\left[
\pi\left\{
\frac12-p_k(t_j)
\right\}
\right].
\end{equation}
Small $p$-values correspond to large positive transformed values. Under a uniform null $p$-value, this transformation follows a standard Cauchy distribution. For dependent $p$-values, ACAT provides an upper-tail approximation under suitable conditions \citep{Liu02012020}, although exact finite-sample calibration is not guaranteed under arbitrary dependence.

To construct cumulative summaries, we combine the available pointwise results up to each evaluation time $t_m$. We assign exponentially decaying weights,
\begin{equation}
\label{eq:acat-weights}
w_{k,mj}
\propto
\exp\left\{
-\frac{7(t_m-t_j)}{\tau}
\right\}
\mathbf{1}\{t_j\leq t_m\}
\mathbf{1}\{p_k(t_j)\text{ is testable}\},
\qquad
\sum_{j=1}^{m}w_{k,mj}=1,
\end{equation}
where time is measured in weeks and $\tau$ is expressed in days. This weighting scheme emphasizes recent evidence while progressively downweighting earlier observations. Only current and previous evaluation times contribute to the cumulative summary.

The cumulative ACAT statistic is
\begin{equation}
\label{eq:acat-statistic}
T_k(t_m)
=
\sum_{j\leq m}
w_{k,mj}X_k(t_j),
\end{equation}
with corresponding combined $p$-value approximation
\begin{equation}
\label{eq:acat-pvalue}
p_k^{\mathrm{cum}}(t_m)
=
\frac12
-
\frac{1}{\pi}
\arctan\{T_k(t_m)\}.
\end{equation}

Untestable time points receive zero weight, and the remaining weights are renormalized. Consequently, missing or untestable evaluations do not contribute to the cumulative statistic.

The parameter $\tau$ controls how strongly the cumulative summary emphasizes recent evidence. Our primary analysis uses $\tau=7$ days. We additionally consider
\[
\tau\in\{1,3,5,7,10,14\}\ \text{days}
\]
on both weekly and fine evaluation grids to assess sensitivity to the temporal weighting scheme. Smaller values emphasize recent results, whereas larger values retain greater contributions from earlier evaluation times.

The resulting curves provide exploratory summaries of how evidence for between-group differences evolves over time within each mixture component. They complement, but do not replace, the primary global functional MMD test. In particular, cumulative ACAT values are not interpreted as independently calibrated sequential tests, and nominal threshold crossings are not considered multiplicity-adjusted findings across components or evaluation times.

\section{Additional clinical results and sensitivity analyses}
\label{sec:supp-case}

\subsection*{Goodness of fit}

Because the model is a function of time, we assess goodness of fit across the longitudinal analysis window rather than only at its endpoints. For participant $i$ and usable weekly block $t$, let $\widehat F_{it}$ be the fitted joint distribution of $(G_{it},\dot G_{it})$. Since direct evaluation of a bivariate distribution is generally not uniform even under correct specification, we apply the Rosenblatt transform \citep{Rosenblatt1952Remarks}. For each observation $(G_{it\ell},\dot G_{it\ell})$ in a fitted weekly window, we define
\[
U_{1,it\ell}=\widehat F^{G}_{it}(G_{it\ell}),
\qquad
U_{2,it\ell}=\widehat F^{\dot G\mid G}_{it}(\dot G_{it\ell}\mid G_{it\ell}),
\]
\noindent where $\widehat F^{G}_{it}$ is the marginal distribution of glucose and $\widehat F^{\dot G\mid G}_{it}$ is the conditional distribution of its rate of change. Under the model distribution, both transformed coordinates have uniform marginal distributions on $[0,1]$.

The Rosenblatt-transform diagnostics are used as descriptive assessments of the fitted distributions rather than formal goodness-of-fit tests. We rely on graphical inspection to assess whether the fitted models are reasonably well calibrated in practice. As shown in Figure \ref{fig:rosenblatt_gof}, the summary statistics of the transformed observations remain approximately consistent with the expected uniform distribution over time for multiple participants. These findings provide no clear graphical evidence of substantial model misspecification and support the adequacy of the fitted distributions in representing the observed CGM data.

\begin{figure}[!htbp]
    \centering
    \begin{subfigure}[b]{0.31\textwidth}
        \includegraphics[width=\textwidth]{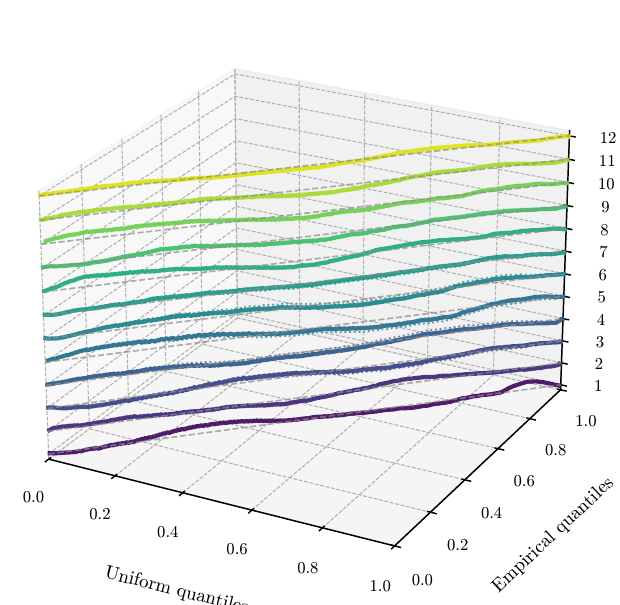}
        \caption{Participant 16 (\Treatment).}
    \end{subfigure}\hfill
    \begin{subfigure}[b]{0.31\textwidth}
        \includegraphics[width=\textwidth]{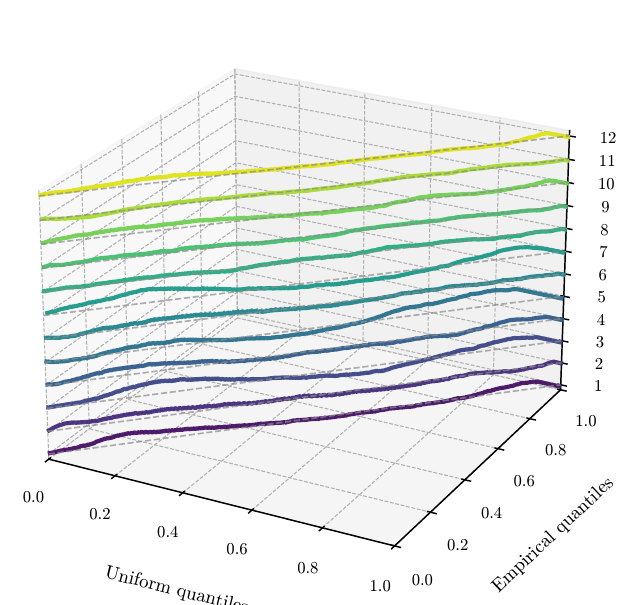}
        \caption{Participant 26 (\Treatment).}
    \end{subfigure}\hfill
    \begin{subfigure}[b]{0.31\textwidth}
        \includegraphics[width=\textwidth]{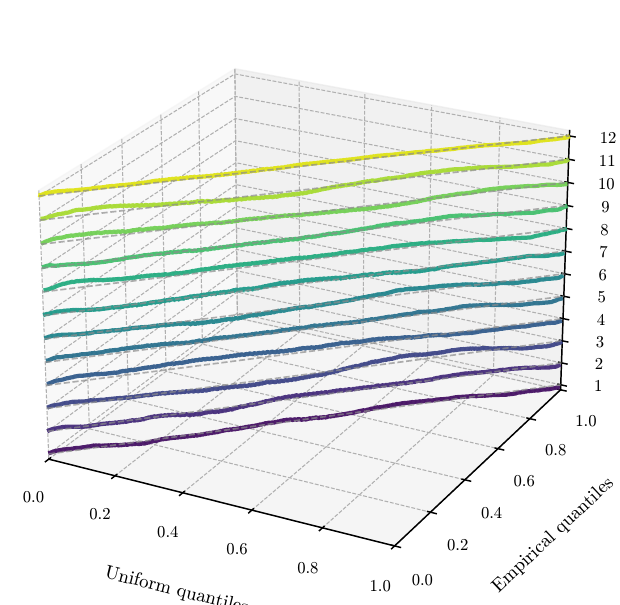}
        \caption{Participant 39 (\Control).}
    \end{subfigure}

    \medskip
    \begin{subfigure}[b]{0.31\textwidth}
        \includegraphics[width=\textwidth]{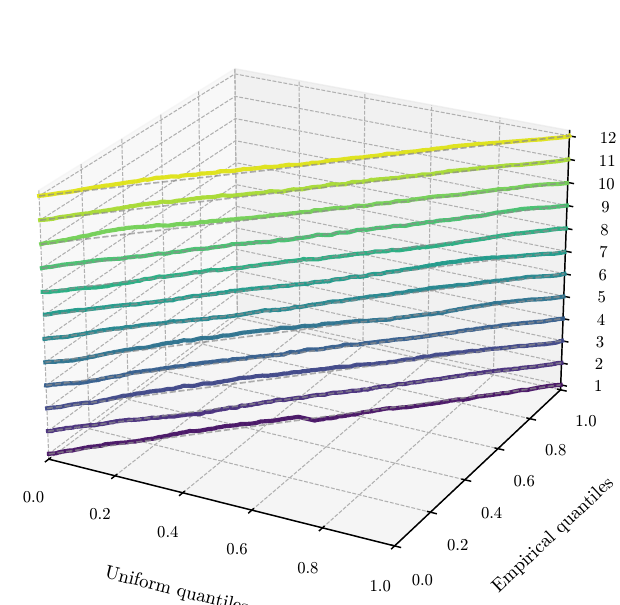}
        \caption{Participant 16 (\Treatment).}
    \end{subfigure}\hfill
    \begin{subfigure}[b]{0.31\textwidth}
        \includegraphics[width=\textwidth]{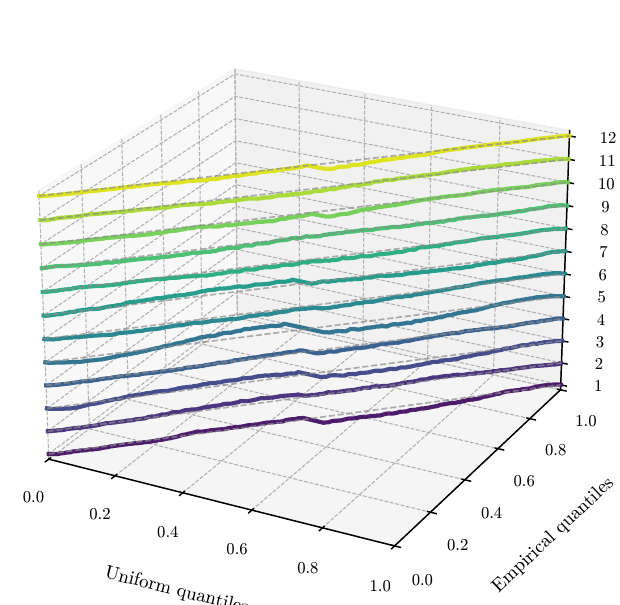}
        \caption{Participant 26 (\Treatment).}
    \end{subfigure}\hfill
    \begin{subfigure}[b]{0.31\textwidth}
        \includegraphics[width=\textwidth]{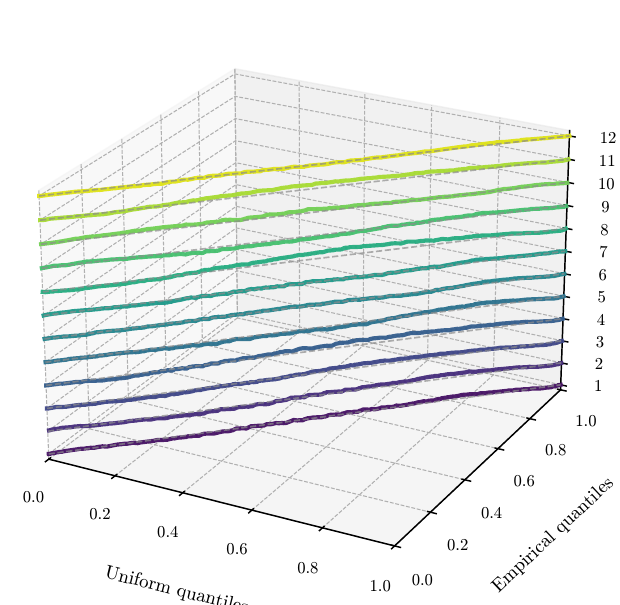}
        \caption{Participant 39 (\Control).}
    \end{subfigure}
    \caption{Rosenblatt QQ diagnostics for three representative participants in the bivariate model. The top row corresponds to the glucose marginal transform $U_1$ and the bottom row to the conditional rate-of-change transform $U_2$. Alignment with the diagonal indicates better calibration. These are descriptive in-sample diagnostics.}
    \label{fig:rosenblatt_gof}
\end{figure}

\subsection*{Global functional tests}

Table~\ref{tab:supp-global-mmd} reports the global functional MMD statistics
and participant-level wild-bootstrap $p$-values for the absolute and
baseline-centered comparisons, using both the bivariate and univariate
representations.

\begin{table}[!htbp]
\centering
\caption{Global functional MMD results. The bivariate representation
(\(d=2\)) is primary; \(d=1\) assesses robustness to excluding the
glucose-rate coordinate. Only wild-bootstrap calibration is reported.}
\label{tab:supp-global-mmd}
\begin{tabular}{lccc}
\hline
Dimension & Estimand & MMD$^2$ & Wild-bootstrap $p$ \\
\hline
$d=2$ & Absolute & 0.1298 & 0.0001 \\
$d=2$ & Change from baseline & 0.0872 & 0.0001 \\
$d=1$ & Absolute & 0.1446 & 0.0001 \\
$d=1$ & Change from baseline & 0.1105 & 0.0001 \\
\hline
\end{tabular}
\end{table}

\subsection*{Univariate (\(d=1\)) sensitivity analysis}

For the glucose-only representation, Figure~\ref{fig:supp-d1-weights}
displays individual and arm-level mixture-weight trajectories. The
corresponding time-varying mixture mean and variance are presented in
Figure~\ref{fig:supp-d1-moments}. Figure~\ref{fig:supp-d1-frechet}
shows the baseline-to-final Fr\'echet summaries, while
Figure~\ref{fig:supp-d1-acat} gives the componentwise cumulative ACAT
summaries. These panels use the univariate representation as a sensitivity
analysis of the primary bivariate model.

\begin{figure}[!htbp]
    \centering
    \begin{subfigure}[t]{0.48\textwidth}
        \includegraphics[width=\linewidth]{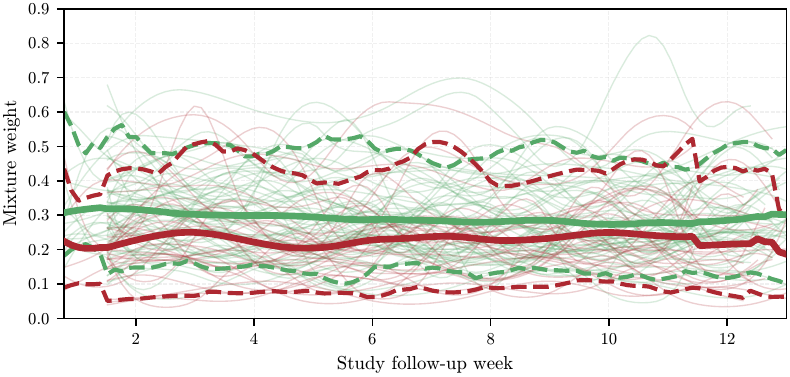}
        \caption{\CompOne}
    \end{subfigure}\hfill
    \begin{subfigure}[t]{0.48\textwidth}
        \includegraphics[width=\linewidth]{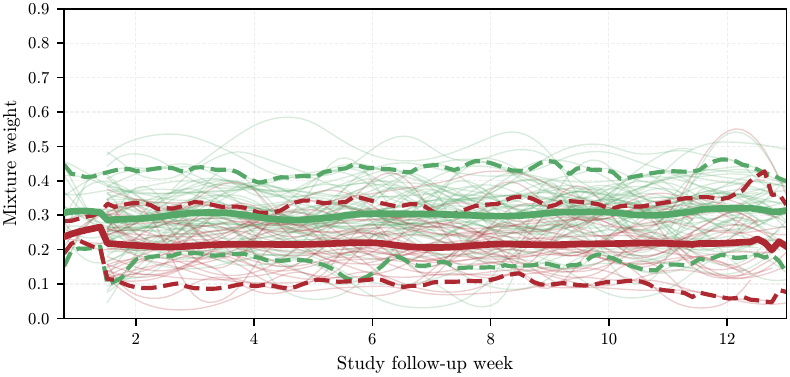}
        \caption{\CompTwo}
    \end{subfigure}
    \medskip
    \begin{subfigure}[t]{0.48\textwidth}
        \includegraphics[width=\linewidth]{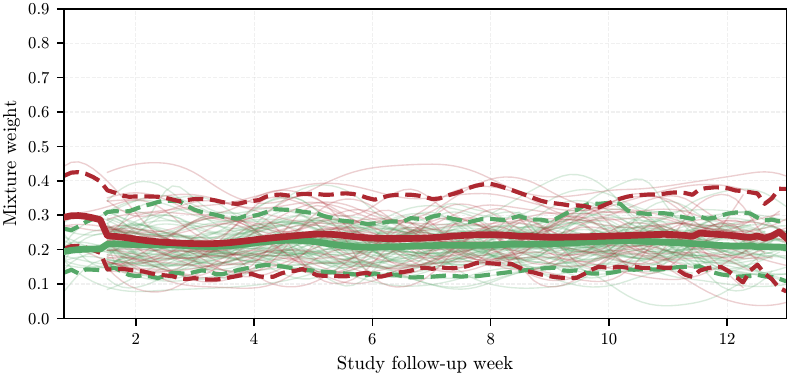}
        \caption{\CompThree}
    \end{subfigure}\hfill
    \begin{subfigure}[t]{0.48\textwidth}
        \includegraphics[width=\linewidth]{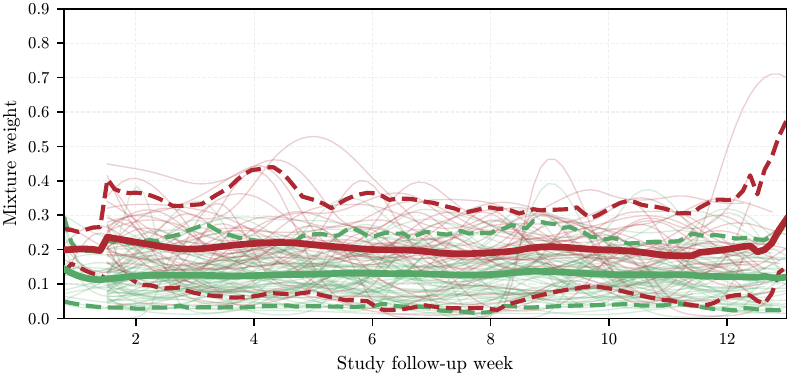}
        \caption{\CompFour}
    \end{subfigure}
    \medskip
    \begin{subfigure}[t]{0.48\textwidth}
        \includegraphics[width=\linewidth]{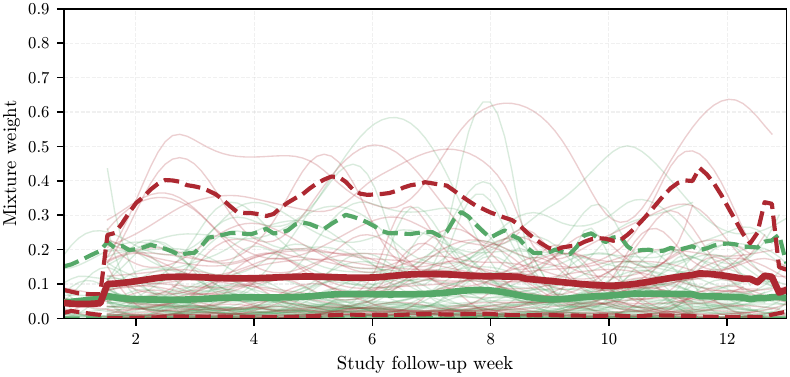}
        \caption{\CompFive}
    \end{subfigure}
    \caption{Univariate mixture-weight trajectories. Thick curves show the
    available-case means for the \Treatment{} and \Control{} arms. Thin curves
    show individual trajectories, and dashed curves show the 5th and 95th
    percentiles across available participants; these are not confidence
    intervals for the arm means.}
    \label{fig:supp-d1-weights}
\end{figure}

\begin{figure}[!htbp]
    \centering
    \begin{subfigure}[t]{0.4\textwidth}
        \includegraphics[width=\linewidth]{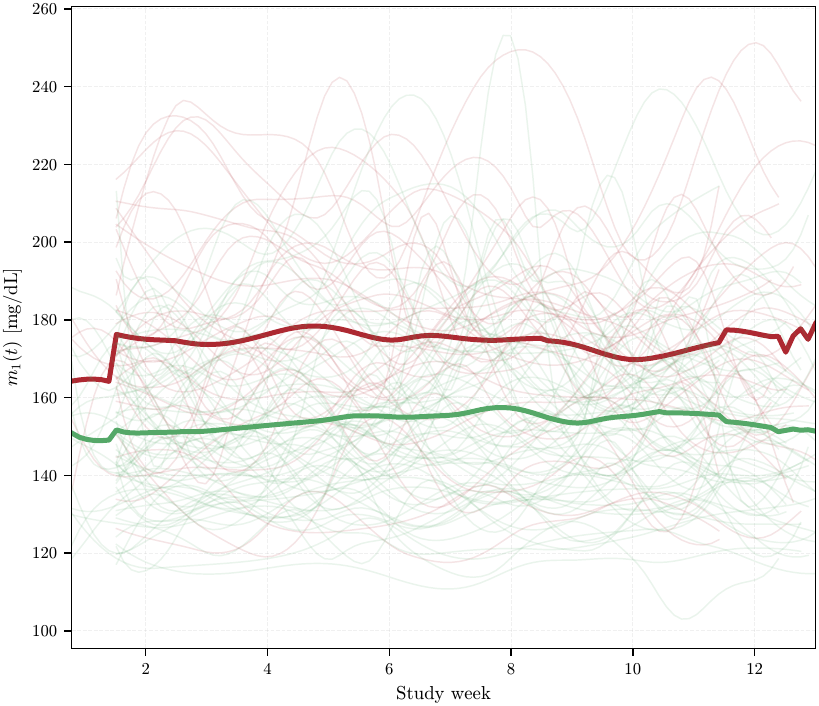}
        \caption{Global mean.}
    \end{subfigure}\hfill
    \begin{subfigure}[t]{0.4\textwidth}
        \includegraphics[width=\linewidth]{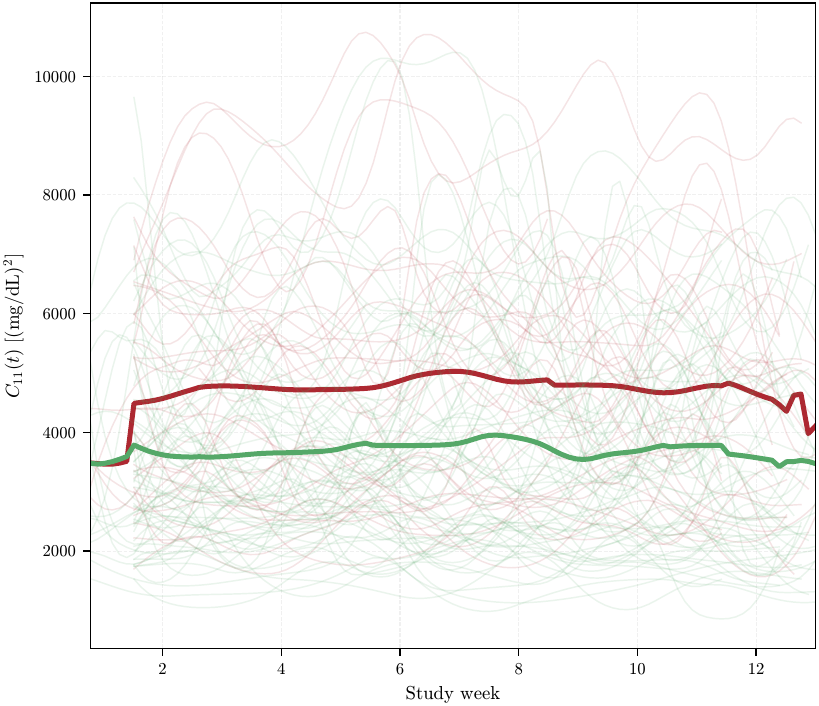}
        \caption{Global variance.}
    \end{subfigure}
    \caption{Univariate time-varying mixture moments for the \Treatment{} and
    \Control{} arms.}
    \label{fig:supp-d1-moments}
\end{figure}

\begin{figure}[!htbp]
    \centering
    \includegraphics[width=0.92\textwidth]{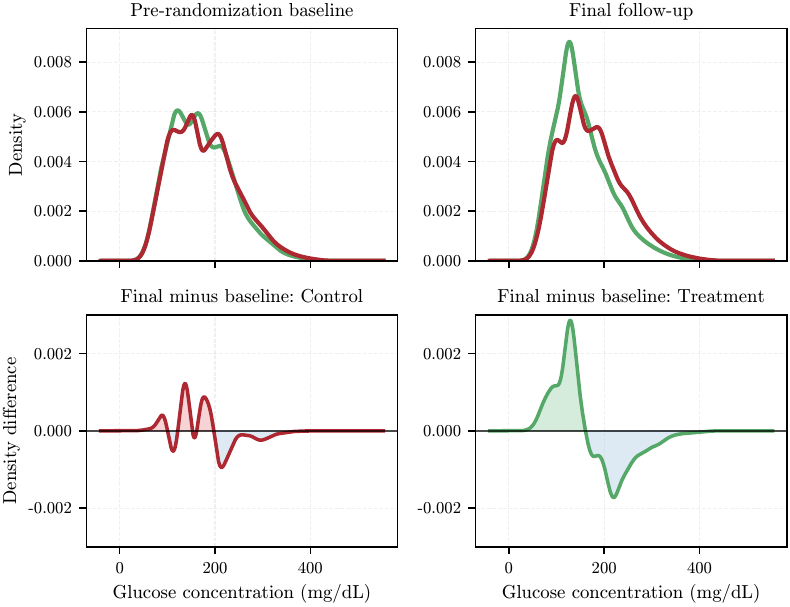}
    \caption{Univariate 2-Wasserstein Fréchet summaries comparing the pre-randomization baseline with the final fitted follow-up state. The same 58 \Treatment{} and 27 \Control{} participants contribute to both endpoints. Signed panels show final minus baseline; densities are displayed using a common-bandwidth KDE of the mean-quantile barycenters.}
    \label{fig:supp-d1-frechet}
\end{figure}

\begin{figure}[!htbp]
    \centering
    \begin{subfigure}[t]{0.4\textwidth}
        \includegraphics[width=\linewidth]{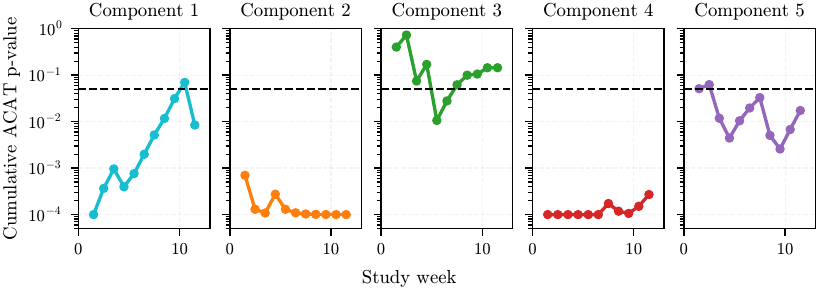}
        \caption{Absolute trajectories.}
    \end{subfigure}\hfill
    \begin{subfigure}[t]{0.4\textwidth}
        \includegraphics[width=\linewidth]{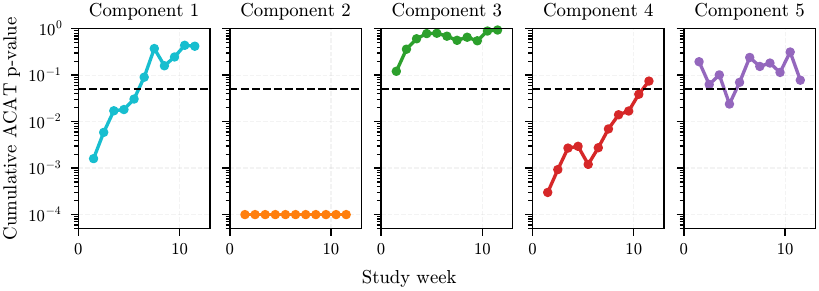}
        \caption{Change from baseline.}
    \end{subfigure}
    \caption{Univariate componentwise cumulative ACAT summaries of the pointwise
    wild-bootstrap $p$-values, with $\tau=7$ days. These exploratory
    summaries assess sensitivity to the distributional representation;
    they are not separately bootstrap-calibrated sequential tests.}
    \label{fig:supp-d1-acat}
\end{figure}

\subsection*{Additional bivariate summaries}

Figure~\ref{fig:supp-d2-frechet-detailed} provides the separate baseline,
final-follow-up, and signed-difference panels for the bivariate Fr\'echet
summaries. Figure~\ref{fig:supp-d2-crosscov} displays the within-component
and total component contributions to the glucose--rate cross-covariance.

\begin{figure}[!htbp]
    \centering
    \begin{subfigure}[t]{0.32\textwidth}
        \includegraphics[width=\linewidth]{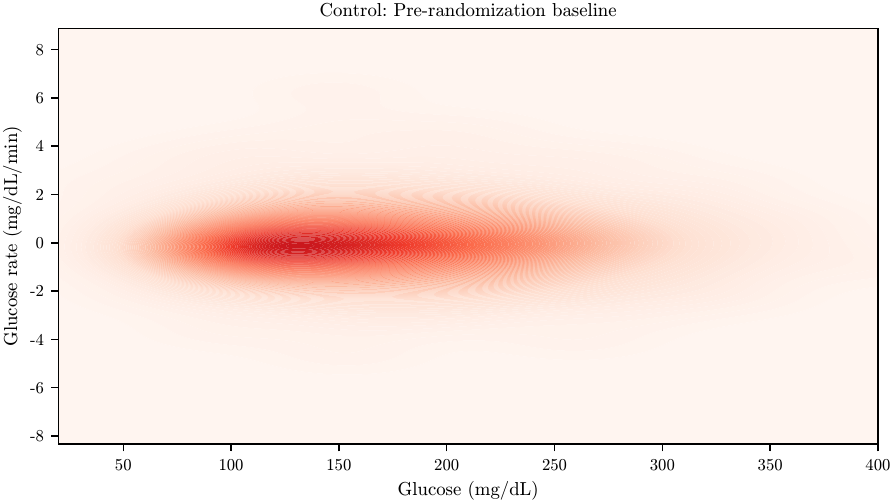}
        \caption{\Control{}: pre-randomization baseline.}
    \end{subfigure}\hfill
    \begin{subfigure}[t]{0.32\textwidth}
        \includegraphics[width=\linewidth]{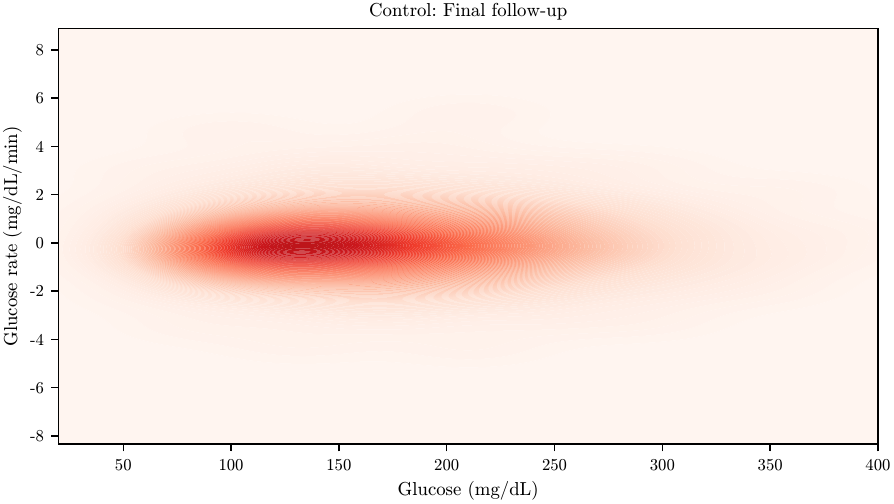}
        \caption{\Control{}: final fitted state.}
    \end{subfigure}\hfill
    \begin{subfigure}[t]{0.32\textwidth}
        \includegraphics[width=\linewidth]{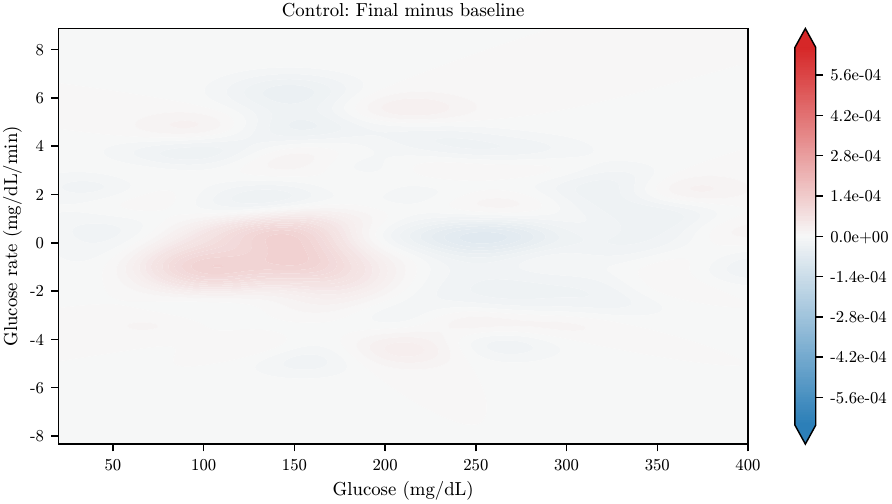}
        \caption{\Control{}: final minus baseline.}
    \end{subfigure}

    \medskip
    \begin{subfigure}[t]{0.32\textwidth}
        \includegraphics[width=\linewidth]{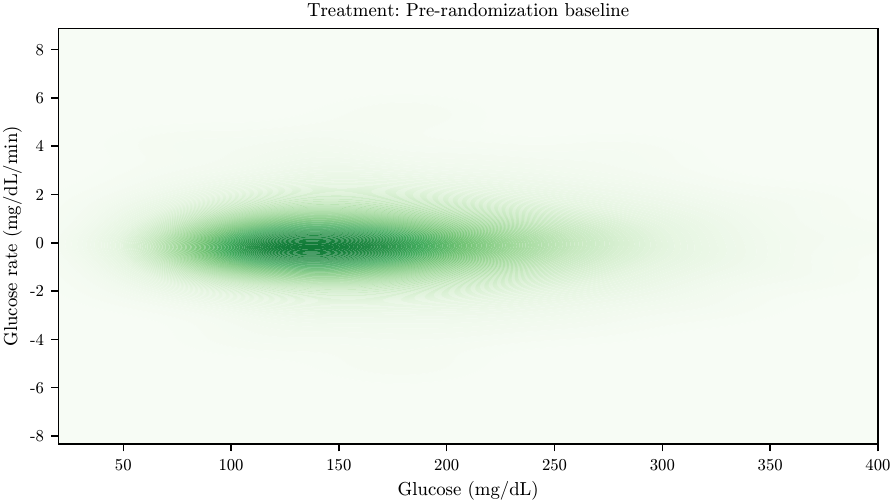}
        \caption{\Treatment{}: pre-randomization baseline.}
    \end{subfigure}\hfill
    \begin{subfigure}[t]{0.32\textwidth}
        \includegraphics[width=\linewidth]{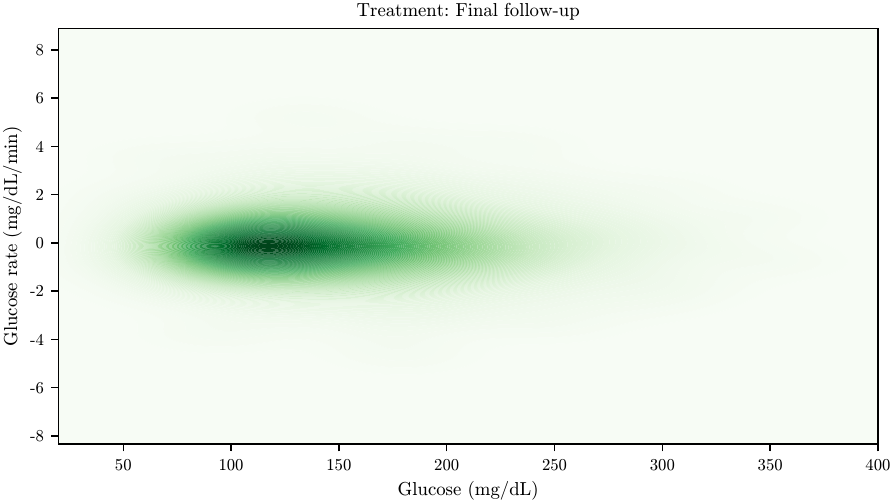}
        \caption{\Treatment{}: final fitted state.}
    \end{subfigure}\hfill
    \begin{subfigure}[t]{0.32\textwidth}
        \includegraphics[width=\linewidth]{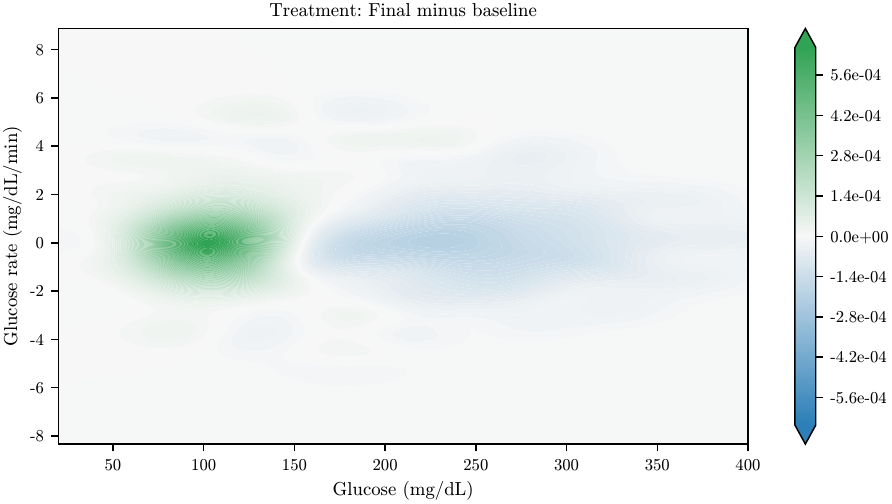}
        \caption{\Treatment{}: final minus baseline.}
    \end{subfigure}
     \caption{Detailed bivariate 2-Wasserstein Fréchet summaries comparing pre-randomization baseline with the exact last fitted follow-up state, with the same 27 \Control{} and 58 \Treatment{} participants at both endpoints. Baseline distributions use the 14 days before randomization and at least 168 hours of CGM; final times vary between participants. The displayed densities are KDEs of numerically approximated barycenters in globally standardized coordinates, transformed to clinical units. Signed panels show final minus baseline.}
    \label{fig:supp-d2-frechet-detailed}
\end{figure}
\begin{figure}[!htbp]
    \centering
    \begin{subfigure}[t]{0.75\textwidth}
        \centering
        \includegraphics[width=\linewidth]{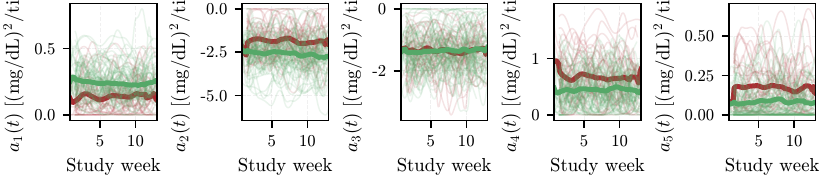}
        \caption{Within-component contribution.}
    \end{subfigure}

    \vspace{1em} 

    \begin{subfigure}[t]{0.75\textwidth}
        \centering
        \includegraphics[width=\linewidth]{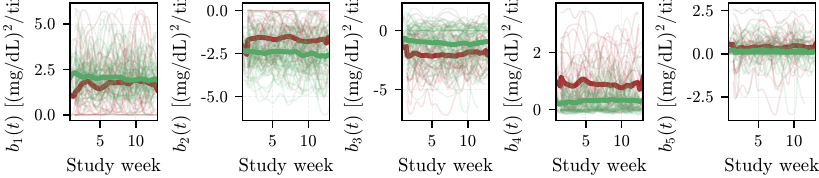}
        \caption{Total component contribution.}
    \end{subfigure}

    \caption{Component contributions to the glucose--rate cross-covariance,
    averaged separately for the \Treatment{} and \Control{} arms.}
    \label{fig:supp-d2-crosscov}
\end{figure}

Although separate univariate analyses of glucose concentration (glucodensity) detect significant differences between treatment groups, they do not capture the dependence between glucose and the rate of change. In particular, glucodensity alone does not characterize glucose dynamics, while the marginal distribution of glucose velocity does not describe how the rate of change depends on glucose concentration. The bivariate model addresses this limitation by jointly modeling both biomarkers and their covariance structure, which provides additional information about their association. The covariance decomposition further distinguishes dependence within mixture components from that induced by differences in their means, allowing a more detailed interpretation of the joint glucose dynamics. This is particularly relevant from a clinical perspective, as glucose dynamics may provide information about physiological processes related to insulin sensitivity and pancreatic beta-cell function that cannot be inferred from glucose concentration alone. Moreover, the association between glucose concentration and its rate of change may vary across treatment groups, individuals, and follow-up time, revealing differences in glycemic regulation that separate univariate analyses cannot capture. Thus, even when univariate analyses identify significant treatment differences, the bivariate representation provides additional clinical and structural information by characterizing how glucose concentration and its rate of change are associated.

\subsection*{Inference-window sensitivity}

Figure~\ref{fig:supp-window-sensitivity} shows how the common interval
and the number of eligible participants vary with the minimum per-arm
coverage threshold. The primary analysis uses the 90\% threshold described
in Section~\ref{sec:inference}.

\begin{figure}[!htbp]
    \centering
    \includegraphics[width=0.75\textwidth]{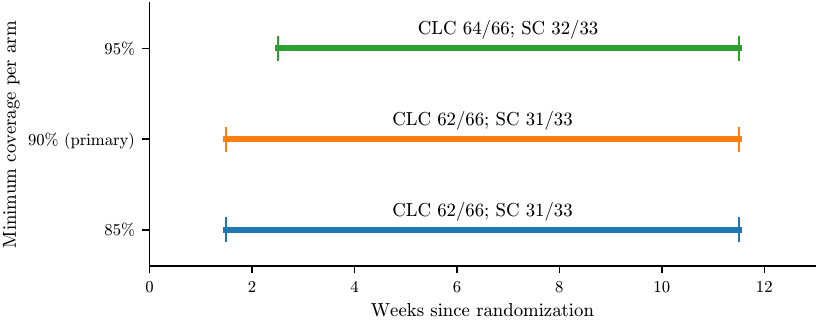}
    \caption{Sensitivity of the common inference interval to the minimum
    per-arm coverage threshold. The figure reports the selected interval and
    retained \Treatment{}/\Control{} counts at 85\%, 90\%, and 95\%.
    Selection uses fitted supports and arm labels only.}
    \label{fig:supp-window-sensitivity}
\end{figure}

\subsection*{Componentwise temporal evidence}

Figure~\ref{fig:supp-d2-component-inference} presents pointwise
wild-bootstrap results and cumulative ACAT summaries for each mixture
component. Figure~\ref{fig:supp-acat-tau} compares the cumulative
summaries over the specified decay parameters and evaluation grids.

\begin{figure}[!htbp]
    \centering
    \begin{subfigure}[t]{0.49\textwidth}
        \includegraphics[width=\linewidth]{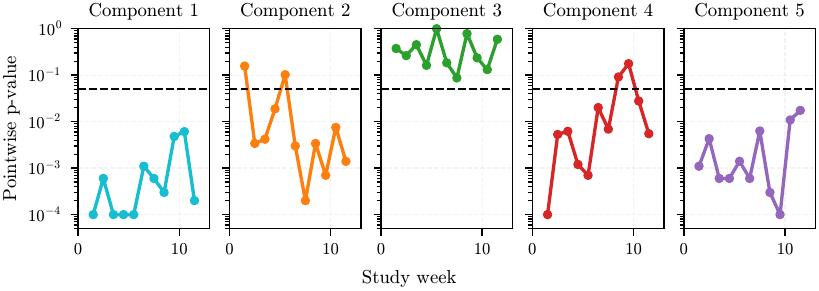}
        \caption{Absolute: pointwise.}
    \end{subfigure}\hfill
    \begin{subfigure}[t]{0.49\textwidth}
        \includegraphics[width=\linewidth]{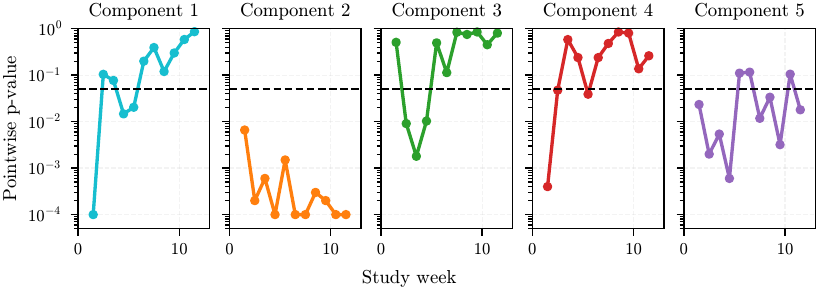}
        \caption{Change from baseline: pointwise.}
    \end{subfigure}
    \medskip
    \begin{subfigure}[t]{0.49\textwidth}
        \includegraphics[width=\linewidth]{Pictures/AoAS/d2/wild_bootstrap/wild_bootstrap_mmd_absolute_weekly_cumulative_acat_pvalues.pdf}
        \caption{Absolute: cumulative ACAT.}
    \end{subfigure}\hfill
    \begin{subfigure}[t]{0.49\textwidth}
        \includegraphics[width=\linewidth]{Pictures/AoAS/d2/wild_bootstrap/wild_bootstrap_mmd_change_from_baseline_weekly_cumulative_acat_pvalues.pdf}
        \caption{Change from baseline: cumulative ACAT.}
    \end{subfigure}
    \caption{Componentwise temporal localization in the primary bivariate
    analysis. Colors correspond to \CompOne{}, \CompTwo{}, \CompThree{},
    \CompFour{}, and \CompFive{}.}
    \label{fig:supp-d2-component-inference}
\end{figure}

\begin{figure}[!htbp]
    \centering
    \begin{subfigure}[t]{0.49\textwidth}
        \includegraphics[width=\linewidth]{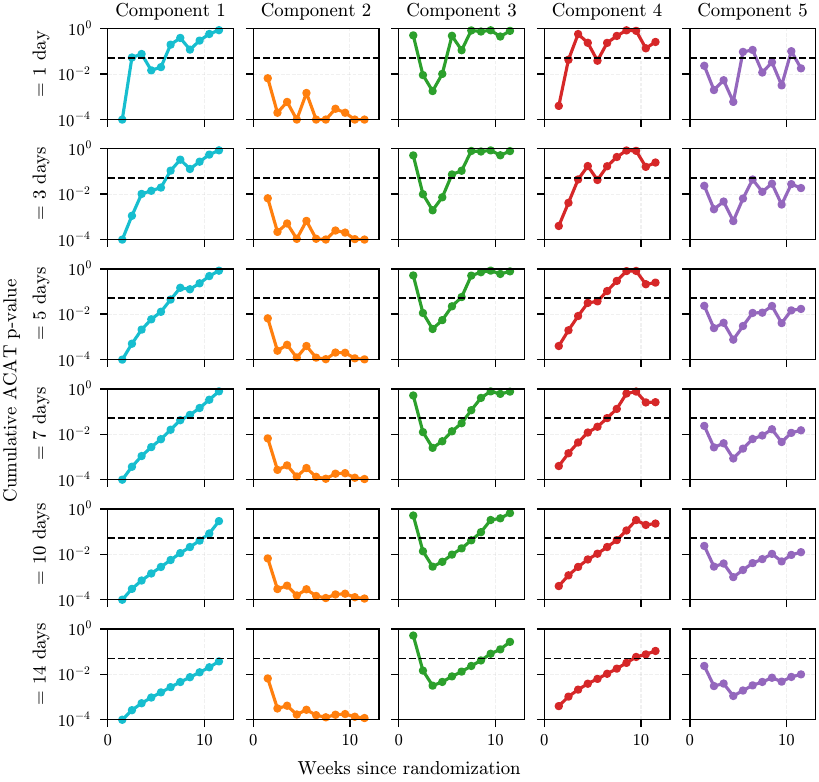}
        \caption{Weekly grid.}
    \end{subfigure}\hfill
    \begin{subfigure}[t]{0.49\textwidth}
        \includegraphics[width=\linewidth]{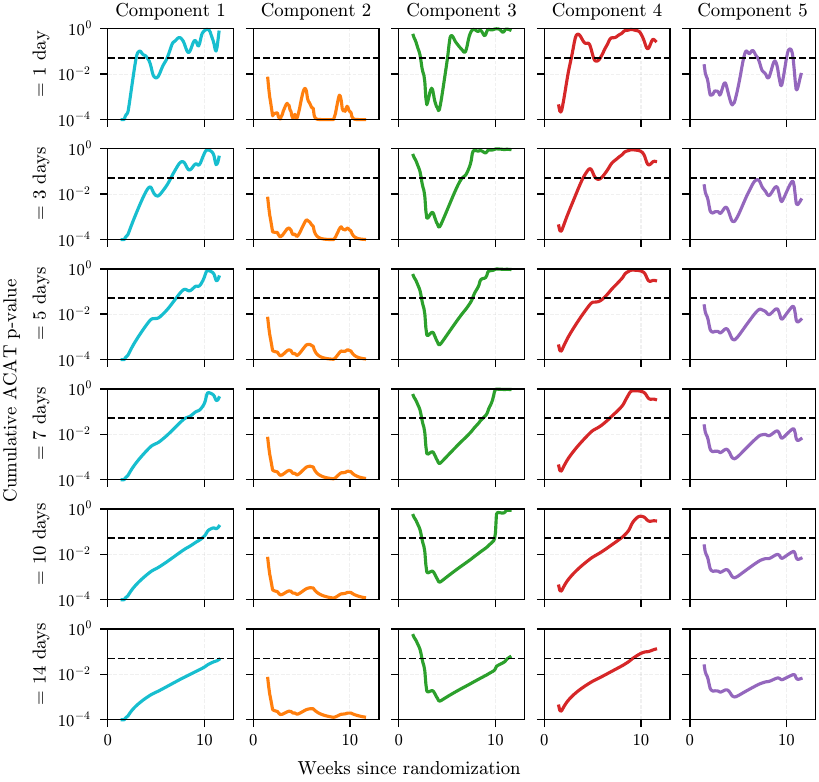}
        \caption{Fine grid.}
    \end{subfigure}
    \caption{Sensitivity of cumulative ACAT localization for the
    change-from-baseline estimand to
    \(\tau\in\{1,3,5,7,10,14\}\) days and to the evaluation grid.}
    \label{fig:supp-acat-tau}
\end{figure}

\subsection*{Model calibration and hyperparameter sensitivity}

\paragraph*{MMD bandwidth}
The reference local-MMD bandwidth is participant specific and held constant
over that participant's observation times. For participant 16, the first and last usable post-randomization block
weights are refitted with multipliers \(c\in\{0.1,0.5,1,2,5\}\), keeping the
shared dictionary and ridge penalty fixed. The resulting
endpoint weights and fitted densities are displayed in
Figure~\ref{fig:supp-bandwidth}.

\begin{figure}[!htbp]
    \centering
    \begin{subfigure}[t]{0.48\textwidth}
        \includegraphics[width=\linewidth]{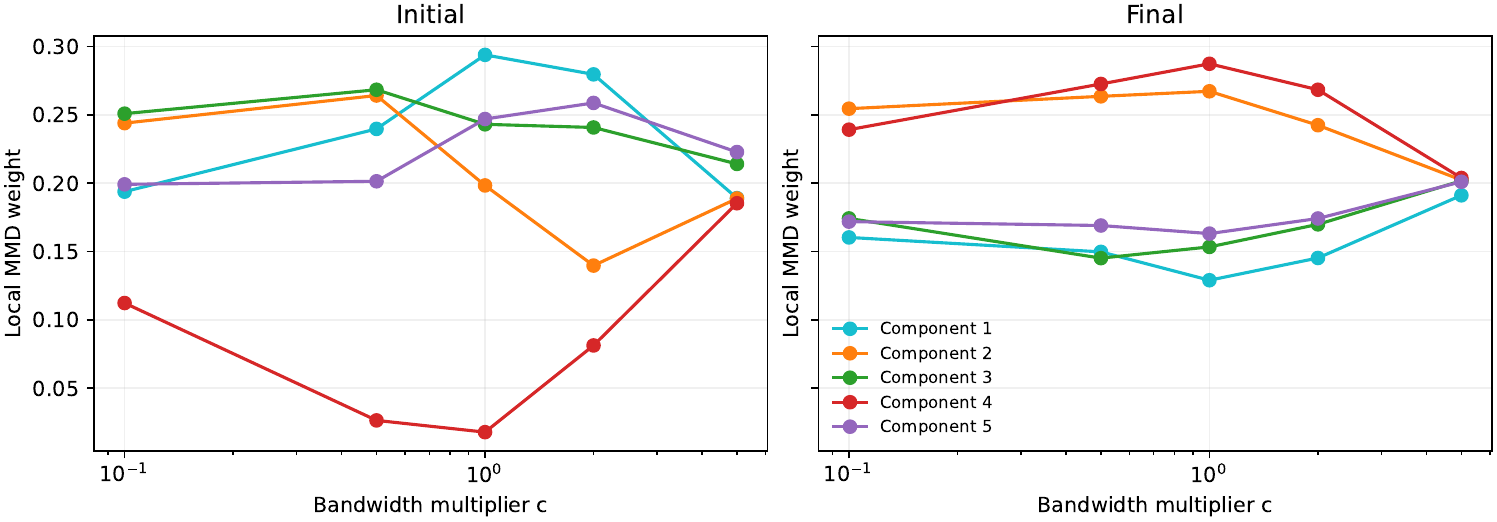}
        \caption{Endpoint weights.}
    \end{subfigure}\hfill
    \begin{subfigure}[t]{0.48\textwidth}
        \includegraphics[width=\linewidth]{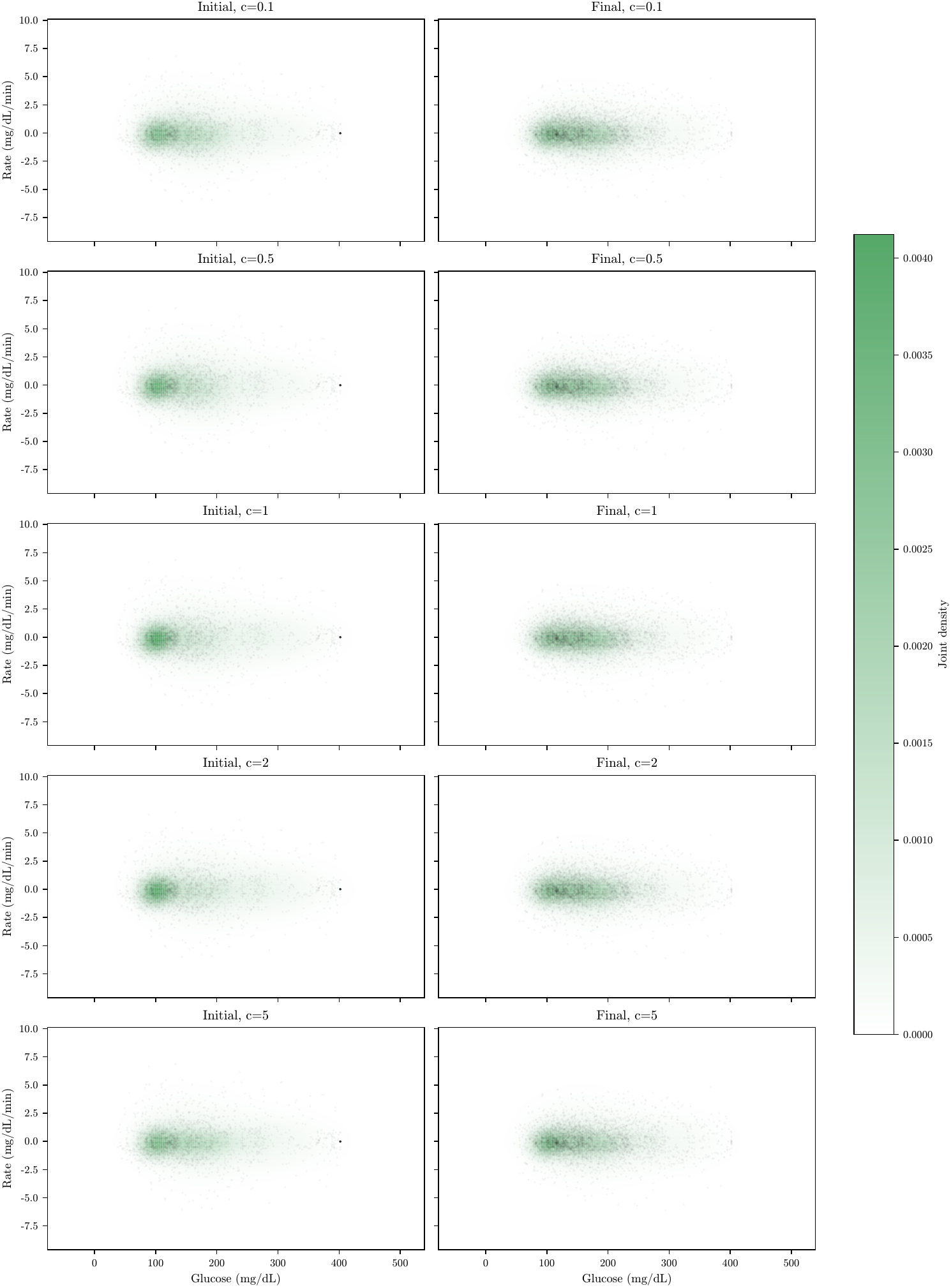}
        \caption{Joint density.}
    \end{subfigure}
    \medskip
    \begin{subfigure}[t]{0.48\textwidth}
        \includegraphics[width=\linewidth]{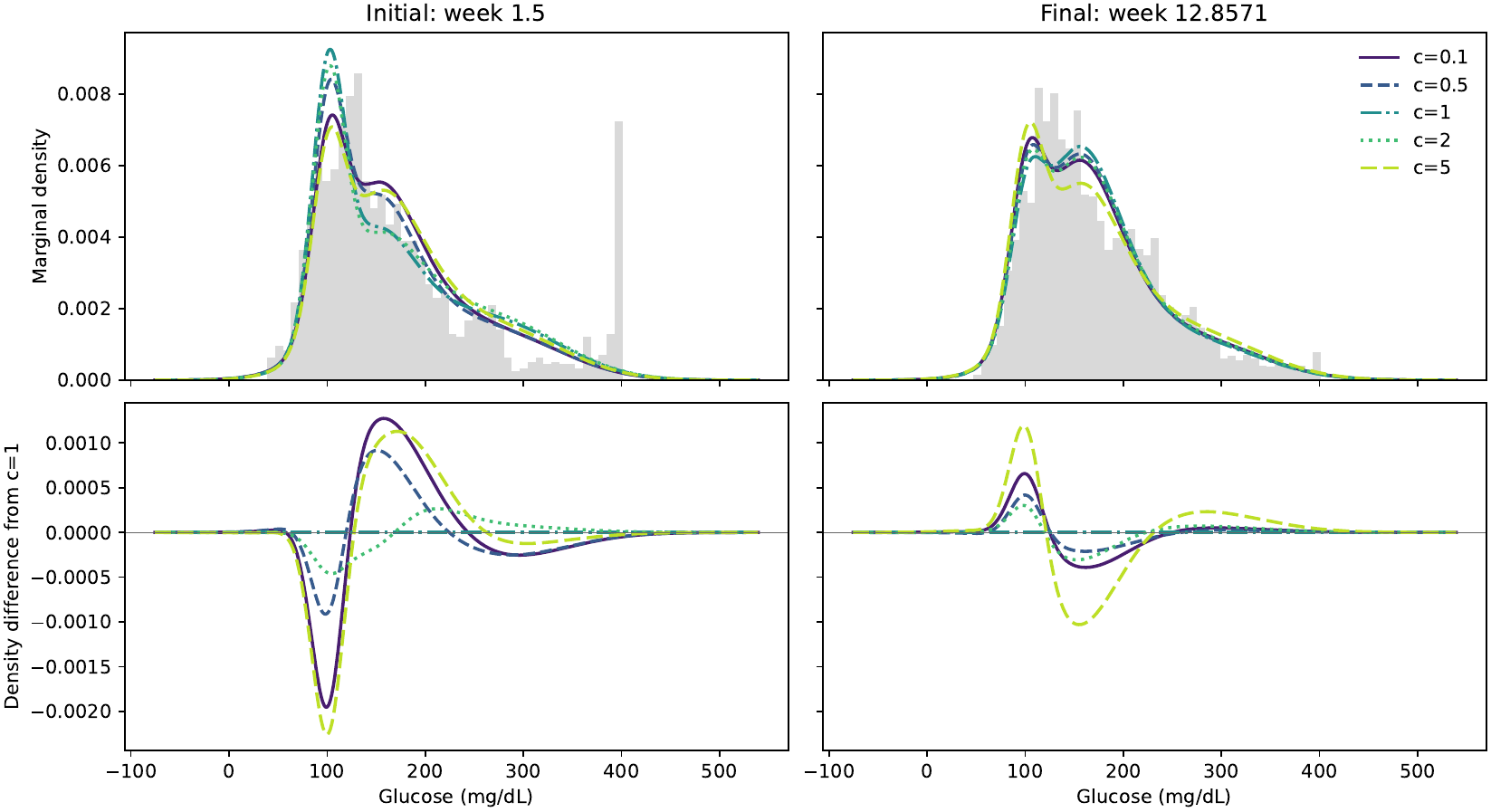}
        \caption{Glucose marginal.}
    \end{subfigure}\hfill
    \begin{subfigure}[t]{0.48\textwidth}
        \includegraphics[width=\linewidth]{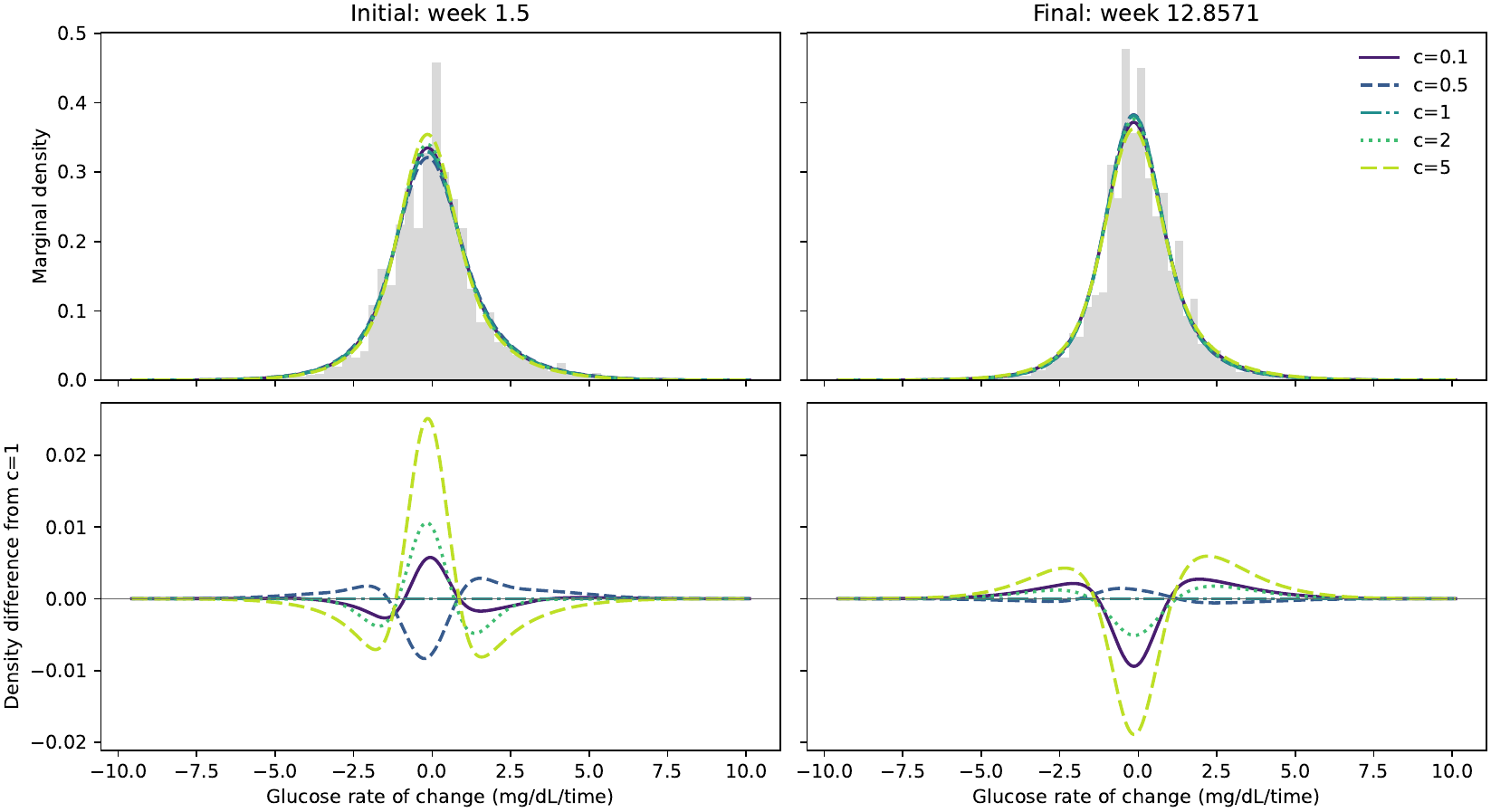}
        \caption{Rate-of-change marginal.}
    \end{subfigure}
    \caption{Bandwidth sensitivity for participant 16. Only the local-MMD
    bandwidth changes; the shared dictionary and ridge penalty are fixed. These panels assess bandwidth sensitivity at the first and last usable follow-up blocks; they do not represent the pre-randomization baseline contrast.}
    \label{fig:supp-bandwidth}
\end{figure}

\paragraph*{Neural-ODE temporal features}
The reported clinical model uses a linear time term and two Fourier harmonics
(Fourier order 2), with no higher-order polynomial terms. The following
panels compare Fourier orders and polynomial-only feature maps while leaving the weekly local-MMD
weights unchanged. These are in-sample fit diagnostics rather than a separate
inferential analysis. The corresponding comparisons for the three
illustrative participants are shown in Figure~\ref{fig:supp-node-features}.

\begin{figure}[!htbp]
    \centering
    \begin{subfigure}[t]{0.48\textwidth}
        \includegraphics[width=\linewidth]{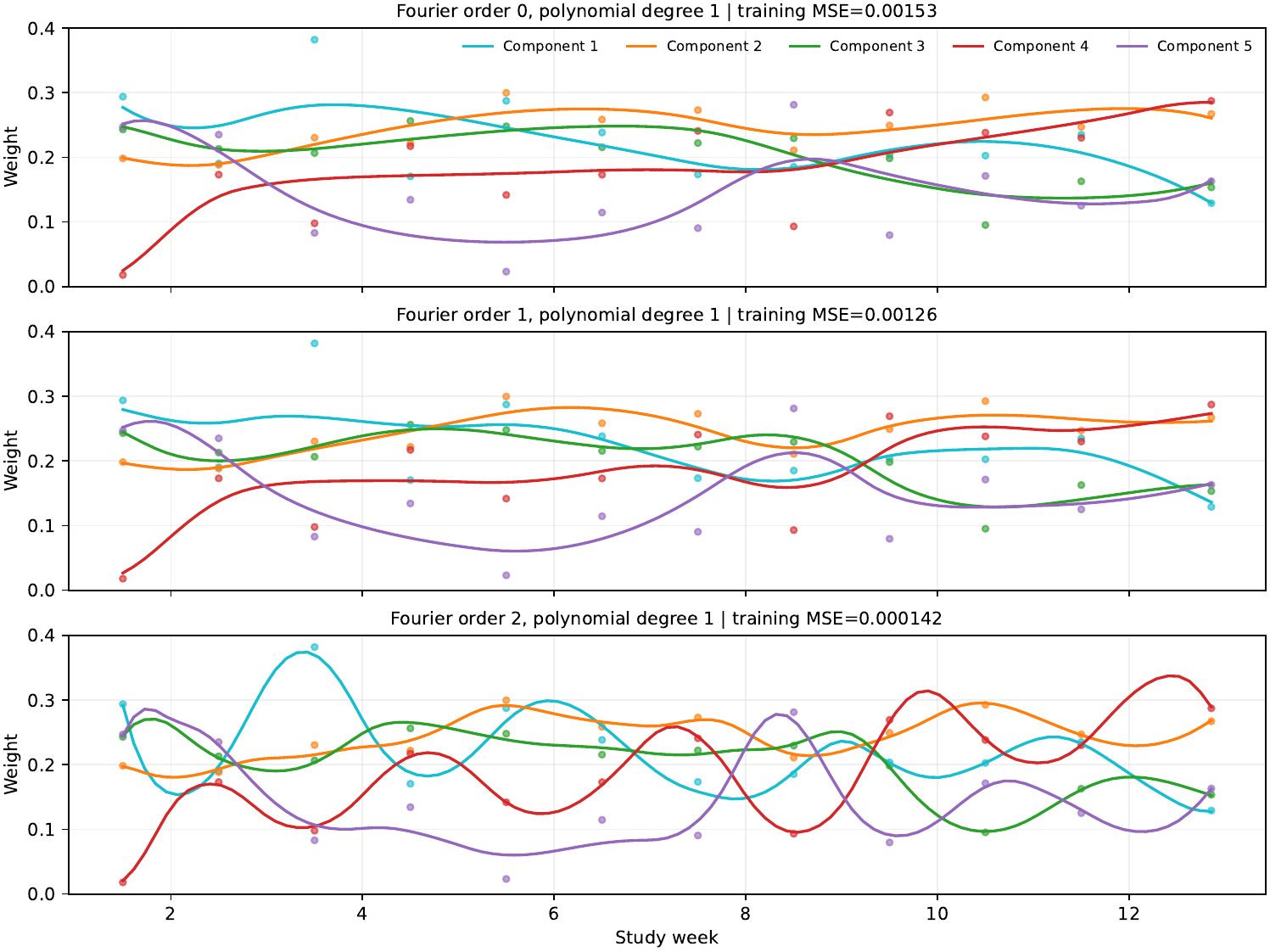}
        \caption{Participant 16 (\Treatment): Fourier.}
    \end{subfigure}\hfill
    \begin{subfigure}[t]{0.48\textwidth}
        \includegraphics[width=\linewidth]{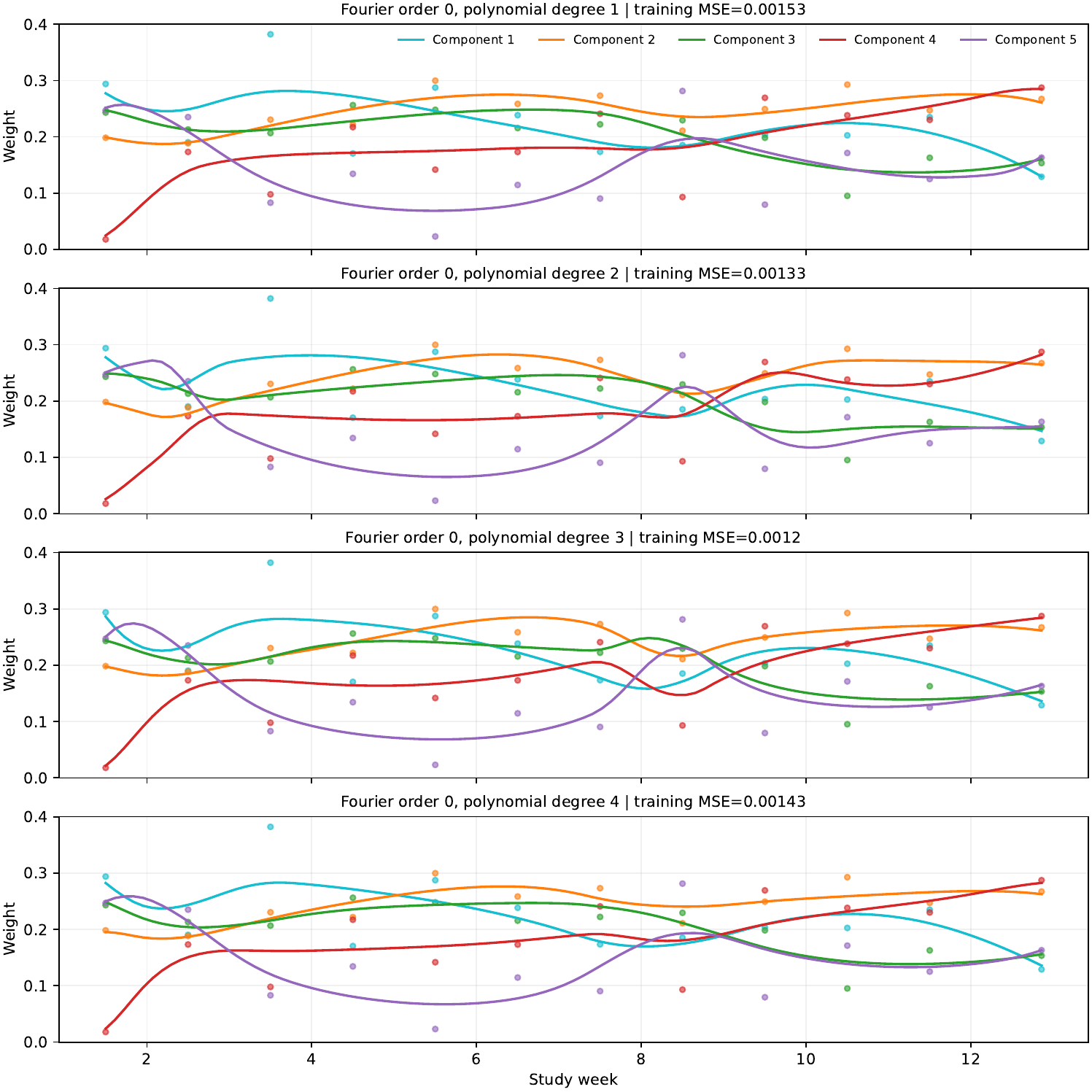}
        \caption{Participant 16 (\Treatment): polynomial.}
    \end{subfigure}
    \medskip
    \begin{subfigure}[t]{0.48\textwidth}
        \includegraphics[width=\linewidth]{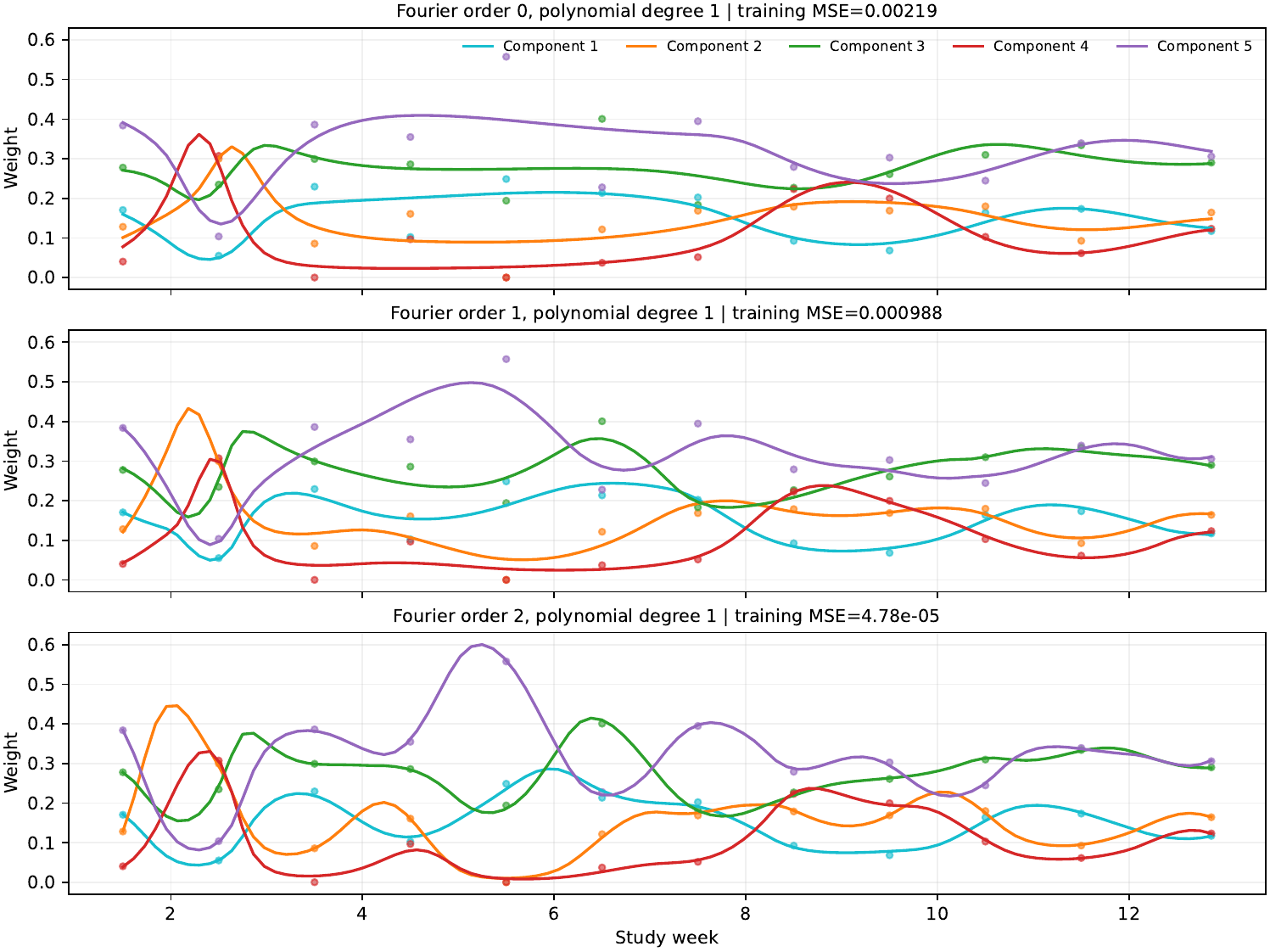}
        \caption{Participant 26 (\Treatment): Fourier.}
    \end{subfigure}\hfill
    \begin{subfigure}[t]{0.48\textwidth}
        \includegraphics[width=\linewidth]{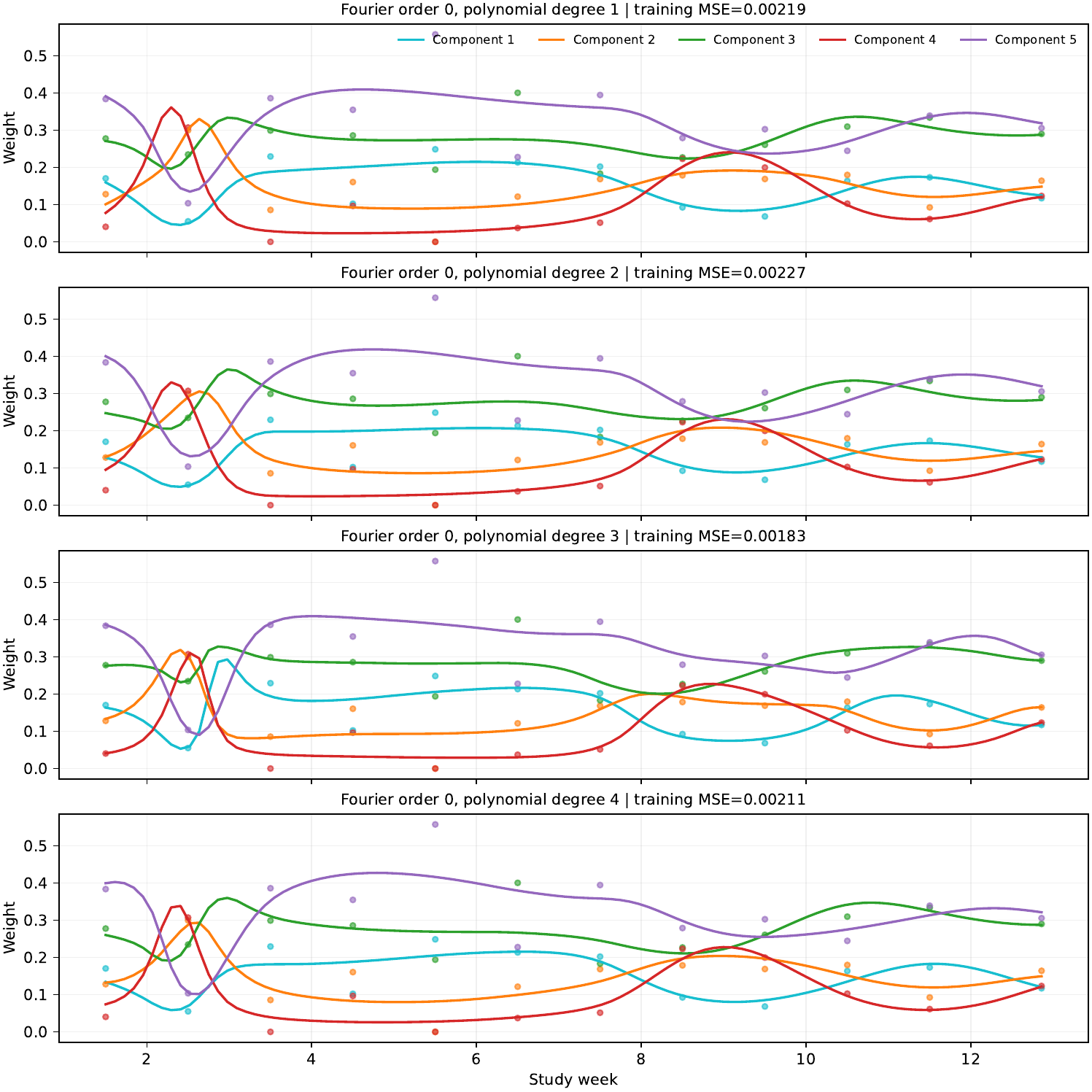}
        \caption{Participant 26 (\Treatment): polynomial.}
    \end{subfigure}
    \medskip
    \begin{subfigure}[t]{0.48\textwidth}
        \includegraphics[width=\linewidth]{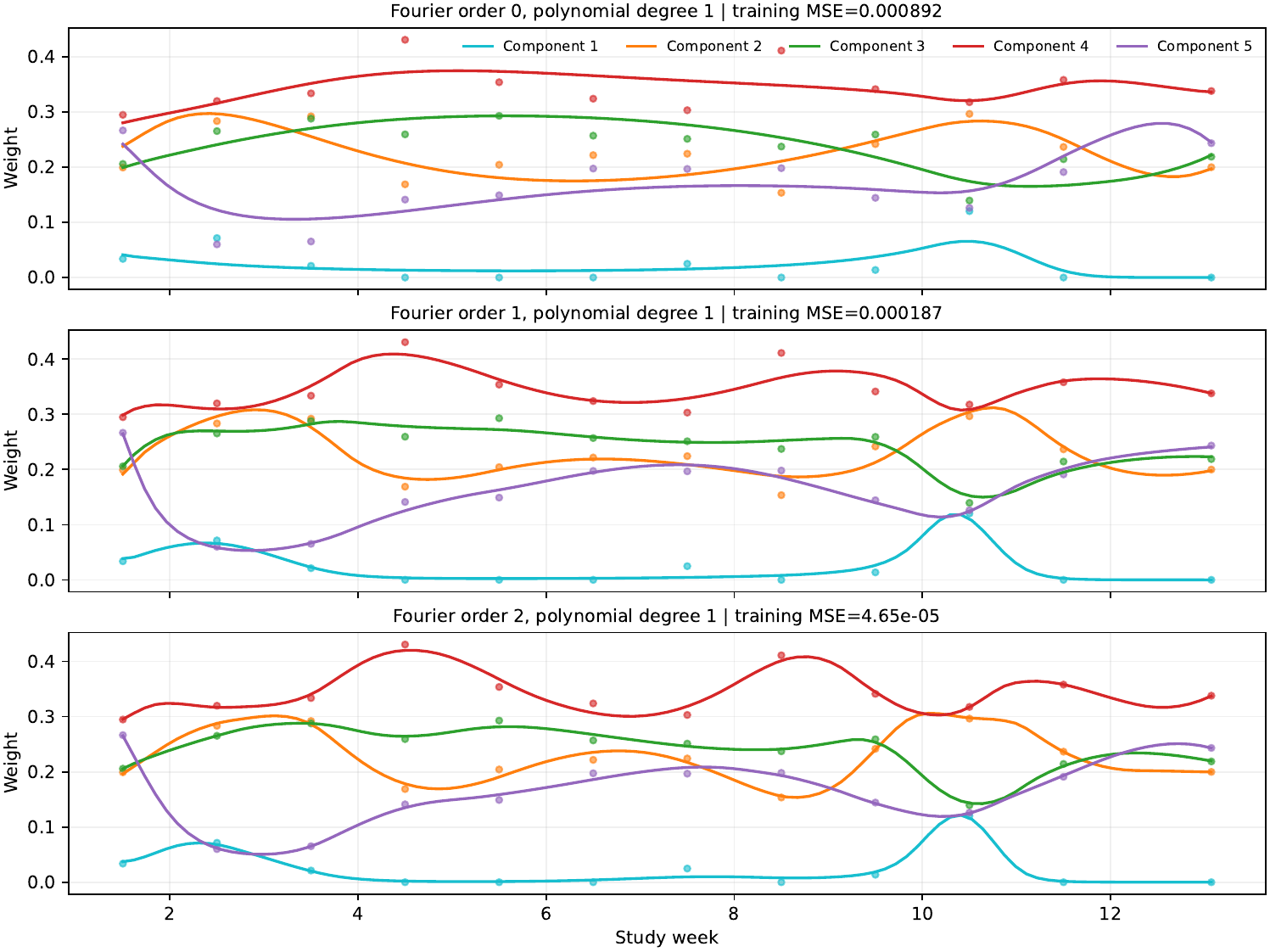}
        \caption{Participant 39 (\Control): Fourier.}
    \end{subfigure}\hfill
    \begin{subfigure}[t]{0.48\textwidth}
        \includegraphics[width=\linewidth]{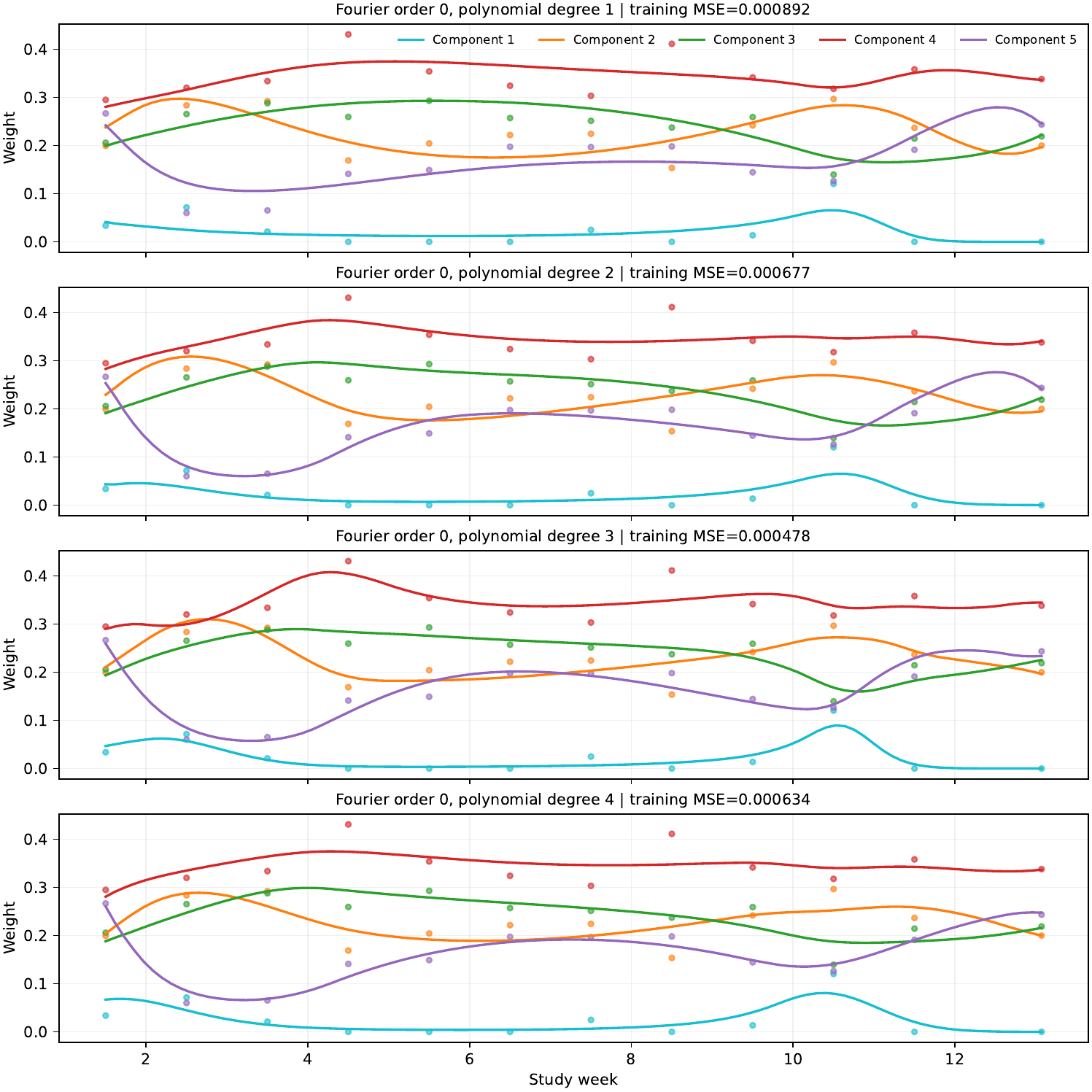}
        \caption{Participant 39 (\Control): polynomial.}
    \end{subfigure}
    \caption{Neural-ODE temporal-feature sensitivity for the three
    representative participants. These are in-sample training diagnostics,
    not held-out model-selection scores.}
    \label{fig:supp-node-features}
\end{figure}

\subsection*{Goodness-of-fit diagnostics}

Rosenblatt QQ diagnostics for the bivariate analysis are shown in
\Cref{fig:rosenblatt_gof} for participants 16, 26, and 39. They are descriptive in-sample
calibration diagnostics.





\end{document}